\documentclass[11pt]{article}
\pdfoutput=1
\usepackage[authoryear, round]{natbib}
\usepackage{bm}
\usepackage{fullpage}
\usepackage[english]{babel}
\usepackage[utf8x]{inputenc}
\usepackage[T1]{fontenc}
\usepackage{ae} \usepackage{aecompl}
\usepackage{enumitem}
\setlist{nosep}
\usepackage{microtype}
\usepackage{booktabs}
\usepackage{hhline}
\usepackage{makecell}
\usepackage{etoolbox}
\apptocmd{\sloppy}{\hbadness 10000\relax}{}{}

\usepackage{amsmath,amsthm,amssymb,mathtools,thmtools}
\usepackage{graphicx}
\usepackage[textsize=scriptsize,disable]{todonotes}
\usepackage[colorlinks=true, allcolors=blue]{hyperref}
\usepackage{algorithm,algorithmicx,algpseudocode}
\usepackage{xfrac}
\usepackage[normalem]{ulem}

\newcommand{\wtilde}{\widetilde}

\hypersetup{colorlinks,
            breaklinks=true,
            linkcolor=blue,
            citecolor=blue,
            urlcolor=magenta,
            linktocpage,
            plainpages=false,
            bookmarks=false}

\usepackage{notations}
\usepackage{comment}
\usepackage{tikz,pgfplots,pgfplotstable}
\usepgfplotslibrary{units}
\usepackage{tabularx} 
\usepackage{subfig}
\usepackage{array}
\usepgfplotslibrary{groupplots}

\title{Benefits of Additive Noise\\ in Composing Classes with Bounded Capacity}

\author{
    Alireza Fathollah Pour\thanks{McMaster University, \texttt{fathola@mcmaster.ca}}
    \and 
    Hassan Ashtiani\thanks{McMaster University, \texttt{zokaeiam@mcmaster.ca}. Hassan Ashtiani is also a faculty affiliate at Vector Institute and supported by an NSERC Discovery Grant.}
}

\numberwithin{equation}{section}

\begin{document}

\maketitle

\begin{abstract}

We observe that given two (compatible) classes of functions $\cF$ and $\cH$ with small capacity as measured by their uniform covering numbers, the capacity of the composition class $\cH \circ \cF$ can become prohibitively large or even unbounded. We then show that adding a small amount of Gaussian noise to the output of $\cF$ before composing it with $\cH$ can effectively control the capacity of $\cH \circ \cF$, offering a general recipe for modular design. To prove our results, we define new notions of uniform covering number of random functions with respect to the total variation and Wasserstein distances. We instantiate our results for the case of multi-layer sigmoid neural networks. Preliminary empirical results on MNIST dataset indicate that the amount of noise required to improve over existing uniform bounds can be numerically negligible (i.e., element-wise i.i.d. Gaussian noise with standard deviation $10^{-240}$).\footnote{The source codes are available at \url{https://github.com/fathollahpour/composition_noise} } 
\end{abstract}
\newpage
\setcounter{tocdepth}{1}
\tableofcontents
\newpage
\section{Introduction}

Let $\cF$ be a class of functions from $\cX$ to $\cY$, and $\cH$ a class of functions from $\cY$ to $\cZ$. Assuming that $\cF$ and $\cH$ have bounded ``capacity'',  can we bound the capacity of their composition, i.e., $\cH \circ \cF=\{h \circ f \mid f\in \cF, h\in\cH\}$? Here, by capacity we mean learning-theoretic quantities such as VC dimension, fat-shattering dimension, and (uniform) covering numbers associated with these classes (see \cite{vapnik1999nature,anthony1999neural,shalev2014understanding,mohri2018foundations} for an introduction). Being able to control the capacity of composition of function classes is useful, as it offers a modular approach to design sophisticated classes (and therefore learning algorithms) out of simpler ones. To be concrete, we want to know if the uniform covering number (as defined in the next section) of $\cH \circ \cF$ can be ``effectively'' bounded as a function of the uniform covering numbers of $\cF$ and $\cH$.

The answer to the above questions is true when $\cF$ is a set of binary valued functions (i.e., $\cY=\{0,1\}$ in the above). More generally, the capacity of the composition class (as measured by the uniform covering number) can be bounded as long as $|\cY|$ is relatively small (see Proposition~\ref{prop7}). But what if $\cY$ is an infinite set, such as the natural case of $\cY=[0,1]$? Unfortunately, in this case the capacity of $\cH \circ \cF$ (as measured by the covering number) 
can become unbounded (or excessively large)
even when both $\cF$ and $\cH$ have bounded (or small) capacities; see Propositions~{\ref{prop8}} and {\ref{prop9}}.

Given the above observation, we ask whether there is a general and systematic way to control the capacity of the composition of bounded-capacity classes. More specifically, we are interested in the case where the domain sets are multi-dimensional real-valued vectors (e.g., $\cX\subset\bR^d$, $\cY\subset\bR^p$, and $\cZ\subset \bR^q$). The canonical examples of such classes are those associated with neural networks.

A common approach to control the capacity of $\cH \circ \cF$ is assuming that $\cH$ and $\cF$ have bounded capacity and $\cH$ consists of Lipschitz functions (with respect to appropriate metrics). Then the capacity of $\cH \circ \cF$ can be bounded as long as $\cH$ has a small ``global cover'' (see Remark~\ref{rem:global}). This observation has been used to bound the capacity of neural networks in terms of the magnitude of their weights~\citep{bartlett1996valid}. More generally, the capacity of neural networks that admit Lipschitz  continuity can be bounded based on their group norms and spectral norms \citep{pmlr-v40-Neyshabur15,bartlett2017spectrally,golowich2018size}. One benefit of this approach is that the composition of Lipschitz classes is still Lipschitz (although with a larger Lipschitz constant). 

While building classes of functions from composition of Lipschitz classes is useful, it does not necessarily work as a general recipe. In fact, some commonly used classes of functions do not admit a small Lipschitz constant. Consider the class of single-layer neural networks defined over bounded input domain $[-B,B]^d$ and with the sigmoid activation function. While the sigmoid activation function itself is Lipschitz, the Lipschitz constant of the network depends on the magnitude of the weights. Indeed, we empirically observe that this can turn Lipschitzness-based bounds on the covering number of neural networks worse than classic VC-based bounds. 

Another limitation of using Lipschitz classes is that they cannot be easily ``mixed and matched'' with other (bounded-capacity) classes. For example, suppose $\cF$ is a class of $L$-Lipschitz functions (e.g., multi-layer sigmoid neural networks with many weights but small magnitudes). Also, assume $\cH$ is a non-Lipschitz class with bounded uniform covering number (e.g., one layer sigmoid neural network with unbounded weights). Then although both $\cF$ and $\cH$ have bounded capacity, $\cH\circ\cF$ is not Lipschitz and its capacity cannot be generally controlled.

We take a different approach for composing classes of functions. A key observation that we make and utilize is that adding a little bit of noise while ``gluing'' two classes can help in controlling the capacity of their composition. In order to prove such results, we define and study uniform covering numbers of random functions with respect to total variation and Wasserstein metrics. The bounds for composition then come naturally through the use of data processing inequality for the total variation distance metric.

{\bf Contributions and Organization.} 

\begin{itemize}
    \item Section~\ref{sec:notations} provides the necessary notations and includes the observations that composing real-valued functions can be more challenging than binary valued functions (Propositions~\ref{prop7},\ref{prop8}, and \ref{prop9}).
   
    \item In Section~\ref{sec:randhypo}, we define a new notion of covering number for random functions (Definition~\ref{def:unif_cnr}) with respect to total variation (TV) and Wasserstein distances.
   
    \item The bulk of our technical results appear in Section~\ref{sec:results}. These include a composition result for random classes with respect to the TV distance (Lemma~\ref{lemma:compose_tv}) that is based on the data processing inequality. 
    We also show how one can translate TV covering numbers to conventional $\|.\|_2$ counterparts (Theorem~\ref{thm:tv_to_ell2}) and vice versa (Corollary~\ref{coll:ell2_to_tv}). A useful tool is Theorem~\ref{thm:wasser_tv} which translates Wasserstein covers to TV covers when a Gaussian noise is added to the output of functions.
    
    \item Section~\ref{sub:NN} provides a stronger type of covering number for classes of single-layer noisy neural networks with the sigmoid activation function (Theorem~\ref{lemma:net}).

    \item In Section~\ref{sec:applications}, we use the tools developed in the previous sections and prove a novel bound on the $\|.\|_2$ covering number of noisy deep neural networks (Theorem~\ref{thm:neuralnet}). We then instantiate our results (Corollary~\ref{thm:ours_main}) and compare it with several other covering number bounds in Section~\ref{subsec:cover_bounds}.
    
    \item In Section~\ref{sec:nvac} we define NVAC, a metric for comparing generalization bounds (Definition~\ref{def:nvac}) based on the number of samples required to make the bound non-vacuous.
    
    \item We offer some preliminary experiments, comparing various generalization bounds in Section~\ref{sec:experiments}. We observe that even a negligible amount of Gaussian noise can improve NVAC over other approaches without affecting the accuracy of the model on train or test data.

\end{itemize}

\section{Related work}
Adding various types of noise have been empirically shown to be beneficial in training neural networks. In dropout noise \citep{JMLR:v15:srivastava14a} (and its variants such as DropConnect \citep{wan2013regularization}) the output of some of the activation functions (or weights) are randomly set to zero. These approaches are thought to act as a regularizer. Another example is Denoising AutoEncoders~\citep{vincent2008extracting}, which adds noise to the input of the network while training stacked autoencoders. 

 There has been efforts on studying the theory behind the effects of noise in neural networks.  \citet{jim1996analysis} study the effects of different types of additive and multiplicative noise on convergence speed and generalization of recurrent neural networks (RNN) and suggest that noise can help to speed up the convergence on local minima surfaces. \citet{lim2021noisy} formalize the regularization effects of noise in RNNs and show that noisy RNNs are more stable and robust to input perturbations. \citet{WANG2019177} and \citet{gao2016dropout} analyze the networks with dropout noise and find bounds on Rademacher complexities that are dependent on the product of norms and dropout probability. It is noteworthy that our techniques and results are quite different, and require a negligible amount of additive noise to work, while existing bounds for dropout improve over conventional bounds only if the amount of noise is substantial. Studying dropout noise with the tools developed in this paper is a direction for future research.

Studying PAC learning and its sample complexity is by now a mature field; see \citet{vapnik1999nature,shalev2014understanding,mohri2018foundations}. In the case of neural networks, standard Vapnik-Chervonenkis-based complexity bounds have been established~\citep{baum1988size, maass1994neural,goldberg1995bounding,vidyasagar1997theory, sontag1998vc,koiran1998vapnik,bartlett1998almost,bartlett2003vapnik, bartlett2019nearly}. These offer generalization bounds that depend on the number of parameters of the neural network. There is also another line of work that aims to prove a generalization bound that mainly depends on the norms of the weights and Lipschitz continuity properties of the network rather than the number of parameters~\citep{bartlett1996valid,anthony1999neural, zhang2002covering,pmlr-v40-Neyshabur15, bartlett2017spectrally,neyshabur2017pac,golowich2018size,arora2018stronger,nagarajan2018deterministic,long2020size}. We provide a more detailed discussion of some of these results in Appendix~\ref{app:coverapproaches}. Finally, we refer the reader to~\citet{anthony1999neural} for an introductory discussion on this subject.

The above-mentioned bounds are usually vacuous for commonly used data sets and architectures. \citet{DR17} (and later \citet{zhou2019non}) show how to achieve a non-vacuous bound using the PAC Bayesian framework. These approaches as well as compression-based methods~\citep{arora2018stronger} are, however, examples of ``two-step'' methods; see Appendix~\ref{app:coverapproaches} for more details. It has  been argued that uniform convergence theory may not fully explain the performance of neural networks~\citep{nagarajan2019uniform, zhang2021understanding}. One conjecture is that implicit bias of gradient descent~\citep{gunasekar2017implicit,arora2019implicit,ji2020gradient,chizat2020implicit,ji2021characterizing} can lead to benign overfitting~\citep{belkin2018overfitting,belkin2019does,bartlett2020benign}; see \citet{bartlett2021deep} for a recent overview.

In a recent line of work, generalization has been studied from the perspective of information theory~\citep{russo2016controlling, xu2017information, russo2019much, steinke2020reasoning}, showing that a learning algorithm will generalize if the (conditional) mutual information between the training sample and the learned model is small. Utilizing these results, a number of generic generalization bounds have been proved for Stochastic Gradient Langevin Descent (SGLD)~\citep{raginsky2017non, haghifam2020sharpened} as well as Stochastic Gradient Descent (SGD)~\cite{neu2021information}. Somewhat related to our ``noise analysis'', these approaches (virtually) add noise to the parameters to control the mutual information. In contrast, we add noise between modules for composition (e.g., in between layers of a neural network). Furthermore, we prove uniform (covering number) bounds while these approaches are for generic SGD/SGLD and are mostly agnostic to the structure of the hypothesis class. Investigating the connections between our analysis and information-theoretic techniques is a direction for future research.

\section{Notations and background}\label{sec:notations}
\paragraph{Notation.} $\cX\subseteq \bR^d$ and $\cY \subseteq \bR^p$ denote two (domain) sets. For $x\in \cX$, let $\|x\|_1,\|x\|_2$, and $\|x\|_{\infty}$ denote the $\ell_1,\ell_2$, and $\ell_{\infty}$ norm of the vector $x$, respectively. We denote the cardinality of a set $S$ by $|S|$. The set of natural numbers smaller or equal to $m$ are denoted by $[m]$. A hypothesis is a Borel function $f:\bR^d\rightarrow\bR^p$, and a hypothesis class $\cF$ is a set of hypotheses.

We also define the random counterparts of the above definitions and use an overline to distinguish them from the non-random versions. $\rv{\cX}$ denotes the set of all random variables defined over ${\cX}$ that admit a generalized density function.\footnote{Both discrete (by using Dirac delta function) and absolutely continuous random variables admit generalized density function.} We sometimes abuse the notation and write $\rv{x}\in \cX$ rather than $\rv{x}\in \rv{\cX}$ (e.g., $\rv{x}\in \bR^d$ is a random variable taking values in $\bR^d$).
By $\rv{y}=f(\rv{x})$ we denote a random variable that is the result of mapping $\rv{x}$ using a Borel function ${f}:\bR^d\rightarrow\bR^p$.
We use $\rv{f}:\bR^d\rightarrow\bR^p$ to indicate that the mapping itself can be random. We use $\rv{\cF}$ to signal that the class can include random hypotheses. 
We conflate the notation for random hypotheses so that they can be applied to both random and non-random inputs (e.g., $\rv{f}(\rv{x})$ and $\rv{f}(x)$).\footnote{Technically, we consider $\rv{f}(x)$ to be $\rv{f}(\rv{\delta_x})$, where $\rv{\delta_x}$ is a random variable with Dirac delta measure on $x$.} 

\begin{definition}[Composition of two hypothesis classes]
We denote by $h\circ f$ the function $h(f(x))$ (assuming the range of $f$ and the domain of $h$ are compatible).
The composition of two hypothesis classes $\cF$ and $\cH$ is defined by $\cH\circ\cF=\{h\circ f\ \mid h\in \cH, f\in \cF\}$. Composition of classes of random hypotheses is defined similarly by $\rv{\cH}\circ\rv{\cF}=\{\rv{h}\circ \rv{f}\ \mid \rv{h}\in \rv{\cH}, \rv{f}\in \rv{\cF}\}$.
\end{definition}

The following singleton class $\rv{\cG_{\sigma}}$ will be used to create noisy functions (e.g., using $\rv{\cG_{\sigma}} \circ \cF$).
\begin{definition}[The Gaussian Noise Class]
The $d$-dimensional noise class with scale $\sigma$ is denoted by $\rv{\cG_{\sigma, d}}=\{\rv{g_{\sigma,d}}\}$. Here, $\rv{g_{\sigma,d}}:\bR^d\rightarrow\bR^d$ is a random function defined by $\rv{g_{\sigma,d}}(\rv{x})=\rv{x}+\rv{z}$, where $\rv{z}\sim \cN(\mathbf{0},\sigma^2 I_d)$. When it is clear from the context we drop $d$ and write $\rv{\cG_{\sigma}}=\{\rv{g_{\sigma}}\}$.
\end{definition}

In the rest of this section, we define the standard notion of uniform covering numbers for hypothesis classes. Intuitively, classes with larger uniform covering numbers have more capacity/flexibility, and therefore require more samples to be learned.

\begin{definition}[Covering number] 
Let $(\cX,\rho)$ be a metric space. We say that a set $A\subset \cX$ is $ \epsilon$-covered by a set $ C\subseteq A$ with respect to $ \rho$, if for all $a\in A$ there exists $c\in C$ such that $\displaystyle \rho(a,c)\leq\epsilon$. The cardinality of the smallest set $ C$ that $ \epsilon$-covers $A$ is denoted by $ N(\epsilon,A,\rho)$ and it is referred to as the $\epsilon$-covering number of $A$ with respect to metric $\rho$.
\end{definition}

\begin{definition}[Extended metrics]
Let $(\cX, \rho)$ be a metric space. Let $u=(a_1, \ldots, a_m),v=(b_1, \ldots, b_m)\in \cX^m$ for $m\in \bN$. The $\infty$-extended and $\ell_2$-extended metrics over $\cX^m$ are defined by $\rho^{\infty, m}(u,v)=\sup_{1\leq i \leq m} \rho(a_i,b_i)$ and $\rho^{\ell_2, m}(u,v)=\sqrt{\frac{1}{m}\sum_{i=1}^m (\rho(a_i,b_i))^2}$, respectively. We drop $m$ and use $\rho^\infty$ or $\rho^{\ell_2}$ if it is clear from the context.
\end{definition}

\begin{remark}\label{rem:one_infty}
The extended metrics are used in Definition~\ref{def:unif_cn} and capture the distance of two hypotheses on an input sample of size $m$.
A typical example of $\rho$ is the Euclidean distance over $\bR^p$, for which the extended metrics are denoted by $\|.\|^{\infty,m}_2$ and $\|.\|^{\ell_2,m}_2$. Unlike $\infty$-extended metric, the $\ell_2$-extended metric is normalized by $1/\sqrt{m}$, and therefore we have $\rho^{\ell_2,m}(u,v)\leq\rho^{\infty,m}(u,v)$ for all $u,v\in \cX^m$.
\end{remark}
\begin{definition}[Uniform covering number]\label{def:unif_cn} 
Let $(\cY,\rho)$ be a metric space and $\cF$ a hypothesis class of functions from $\cX$ to $\cY$. For a set of inputs $S=\{x_1, x_2, \ldots, x_m\}\subseteq \cX$, we define the restriction of $\cF$ to $S$ as $\cF_{|S}=\{(f(x_1), f(x_2), \ldots, f(x_m)) :f\in \cF\}\subseteq \cY^m$. The uniform $\epsilon$-covering numbers of hypothesis class $\cF$ with respect to metrics $\rho^{\infty},\rho^{\ell_2}$ are denoted by $N_{U}(\epsilon,\cF,m,\rho^{\infty})$ and $N_{U}(\epsilon,\cF,m,\rho^{\ell_2})$ and are the maximum values of $ N(\epsilon,\cF_{|S},\rho^{\infty, m})$ and $ N(\epsilon,\cF_{|S},\rho^{\ell_2, m})$ over all $S\subseteq \cX$ with $|S|=m$, respectively.
\end{definition}
It is well-known that the Rademacher complexity and therefore the generalization gap of a class can be bounded based on logarithm of the uniform covering number. For sake of brevity, we defer those results to Appendix~\ref{app:dudley}. Therefore, our main object of interest is bounding (logarithm of) the uniform covering number. The following propositions show that there is a stark difference between classes of functions with finite range versus continuous valued functions when it comes to bounding the uniform covering number of composite classes; the proofs can be found in Appendix \ref{app:props}.
\begin{proposition}\label{prop7}
Let $\cY$ be a finite domain ($|\cY|=k$) and $\rho(y,\hat{y})=\indicator\{y \neq \hat{y}\}$ be a metric over $\cY$. For any class $\cF$ of functions from $\cX$ to $\cY$ and any class $\cH$ of functions from $\cY$ to $\bR^d$ we have $N_U(\epsilon,\cH\circ\cF,m,\|.\|_2^{\infty})\leq N_1.N_U(\epsilon,\cH,mN_1,\|.\|_2^{\infty})$ where $N_1=N_U(0.5,\cF,m,\rho^{\infty})$.
\end{proposition}
\begin{proposition}\label{prop8}
Let $\cF=\{f_w(x)=w x\mid w\in (0,1), x\in (0,1)\}$ be a class of functions and $\cH=\{h(y)=1/y \mid y\in(0,1)\}$ be a singleton class. Then, $N_U(\epsilon,\cF,m,\|.\|_2^{\ell_2})\leq\lceil 2/\epsilon^2\rceil$ and $N_U(\epsilon,\cH,m,\|.\|_2^{\ell_2})=1$, but $N_U(\epsilon,\cH\circ\cF,m,\|.\|_2^{\ell_2})$ is unbounded.
\end{proposition}

\begin{proposition}\label{prop9}
For every $\epsilon>\epsilon'>0$, there exist hypothesis classes $\cF$ and $\cH$ such that for every $m$ we have $N_U(\epsilon',\cH,m,\|.\|_2^{\infty})\leq m+1$ and 
$N_U(\epsilon',\cF,m,\|.\|_2^{\infty})=1$, yet $N_U(\epsilon,\cH\circ\cF,m,\|.\|_2^{\infty}) \geq 2^m$.
\end{proposition}

\section{Covering random hypotheses}\label{sec:randhypo}

We want to establish the benefits of adding (a little bit of) noise when composing hypothesis classes. Therefore, we need to analyze classes of \emph{random} hypotheses. One way to do this is to replace each hypothesis with its expectation, creating a deterministic version of the hypothesis class. Unfortunately, this approach misses the whole point of having noisy hypotheses (and their benefits in composition). Instead, we extend the definition of uniform covering numbers to classes of random hypotheses $\rv{\cF}$. The following is basically the random counterpart of Definition~\ref{def:unif_cn}. 

\begin{definition}[Uniform covering number for classes of random hypotheses]\label{def:unif_cnr}
Let $(\rv{\cY},\rho)$ be a metric space and $\rv{\cF}$ a class of random hypotheses from $\rv{\cX}$ to $\rv{\cY}$. For a set of random variables $\rv{S}=\{\rv{x_1}, \rv{x_2}, \ldots, \rv{x_m}\}\subseteq \rv{\cX}$, we define the restriction of $\rv{\cF}$ to $\rv{S}$ as $\displaystyle \rv{\cF}_{|\rv{S}}=\{(\rv{f}(\rv{x_1}), \rv{f}(\rv{x_2}), \ldots, \rv{f}(\rv{x_m})) :\rv{f}\in \rv{\cF}\}\subseteq{\rv{\cY}}^m$. 
Let $\Gamma\subseteq \rv{\cX}$. The uniform $\epsilon$-covering numbers of $\rv{\cF}$ with respect to $\Gamma$ and metrics $\rho^{\infty}$ and $\rho^{\ell_2}$ are defined by
\begin{equation*}
\begin{aligned}
    N_{U}(\epsilon,\rv{\cF},m,\rho^{\infty}, \Gamma) = \sup_{S\subseteq \Gamma, |S|=m} N(\epsilon,\rv{\cF}_{|\rv{S}},\rho^{\infty, m}),\\
    N_{U}(\epsilon,\rv{\cF},m,\rho^{\ell_2}, \Gamma) = \sup_{S\subseteq \Gamma, |S|=m} N(\epsilon,\rv{\cF}_{|\rv{S}},\rho^{\ell_2, m}).
\end{aligned}
\end{equation*}
\end{definition}

\begin{remark}
Unlike in Definition~\ref{def:unif_cn} where $\rho$ is usually the $\|.\|_2$ metric in the Euclidean space, here in Definition~\ref{def:unif_cnr} $\rho$ is defined over random variables. More specifically, we will use the Total Variation and Wasserstein metrics as concrete choices for $\rho$.
\end{remark}

\begin{remark}
The specific choices that we use for $\Gamma$ are
\begin{itemize}
    \item $\Gamma=\rv{\cX_d}$: the set of all random variables defined over $\bR^d$ that admit a generalized density function.
    \item $\Gamma=\rv{\cX_{B,d}}$:  the set of all random variables defined over $[-B,B]^d$ that admit a generalized density function.
    \item $\Gamma=\rv{\Delta_{d}}=\{\rv{\delta_x} \mid x\in \bR^d\}$ and $\Gamma=\rv{\Delta_{B,d}}=\{\rv{\delta_x} \mid x\in [-B,B]^d\}$, where $\rv{\delta_x}$ is the random variable associated with Dirac delta measure on $x$.
    \item $\Gamma=\rv{\cG_{\sigma,d}} \circ \rv{\cX_{B,d}}=\{\rv{g_{\sigma,d}}(\rv{x})\mid \rv{x}\in \rv{\cX_{B, d}}\}$: all members of $\rv{\cX_{B,d}}$ after being ``smoothed'' by adding (convolving with) Gaussian noise. 
\end{itemize}
\end{remark}

\begin{remark}\label{rem:global}
Some hypothesis classes that we work with have ``global'' covers, in the sense that the uniform covering number does not depend on $m$. We therefore use the following notation
\[
N_{U}(\epsilon,\rv{\cF},\infty,\rho^{{\infty}}, \Gamma) = \lim_{m\to \infty} N_{U}(\epsilon,\rv{\cF},m,\rho^{{\infty}}, \Gamma).
\]
\end{remark}

We now define Total Variation (TV) and Wasserstein metrics over probability measures rather than random variables, but with a slight abuse of notation we will use them for random variables too.

\begin{definition}[Total Variation Distance]
Let $\mu$ and $\nu$ denote two probability measures over $\cX$ and let $\Omega$ be the Borel sigma-algebra over $\cX$. The TV distance between $\mu$ and $\nu$ is defined by
\begin{equation*}
    d_{TV}(\mu,\nu)=\sup_{B\in\Omega}|\mu(B)-\nu(B)|.
\end{equation*}
Furthermore, if $\mu$ and $\nu$ have densities $f$ and $g$ then
\begin{equation*}
    d_{TV}(\mu,\nu)=\sup_{B\in\Omega}\Big|\int_{B}(f(x)-g(x))dx\Big|=\frac{1}{2} \int_{\cX}\left|f(x)-g(x)\right|dx=\frac{1}{2}\|f-g\|_1.
\end{equation*}
\end{definition}
\begin{definition}[Wasserstein Distance]
Let $\mu$ and $\nu$ denote two probability measures over $\cX$, and $\Pi(\mu,\nu)$ be the set of all their couplings. The Wasserstein distance between $\mu$ and $\nu$ is defined by
\begin{equation*}
    d_\cW(\mu,\nu)=\left(\inf_{\pi\in\Pi(\mu,\nu)}\int_{\cX\times\cX}\|x-y\|_2d\pi(x,y)\right).
\end{equation*}
\end{definition}
The following proposition makes it explicit that the conventional uniform covering number with respect to $\|.\|_2$ (Definition~\ref{def:unif_cn}) can be regarded as a special case of Definition~\ref{def:unif_cnr}.
\begin{proposition}\label{prop12}
Let $\cF$ be a class of (deterministic) hypotheses from $\bR^d$ to $\bR^p$. Then

$N_{U}(\epsilon,\cF,m,\|.\|_2^{\infty})=N_{U}(\epsilon,{\cF},d_\cW^{\infty}, m, \rv{\Delta_d})$ and $N_{U}(\epsilon,\cF,m,\|.\|_2^{\ell_2} )=N_{U}(\epsilon,{\cF},d_\cW^{\ell_2}, m, \rv{\Delta_d})$.

\end{proposition}

The proposition is the direct consequence of the Definitions \ref{def:unif_cn} and \ref{def:unif_cnr} once we note that the Wasserstein distance between Dirac random variables is just their $\ell_2$ distance, i.e.,  $d_\cW(\rv{\delta_x}, \rv{\delta_y})=\|x-y\|_2$.

\section{Bounding the uniform covering number}\label{sec:results}

This section provides tools that can be used in a general recipe for bounding the uniform covering number. The ultimate goal is to bound the (conventional) $\|.\|_2^\infty$ and $\|.\|_2^{\ell_2}$ uniform covering numbers for (noisy) compositions of hypothesis classes. In order to achieve this, we will show how one can turn TV covers into $\|.\|_2$ covers (Theorem~\ref{thm:tv_to_ell2}) and vice versa (Corollary~\ref{coll:ell2_to_tv}). But what is the point of going back and forth between $\|.\|_2$ and TV covers? Basically, the data processing inequality ensures an effective composition (Lemma~\ref{lemma:compose_tv}) for TV covers. Our analysis goes through a number of steps, connecting covering numbers with respect to $\|.\|_2$, Wasserstein, and TV distances. The missing proofs of this section can be found in Appendix \ref{app:sec:results}.

The following theorem considers the deterministic class $\cH$ associated with expectations of random hypotheses from $\rv{\cF}$, and shows that bounding the uniform covering number of $\rv{\cF}$ with respect to TV distance is enough for bounding the uniform covering number of $\cH$ with respect to $\|.\|_2$ distance.

\begin{theorem}[From a TV cover to a $\|.\|_2$ cover]\label{thm:tv_to_ell2}
Consider any class $\rv{\cF}$ of random hypotheses $\rv{f}:\bR^d\rightarrow[-B,B]^p$ with bounded output. 
Define the (nonrandom) hypothesis class $\cH=\{h:\bR^d\to [-B,B]^p \mid h(x)=\expects{\rv{f}}{~\rv{f}({x})}, \rv{f}\in \rv{\cF} \}$.
Then for every $\epsilon>0$, $m\in \bN$ these two inequalities hold:
\begin{equation*}
\begin{aligned}
     N_U(2B\epsilon\sqrt{p},\cH,m,\|.\|_2^{{\infty}}) \leq N_U(\epsilon,\rv{\cF},m,d_{TV}^{{\infty}}, \rv{\Delta_{d}}) \leq N_U(\epsilon,\rv{\cF},m,d_{TV}^{{\infty}},  \rv{\cX_{d}}),\\
      N_U(2B\epsilon\sqrt{p},\cH,m,\|.\|_2^{{\ell_2}}) \leq N_U(\epsilon,\rv{\cF},m,d_{TV}^{{\ell_2}}, \rv{\Delta_{d}}) \leq N_U(\epsilon,\rv{\cF},m,d_{TV}^{{\ell_2}}, \rv{\cX_{d}}).
\end{aligned}
\end{equation*}
\end{theorem}
But what is the point of working with the TV distance? An important ingredient of our analysis is the use of data processing inequality which holds for the TV distance (see  Lemma~\ref{lemma:DPI}). The following lemma uses this fact, and shows how one can compose classes with bounded TV covers. 

\begin{lemma}[Composing classes with bounded TV covers]\label{lemma:compose_tv}
Let $\rv{\cF}$ be a class of random hypotheses from $\bR^d$ to $\bR^p$, and $\rv{\cH}$ be a class of random hypotheses from $\bR^p$ to $\bR^q$. For every $\epsilon, \epsilon'>0$, and every $m\in \bN$ these three inequalities hold:
\begin{align*}
  N_U\left(\epsilon+\epsilon',\rv{\cH}\circ\rv{\cF},m,d_{TV}^{{\infty}}, \rv{\cX_{d}}\right) \leq  N_U\left(\epsilon',\rv{\cH},{m}N_1,d_{TV}^{{\infty}}, \rv{\cX_{p}}\right).N_1,\\
  N_U\left(\epsilon+\epsilon',\rv{\cH}\circ\rv{\cF},m,d_{TV}^{{\infty}}, \rv{\Delta_d}\right) \leq  N_U\left(\epsilon',\rv{\cH},{m}N_2,d_{TV}^{{\infty}}, \rv{\cX_{p}}\right).N_2,\\
  N_U\left(\epsilon+\epsilon',\rv{\cH}\circ\rv{\cF},m,d_{TV}^{\ell_2}, \rv{\Delta_d}\right) \leq  N_U\left(\epsilon',\rv{\cH},{m}N_3,d_{TV}^{{\infty}}, \rv{\cX_{p}}\right).N_3, 
\end{align*}
where $N_1=N_U\left(\epsilon,\rv{\cF},m,d_{TV}^{{\infty}}, \rv{\cX_{d}}\right)$, $N_2=N_U\left(\epsilon,\rv{\cF},m,d_{TV}^{{\infty}}, \rv{\Delta_d}\right)$ and $N_3=N_U\left(\epsilon,\rv{\cF},m,d_{TV}^{{\ell_2}}, \rv{\Delta_d}\right)$.
\end{lemma}

\begin{remark}\label{rmk:compos}
In Lemma~\ref{lemma:compose_tv}, for $\rv{\cH}$, we required the stronger notion of cover with respect to $\rv{\cX_d}$ (i.e., the input to the hypotheses can be any random variable with a density function), whereas for $\rv{\cF}$ a cover with respect to $\rv{\Delta_d}$ sufficed in some cases. As we will see below, finding a cover with respect to $\rv{\Delta_d}$ is easier since one can reuse conventional $\|.\|_2$ covers. However, finding covers with respect to $\rv{\cX_d}$ is more challenging. In the next section we show how to do this for a class of neural networks.
\end{remark}
The next step is bounding the uniform covering number with respect to the TV distance (TV covering number for short). It will be useful to be able to bound TV covering number with Wasserstein covering number. However, this is generally impossible since closeness in Wasserstein distance does not imply closeness in TV distance. Yet, the following theorem establishes that one can bound the TV covering number as long as some Gaussian noise is added to the output of the hypotheses.

\begin{theorem}[From a Wasserstein cover to a TV cover]\label{thm:wasser_tv}
Let $\rv{\cF}$ be a class of random hypotheses from $\bR^d$ to $\bR^p$, and $\rv{\cG_{\sigma,p}}$ be a Gaussian noise class. Then for every $\epsilon>0$ and $m\in \bN$ we have  
\begin{equation*}
\begin{aligned}
     N_U\left(\frac{\epsilon}{2\sigma},\rv{\cG_{\sigma, p}}\circ\rv{\cF},m,d_{TV}^{{\infty}}, \rv{\cX_d}\right)\leq N_U(\epsilon,\rv{\cF},m,d_\cW^{{\infty}}, \rv{\cX_d}),\\
         N_U\left(\frac{\epsilon}{2\sigma},\rv{\cG_{\sigma, p}}\circ\rv{\cF},m,d_{TV}^{{\infty}}, \rv{\Delta_d}\right)\leq N_U(\epsilon,\rv{\cF},m,d_\cW^{{\infty}}, \rv{\Delta_d}).
\end{aligned}
\end{equation*}
\end{theorem}
Intuitively, the Gaussian noise smooths out densities of random variables that are associated with applying transformation in $\rv{\cF}$ to random variables in $\rv{\cX_d}$ or $\rv{\Delta_d}$. As a result, the proof of Theorem~\ref{thm:wasser_tv} has a step on relating the Wasserstein distance between two smoothed (by adding random Gaussian noise) densities to their total variation distance (see Lemma~\ref{lemma:wasser_conv}). Finally, we can use Proposition~\ref{prop12} to relate the Wasserstein covering number with the $\|.\|_2$ covering number. The following corollary is the result of Proposition~\ref{prop12} and Theorem~\ref{thm:wasser_tv} that is stated for both $d_{TV}^{\ell_2}$ and $d_{TV}^{\infty}$ extended metrics.

\begin{corollary}[From a $\|.\|_2$ cover to a TV cover]\label{coll:ell2_to_tv}
Let $\cF$ be a class of hypotheses $f:\bR^d\rightarrow\bR^p$ and $\rv{\cG_{\sigma,p}}$ be a Gaussian noise class. Then for every $\epsilon>0$ and $m\in \bN$ we have
\begin{equation*}
\begin{aligned}
     N_U(\frac{\epsilon}{2\sigma},\rv{\cG_{\sigma,p}}\circ\cF,m,d_{TV}^{{\infty}},\rv{\Delta_d})\leq N_U(\epsilon,\cF,m,\|.\|_2^{{\infty}}),\\
    {N_U(\frac{\epsilon}{2\sigma},\rv{\cG_{\sigma,p}}\circ\cF,m,d_{TV}^{\ell_2},\rv{\Delta_d})\leq N_U(\epsilon,\cF,m,\|.\|_2^{\ell_2})}.
\end{aligned}
\end{equation*}
\end{corollary}
The following theorem shows that we can get a stronger notion of TV cover with respect to $\rv{\cX_{B,d}}$ from a $\|.\|_2$ global cover, given that some Gaussian noise is added to the output of hypotheses. 
\begin{theorem}[From a global $\|.\|_2$ cover to a global TV cover]\label{thm:global_ell2_to_tv}
Let $\cF$ be a class of hypotheses $f:\bR^d\rightarrow\bR^p$ and $\rv{\cG_{\sigma,p}}$ be a Gaussian noise class. Then for every $\epsilon>0$ and $m\in \bN$ we have
\begin{equation*}
    N_U\left(\frac{\epsilon}{2\sigma},\rv{\cG_{\sigma,p}}\circ\cF,\infty,d_{TV}^{\infty},\rv{\cX_{B,d}}\right)\leq N_U(\epsilon,\cF,\infty,\|.\|_2^{\infty}).
\end{equation*}
\end{theorem}
The proof involves finding a Wasserstein covering number and using Theorem~\ref{thm:wasser_tv} to obtain TV covering number.

\section{Uniform TV covers for single-layer neural networks}\label{sub:NN}

In this section, we study the uniform covering number of single-layer neural networks with respect to the total variation distance. This will set the stage for the next section, where we want to use the tools from Section~\ref{sec:results} to bound covering numbers of deeper networks. We start with the following definition for the class of single-layer neural networks.
\begin{definition}[Single-Layer Sigmoid Neural Networks]\label{def:neuralnet}
Let $\Phi:\bR^p\rightarrow {[0,1]}^p$ be the element-wise sigmoid activation function defined by $\Phi((x^{(1)},\ldots, x^{(p)}))=(\phi(x^{(1)}),\ldots, \phi(x^{(p)}))$, where $\phi(x)=\frac{1}{1+e^{-x}}$ is the sigmoid function.
The class of single-layer neural networks with $d$ inputs and $p$ outputs is defined by $\net{d}{p}{L}=\{f_W:\bR^d\to[0,1]^p \mid f_W(x)=\Phi(W\transpose x), W\in\bR^{d\times p}\}$.  
\end{definition}
\begin{remark}
We choose sigmoid function for simplicity, but our analysis for finding uniform covering numbers of neural networks (Theorem~\ref{lemma:net}) is not specific to the sigmoid activation function. We present a stronger version of Theorem~\ref{lemma:net} in Appendix~\ref{app:sub:NN} which works for any activation function that is Lipschitz, monotone, and bounded.
\end{remark}
As mentioned in Remark~\ref{rmk:compos}, Lemma~\ref{lemma:compose_tv} requires stronger notion of covering numbers with respect to $\rv{\cX_d}$ and TV distance. In fact, the size of this kind of cover is infinite for deterministic neural networks defined above. In contrast, Theorem~\ref{lemma:net} shows that one can bound this covering number as long as some Gaussian noise is added to the input and output of the network. The proof is quite technical, starting with estimating the smoothed input distribution ($\rv{g_{\sigma}}(x)$) with mixtures of Gaussians using kernel density estimation (see Lemma~\ref{lema:gmm} in Appendix~\ref{app:gmm}). Then a cover for mixtures of Gaussians with respect to Wasserstein distance is found. Finally, Theorem~\ref{thm:wasser_tv} helps to find the cover with respect to total variation distance. For a complete proof of theorem see Appendix~\ref{app:sub:NN}.

\begin{theorem}[A global total variation cover for noisy neural networks with unbounded weights]\label{lemma:net} 
For every $p,d\in \bN, \epsilon>0, \sigma<5d/\epsilon$ we have
\begin{align*}
   N_U(\epsilon,\rv{\cG_{\sigma}}\circ\net{d}{p}{L},\infty,d_{TV}^{\infty},\rv{\cG_\sigma}\circ\rv{\cX_{1,d}})\leq  \left(30\frac{d^{5/2}\sqrt{\ln\left((5d-\epsilon\sigma)/(\epsilon\sigma)\right)}}{\epsilon^{3/2}\sigma^2}\ln\left(\frac{5d}{\epsilon\sigma}\right)\right)^{p(d+1)}.
\end{align*}
\end{theorem}

Note that the dependence of the bound on $1/\sigma$ is polynomial. The assumption $\sigma\ll 5d/\epsilon$ holds for any reasonable application (we will use $\sigma\ll1$ in the experiments). In contrast to the analyses that exploit Lipschitz continuity, the above theorem does not require any assumptions on the norms of weights. 
Theorem~\ref{lemma:net} is a key tool in analyzing the uniform covering number of deeper networks. 
\begin{remark}
Another approach to find a TV cover for neural networks is to find ``global'' $\|.\|_2$ covers and apply Theorem~\ref{thm:global_ell2_to_tv}. We know of only one such bound for neural networks with real-valued output in the literature, i.e., Lemma~14.8 in \citet{anthony1999neural}. 
This bound can be translated to multi-output layers (see Lemma~\ref{lemma:real_high} in Appendix~\ref{app:coverapproaches}).
However, unlike Theorem~\ref{lemma:net}, the final bound would depend on the norms of weights of the network and requires Lipschitzness assumption. 
\end{remark}

\section{Uniform covering numbers for deeper networks}\label{sec:applications}
In the following, we discuss how one can use Theorem~\ref{lemma:net} and techniques provided in Section~\ref{sec:results} to obtain bounds on covering number for deeper networks. For a $T$-layer neural network, it is useful to separate the first layer from the rest of the network. The following theorem offers a bound on the uniform covering number of (the expectation of) a noisy network based on the usual $\|.\|_2^{\ell_2}$ covering number of the first layer and the TV covering number of the subsequent layers.
\begin{theorem}\label{thm:neuralnet}
    Let $\net{d}{p_1}{L_1},\net{p_1}{p_2}{L_2},\ldots,\net{p_{T-1}}{p_T}{L_T}$ be $T$ classes of neural networks. Denote the $T$-layer noisy network by
    \begin{equation*}
        \rv{\cF}=\rv{\cG_{\sigma}}\circ\net{p_{T-1}}{p_T}{L_T}     \circ \ldots \circ \rv{\cG_{\sigma}}\circ\net{p_1}{p_2}{L_2}     \circ    \rv{\cG_{\sigma}}\circ\net{d}{p_1}{L_1},
    \end{equation*}
    and let $\cH=\{h:\bR^d\to [0,1]^{p_T} \mid h(x)=\expects{\rv{f}}{~\rv{f}({x})}, \rv{f}\in \rv{\cF} \}$.
    Denote the uniform covering numbers of compositions of neural network classes with the Gaussian noise class (with respect to $d_{TV}^{\infty}$) as
    \begin{equation}\label{eq:thm:neuralnet}
        N_i=N_U\left(\frac{\epsilon}{T\sqrt{p_T}},\rv{\cG_{\sigma}}\circ\net{p_{i-1}}{p_i}{L_i},\infty,d_{TV}^{\infty},\rv{\cG_\sigma}\circ\rv{\cX_{1,p_{i-1}}}\right), \,\, 2\leq i\leq T,
    \end{equation}
    and the uniform covering number of $\rv{\cG_{\sigma}}\circ\net{d}{p_1}{L_1}$ with respect to $\|.\|_2^{\ell_2}$ as
    \begin{equation*}
        N_1 = N_U\left(\frac{2\sigma\epsilon}{T\sqrt{p_T}},\net{d}{p_1}{L_1},m,\|.\|_2^{\ell_2}\right).
    \end{equation*}
    Then we have
    \begin{equation*}
        N_U\left(\epsilon,\cH,m,\|.\|_2^{\ell_2}\right)\leq \prod_{i=1}^T N_i.
    \end{equation*}
\end{theorem}
The proof of Theorem~\ref{thm:neuralnet} involves applying Corollary~\ref{coll:ell2_to_tv} to turn the $\|.\|_2$ cover of first layer into a TV cover. We then find a TV cover for rest of the network by applying Lemma~\ref{lemma:compose_tv} recursively to compose all the other layers. We will compose the first layer with the rest of the network and bound the covering number by another application of Lemma~\ref{lemma:compose_tv}. Finally, we turn the TV covering number (of the entire network) back into $\|.\|_2^{\ell_2}$ covering number using Theorem~\ref{thm:tv_to_ell2}. The complete proof can be found in Appendix~\ref{app:sec:application}. The above bound does not depend on the norm of weights and therefore we can use it for networks with large weights.

The $\|.\|_2^{\ell_2}$ covering number of the first layer (i.e., $N_1$ in above) can be bounded using standard approaches in the literature. For instance, in the following corollary we will use the bound of Lemma 14.7 in \citet{anthony1999neural}. Other $N_i$'s can be bounded using Theorem~\ref{lemma:net}. The proof can be found in Appendix~\ref{app:coverapproaches}.

\begin{corollary}[Covering number bound of Theorem~\ref{thm:neuralnet}]\label{thm:ours_main}
 Let $\net{d}{p_1}{L_1},\net{p_1}{p_2}{L_2},\ldots,\net{p_{T-1}}{p_T}{L_T}$ be $T$ classes of neural networks. Denote the $T$-layer noisy network by
    \begin{equation*}
        \rv{\cF}=\rv{\cG_{\sigma}}\circ\net{p_{T-1}}{p_T}{L_T}     \circ \ldots \circ \rv{\cG_{\sigma}}\circ\net{p_1}{p_2}{L_2}     \circ    \rv{\cG_{\sigma}}\circ\net{d}{p_1}{L_1},
    \end{equation*}
    and let $\cH=\{h:\bR^d\to [0,1]^{p_T} \mid h(x)=\expects{\rv{f}}{~\rv{f}({x})}, \rv{f}\in \rv{\cF} \}$.
    Then we have
    \begin{equation*}
    \begin{aligned}
            &\ln N_U\left(\epsilon,\cH,m,\|.\|_2^{\ell_2}\right)\\
            &\leq \sum_{i=2}^T p_i.p_{i-1}\ln\left(30\frac{(T\sqrt{p_T})^{3/2}p_{i-1}^{5/2}\sqrt{\ln\left(\frac{\displaystyle 5T\sqrt{p_T}p_{i-1}-\epsilon\sigma}{\displaystyle \epsilon\sigma} \right)}}{\epsilon^{3/2}\sigma^2}\ln\left(\frac{5Tp_{i-1}\sqrt{p_T}}{\epsilon\sigma}\right)\right)\\
            & + dp_1\ln\left(\frac{Tem\sqrt{p_T}}{2\epsilon\sigma}\right).
    \end{aligned}
    \end{equation*}
\end{corollary}

One can generalize the above analysis in the following way: instead of separating the first layer, one can basically ``break'' the network from any layer, use existing $\|.\|_2$ covering number bounds for the first few layers, and Theorem~\ref{lemma:net} for the rest. See Lemma~\ref{lemma:compos_old} in Appendix~\ref{app:sec:application} for details.

\subsection{Analyzing different covering number bounds}\label{subsec:cover_bounds}
\begin{table}[h!]
  \centering
\begin{tabular}{|c|c|c|}
      \hline
      \rule{0pt}{3ex}   
    Approach & Logarithm of covering number: $\ln N_U(\epsilon,\cF,m,\|.\|_2^{\ell_2})$ & Nature\\
    \hline
    \rule{0pt}{4ex}   
     Corollary~\ref{thm:ours_main} & $O\left(W_{win}\ln\left(\frac{(T\sqrt{p_{T}})^{3/2}d_{max}^{5/2}}{\epsilon^{3/2}\sigma^2}\right) + d_{max}d\ln(\frac{mT\sqrt{p_{T}}}{\epsilon\sigma})\right)$ & $\mathcal{MOL}$\\[2ex]
     \hline
     \rule{0pt}{4ex}
     Norm-based (Theorem~\ref{thm:norm_based}) & $O\left( \left( \frac{1}{\epsilon}\right)^{2T}\left({p_{T}}\right)^{T+1}(2V)^{T^2+T}\log_2(2d)\right)$ & $\mathcal{RVO}$\\[2ex]
     \hline
     \rule{0pt}{4ex}
     Pseudo-dim-based (Theorem~\ref{thm:pdim_based}) & $O\left( p_{T} (W_{rvo}r_{rvo})^2\ln\left( \frac{m\sqrt{p_{T}}}{(W_{rvo}r_{rvo})^2\epsilon}\right)\right)$ & $\mathcal{RVO}$ \\[2ex]
     \hline
     \rule{0pt}{4ex}
     Lipschitzness-based (Theorem~\ref{thm:lipsch-based}) & $O\left(p_{T}W_{rvo}\ln\left( \frac{m\sqrt{p_{T}}W_{rvo}V^T}{\epsilon(V-1)}\right)\right)$ & $\mathcal{RVO}$\\[2ex]
     \hline
     \rule{0pt}{4ex}
     Spectral (Theorem~\ref{thm:spectral}) & $O\left( \frac{\|X\|_F^2\ln(w^2)}{\epsilon^2} \left(\prod_{i=1}^Ts_i^2\right)\left(\sum_{i=1}^T(\frac{b_i}{s_i})^{2/3}\right)^3\right)$ & $\mathcal{MOL}$\\[2ex]
     \hline
\end{tabular}
    \caption{Covering number of a $T$-layer sigmoid network from $\bR^d$ to $\bR^{p_T}$ defined by $\cF=\net{p_{T-1}}{p_T}{L_T}     \circ \ldots \circ\net{p_1}{p_2}{L_2}    \circ\net{d}{p_1}{L_1}$. Corollary~\ref{thm:ours_main} is computed on the $T$-layer noisy sigmoid network. $\|X\|_F$ denotes the normalized Frobenious norm of input matrix $X\in\bR^{d\times m}$ (see Appendix~\ref{app:coverapproaches} for more details). The definition of other quantifiers used in these bounds can be found in Table~\ref{table:ref}.}\label{table:cover}
\end{table}

\begin{table}[h!]
  \centering
\begin{tabular}{|c|c|c|}
      \hline
    Quantifier & Definition & Description\\
    \hline
    \rule{0pt}{3ex}
    $d_{max}$ & $\max_{1\leq i <T-1} p_i$ & Maximum number of neurons in a hidden layer\\[1ex]
     \hline
    \rule{0pt}{3ex}
    $W_{rvo}$ & $dp_1+\sum_{i=2}^{T-1}p_i.p_{i-1}+p_{T-1}$ & \makecell{Total number of parameters of the real-valued networks \\corresponding to each dimension of the output}\\[1ex]
     \hline
    \rule{0pt}{3ex}
    $W_{win}$ & $\sum_{i=2}^{T}p_i.p_{i-1}$ & \makecell{Total number of parameters excluding \\the weights between input and first hidden layers}\\[1ex]
    \hline
    \rule{0pt}{3ex}
    $r_{rvo}$ & $1+\sum_{i=1}^{T-1}p_i$ & \makecell{Total number of neurons in all but the input layer of \\the real-valued networks corresponding\\ to each dimension of the output}\\[1ex]
    \hline
    \rule{0pt}{3ex}
    $w$ & $\max\left\{d,d_{max},p_T\right\}$ & Maximum number of neurons in all layers of the network\\[1ex]
        \hline
    \rule{0pt}{3ex}
    $V$ & $\max _{1\leq i \leq  T}\|W_i\|_{1,\infty}$ & Maximum  $\ell_1$ norm of incoming weights to a neuron\\[1ex]
        \hline
    \rule{0pt}{3ex}
    $s_i$ & $\|W_i\|_{\sigma}$ & Spectral norm of $W_i$\\[1ex]
        \hline
    \rule{0pt}{3ex}
    $b_i$ & $\|W_i\|_{2,1}$ & $\|.\|_{2,1}$ norm of $W_i$ (see Appendix~\ref{app:coverapproaches})\\[1ex]
    \hline
    \end{tabular}
    \caption{Definition of quantifiers used in Table~\ref{table:cover}. Here, $W_i\in\bR^{p_{i-1}\times p_{i}}$ denotes the weight vector associated with $\net{p_{i-1}}{p_i}{}$ for $2\leq i\leq T$ and $W_1\in\bR^{d\times p_1}$ is the weight vector associated with $\net{d}{p_1}{}$. It is noteworthy that the total number of parameters of the network, $dp_1+\sum_{i=2}^{T}p_i.p_{i-1}$, is always smaller than $ p_TW_{rvo}$.}\label{table:ref}
\end{table}
For the remainder of this section we will qualitatively compare some of the approaches in finding covering number with our approach in Corollary~\ref{thm:ours_main}. (Later in the next section we will propose a quantitative metric (see Definition~\ref{def:nvac}) to compare these approaches based on their suggested generalization bounds). Particularly, we compare the following approaches: Corollary~\ref{thm:ours_main}, Norm-based (Theorem 14.17 in \citet{anthony1999neural}), Lipschitzness-based (Theorem 14.5 in \citet{anthony1999neural}), Pseudo-dim-based (Theorem 14.2 in \citet{anthony1999neural}), and Spectral (\citet{bartlett2017spectrally}). Some of these bounds work naturally for multi-output layers, which we will denote by $\mathcal{MOL}$, while some of them are derived for real-valued outputs, which we denote by $\mathcal{RVO}$. In Appendix~\ref{app:coverapproaches}, we recommend one possible approach to turn $\mathcal{RVO}$ covering number bounds into $\mathcal{MOL}$ bounds (Lemma~\ref{lemma:real_high}). A simplified form of these bounds is presented in Table~\ref{table:cover}. More details about these bounds and their exact forms can be found in Appendix~\ref{app:coverapproaches}.

{\bf Qualitative comparison of bounds on the logarithm of covering number.} In the following, whenever we reference to an approach, we are considering the logarithm of the covering number. For the definition of quantifiers see Table~\ref{table:ref}. Corollary~\ref{thm:ours_main} and Pseudo-dim-based bounds have no dependence on the norms of weights while Norm-based and Spectral bounds mostly depend (polynomially) on the norms with Spectral having a slight dependence of $\ln(w^2)$ on the size of the network. Pseudo-dim-based bound has the worst dependence on the size of the network, i.e., $\wtilde{O}\left(p_T\left(d_{max}^3+d_{max}^2d\right)^2\right)$, where $\wtilde{O}$ hides logarithmic factors. On the other hand, comparing Corollary~\ref{thm:ours_main} and Lipschitzness-based bound requires more attention. In terms of dependence on the size of the network, Corollary~\ref{thm:ours_main} and Lipschitzness-based bound are incomparable: Corollary~\ref{thm:ours_main} has a dependence of $\wtilde{O}\left(d_{max}^2+d_{max}d\right)$ while  Lipschitzness-based bound depends on $\wtilde{O}\left(p_T(d_{max}^2+d_{max}d)\right)$. However, in contrast to Corollary~\ref{thm:ours_main}, Lipschitzness-based bound has a dependence on the norms of the weights. Another important dependence is on $1/\epsilon$. Corollary~\ref{thm:ours_main} has a logarithmic dependence on $1/\epsilon$. While Lipschitzness-based and Pseudo-dim-based bounds also enjoy the logarithmic dependence, the Norm-based and Spectral bounds depend polynomially on $1/\epsilon$. It is also worth mentioning that Corollary~\ref{thm:ours_main}, Pseudo-dim-based and Lipschitzness-based bounds depend on $\ln(m)$. The empirical results of Section~\ref{sec:experiments} suggest that Corollary~\ref{thm:ours_main} can outperform all other bounds including Lipschitzness-based bound.

\section{NVAC: a metric for comparing generalization bounds}\label{sec:nvac}
We want to provide tools to compare different approaches in finding covering numbers and their suggested generalization bounds. First, we define the notion of a generalization bound for classification.
Let $\cY=[k]$ and $\cF$ be a class of functions from $\cX$ to $\bR^k$. Let $\cA$ be an algorithm that receives a labeled sample $S=\left((x_1,y_1),\ldots,(x_m,y_m)\right) \in (\cX \times \cY)^m$ and outputs a function $\hat{h}\in\cF$. Note that the output of this function is a real vector so it can capture margin-based classifiers too.
Let $l^{0-1}:\bR^k\times [k] \to \{0,1\}$ be the ``thresholded'' 0-1 loss function defined by $l^{0-1}(u,y)=\indicator\{{\text{argmax}}_iu^{(i)}\neq y\}$ where $u^{(i)}$ is the $i$-th dimension of $u$. 
\begin{definition}[Generalization Bound for Classification]\label{def:gb}
A (valid) generalization bound for $\cA$ with respect to $l^{0-1}$ and another (surrogate) loss function $l$ is a function $\text{GB}:\cF\times (\cX\times\cY)^m\to \bR$ such that for every distribution $\cD$ over $\cX\times \cY$, if $S\sim\cD^m$, then with probability at least 0.99 (over the randomness of $S$) we have
\begin{equation*}
    \left|\frac{1}{m}\sum_{(x,y)\in S}l(\hat{h}(x),y) - \expects{(x,y)\sim \cD}{l^{0-1}(\hat{h}(x),y)}\right| \leq \text{GB}(\hat{h},S).
\end{equation*}
\end{definition}
For example, $GB(\hat{h}, S)=2$ is a useless but valid generalization bound. Various generalization bounds that have been proposed in the literature are examples of a $GB$. Note that $GB$ can depend both on $S$ (for instance on $|S|$) and on $\hat{h}$ (for example, on the norm of the weights of network). 

It is not straightforward to empirically compare generalization bounds since they are often vacuous for commonly used applications. \citet{jiang2019fantastic} address this by looking at other metrics, such as the correlation of each bound with the actual generalization gap. While these metrics are informative, it is also useful to know how far off each bound is from producing a ``non-vacuous'' bound~\citep{DR17}. Therefore, we will take a more direct approach and propose the following metric. 
\begin{definition}[NVAC]\label{def:nvac} Let $\hat{h}$ be a hypothesis, $S\in(\cX\times\cY)^m$ a sample,  and $GB$ a generalization bound for algorithm $\cA$. Let $S^n$ denote a sample of size $mn$ which includes $n$ copies of $S$. Let $n^*$ be the smallest integer such that the following holds:
\begin{equation*}
    \text{GB}(\hat{h},S^{n^*}) + \frac{1}{|S^{n^*}|}\sum_{(x,y)\in S^{n^*}}l(\hat{h}(x),y) \leq 1.
\end{equation*}
We define NVAC to be ${|S^{n^*}|=mn^*}$.
\end{definition}
Informally speaking, NVAC is an upper bound on the minimum number of samples required to obtain a non-vacuous generalization bound. Approaches that get tighter upper bounds on covering number will generally result in smaller NVACs. In Appendix~\ref{app:nvacfromCover}, we will show how one can calculate NVAC using the uniform covering number bounds.

\section{Experiments}\label{sec:experiments}
In Section~\ref{subsec:cover_bounds} we qualitatively compared different approaches in bounding covering number. In this section, we empirically compare the exact form of these bounds (see Appendix~\ref{app:coverapproaches}) using the NVAC metric. 

We train fully connected neural networks on MNIST dataset. We use a network with an input layer, an output layer, and three hidden layers each containing $250$ hidden neurons as the baseline architecture. See Appendix~\ref{app:nvacgraphs} for the details of the learning settings. The left two graphs in Figure~\ref{fig:nvac_all} depict NVACs as functions of the depth and width of the network. It can be observed that our approach achieves the smallest NVAC. The Norm-based bound is the worst and is removed from the graph (see Appendix~\ref{app:nvacgraphs}). Overall, bounds that are based on the norm of the weights (even the spectral norm) perform poorly compared to those that are based on the size of the network. This is an interesting observation since we have millions of parameters ($\approx 3\times 10^9$) in some of the wide networks and one would assume approaches based on norm of weights should be able to explain generalization behaviour better. Aside from the fact that our bound does not have any dependence on the norms of weights, there are several reasons why it performs better. First, the NVAC in Spectral and Norm-based approaches have an extra polynomial dependence on $1/\epsilon$, compared to all other approaches. Moreover, these bounds depend on product of norms and group norms which can get quite large. Finally, our method works naturally for multi-output layers, while the Pseudo-dim-based, Lipschitzness-based, and Norm-based approaches work for real-valued output (and therefore one needs to bound the cover for each output separately).

The covering number bound of Corollary~\ref{thm:ours_main} has a polynomial dependence on $1/\sigma$. Therefore, NVAC has a mild logarithmic dependence on $1/\sigma$ (see Appendix~\ref{app:nvacfromCover} for details). The third graph in Figure~\ref{fig:nvac_all} corroborates that even a negligible amount of noise ($\sigma \approx 10^{-240}$) is sufficient to get tighter bounds on NVAC compared to other approaches.
Finally, the right graph in Figure~\ref{fig:nvac_all} shows that even with a considerable amount of noise (e.g, $\sigma=0.2$), the train and test accuracy of the model remain almost unchanged. This is perhaps expected, as the dynamics of training neural networks with gradient descent is already noisy even without adding Gaussian noise. Overall, our preliminary experiment shows that small amount of noise does not affect the performance, yet it enables us to prove tighter generalization bounds.

 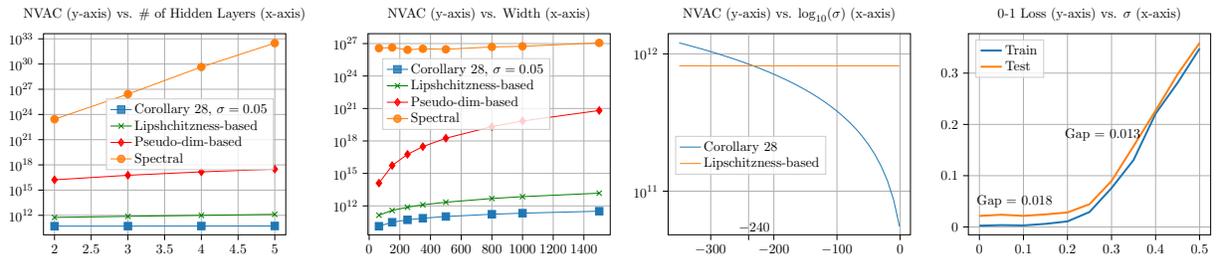
\begin{figure}
    \centering
    \subfloat{
\begin{tikzpicture}[scale=0.47]

\definecolor{darkgray176}{RGB}{176,176,176}
\definecolor{darkorange25512714}{RGB}{255,127,14}
\definecolor{green01270}{RGB}{0,127,0}
\definecolor{lightgray204}{RGB}{204,204,204}
\definecolor{steelblue31119180}{RGB}{31,119,180}

\begin{axis}[
legend cell align={left},
legend style={
  fill opacity=0.8,
  draw opacity=1,
  text opacity=1,
  at={(0.95,0.5)},
  anchor=east,
  draw=lightgray204
},
log basis y={10},
tick align=outside,
tick pos=left,
title={NVAC (y-axis) vs. \# of Hidden Layers (x-axis)},
x grid style={darkgray176},
xlabel={},
xmajorgrids,
xmin=1.85, xmax=5.15,
xtick style={color=black},
y grid style={darkgray176},
ylabel={},
ymajorgrids,
ymin=4791534854.11726, ymax=4.10693389987566e+33,
ymode=log,
ytick style={color=black},
ytick={1000000,1000000000,1000000000000,1e+15,1e+18,1e+21,1e+24,1e+27,1e+30,1e+33,1e+36,1e+39},
yticklabels={
  \(\displaystyle {10^{6}}\),
  \(\displaystyle {10^{9}}\),
  \(\displaystyle {10^{12}}\),
  \(\displaystyle {10^{15}}\),
  \(\displaystyle {10^{18}}\),
  \(\displaystyle {10^{21}}\),
  \(\displaystyle {10^{24}}\),
  \(\displaystyle {10^{27}}\),
  \(\displaystyle {10^{30}}\),
  \(\displaystyle {10^{33}}\),
  \(\displaystyle {10^{36}}\),
  \(\displaystyle {10^{39}}\)
}
]
\addplot [semithick, steelblue31119180, mark=square*, mark size=3, mark options={solid}]
table {%
2 51348186946.4345
3 51897254520.7146
4 52940188934.7658
5 54163313232.3868
};
\addlegendentry{Corollary~28, $\sigma=0.05$}
\addplot [semithick, green01270, mark=x, mark size=3, mark options={solid}]
table {%
2 561348889381.558
3 756692656869.744
4 998775160950.438
5 1300399857111.24
};
\addlegendentry{Lipshchitzness-based}
\addplot [semithick, red, mark=diamond*, mark size=3, mark options={solid}]
table {%
2 1.67727514437642e+16
3 5.81160033330929e+16
4 1.47368111391297e+17
5 3.11306486643974e+17
};
\addlegendentry{Pseudo-dim-based}
\addplot [semithick, darkorange25512714, mark=*, mark size=3, mark options={solid}]
table {%
2 2.70414353980701e+23
3 2.65862473449465e+26
4 4.48382303730719e+29
5 3.07972375913828e+32
};
\addlegendentry{Spectral}
\end{axis}

\end{tikzpicture}}
    \subfloat{\begin{tikzpicture}[scale=0.47]

\definecolor{darkgray176}{RGB}{176,176,176}
\definecolor{darkorange25512714}{RGB}{255,127,14}
\definecolor{green01270}{RGB}{0,127,0}
\definecolor{lightgray204}{RGB}{204,204,204}
\definecolor{steelblue31119180}{RGB}{31,119,180}

\begin{axis}[
legend cell align={left},
legend style={
  fill opacity=0.8,
  draw opacity=1,
  text opacity=1,
  at={(0.06,0.7)},
  anchor=west,
  draw=lightgray204
},
log basis y={10},
tick align=outside,
tick pos=left,
title={NVAC (y-axis) vs. Width (x-axis)},
x grid style={darkgray176},
xlabel={},
xmajorgrids,
xmin=-7.8, xmax=1571.8,
xtick style={color=black},
xticklabel style={/pgf/number format/1000 sep=},
y grid style={darkgray176},
ylabel={},
ymajorgrids,
ymin=2120642490.60856, ymax=8.21223903613458e+27,
ymode=log,
ytick style={color=black},
ytick={1000000,1000000000,1000000000000,1e+15,1e+18,1e+21,1e+24,1e+27,1e+30,1e+33},
yticklabels={
  \(\displaystyle {10^{6}}\),
  \(\displaystyle {10^{9}}\),
  \(\displaystyle {10^{12}}\),
  \(\displaystyle {10^{15}}\),
  \(\displaystyle {10^{18}}\),
  \(\displaystyle {10^{21}}\),
  \(\displaystyle {10^{24}}\),
  \(\displaystyle {10^{27}}\),
  \(\displaystyle {10^{30}}\),
  \(\displaystyle {10^{33}}\)
}
]
\addplot [semithick, steelblue31119180, mark=square*, mark size=3, mark options={solid}]
table {%
64 13205924003.3771
150 30686361777.6472
250 51216675135.9733
350 72268338983.3325
500 104064325981.589
800 168856734692.262
1000 211558469685.488
1500 322519292017.711
};
\addlegendentry{Corollary~28, $\sigma=0.05$}
\addplot [semithick, green01270, mark=x, mark size=3, mark options={solid}]
table {%
64 137766868695.87
150 380186511096.053
250 756692656869.744
350 1247122380783.46
500 2185495634834.64
800 4809059777031.86
1000 7131565315725.06
1500 14937255156100.9
};
\addlegendentry{Lipshchitzness-based}
\addplot [semithick, red, mark=diamond*, mark size=3, mark options={solid}]
table {%
64 126913696006677
150 5.37782869049592e+15
250 5.81160033330929e+16
350 2.97997493555181e+17
500 1.79263869708445e+18
800 2.09689808144143e+19
1000 6.98024506489934e+19
1500 6.52683570575932e+20
};
\addlegendentry{Pseudo-dim-based}
\addplot [semithick, darkorange25512714, mark=*, mark size=3, mark options={solid}]
table {%
64 3.7956495693336e+26
150 4.33315129139668e+26
250 2.65862473449465e+26
350 3.29559739224765e+26
500 2.97220906335308e+26
800 4.87732395291279e+26
1000 5.58337408836247e+26
1500 1.15447141003031e+27
};
\addlegendentry{Spectral}
\end{axis}

\end{tikzpicture}}
    \subfloat{
\begin{tikzpicture}[scale=0.47]

\definecolor{darkgray176}{RGB}{176,176,176}
\definecolor{darkorange25512714}{RGB}{255,127,14}
\definecolor{lightgray204}{RGB}{204,204,204}
\definecolor{steelblue31119180}{RGB}{31,119,180}

\begin{axis}[
legend cell align={left},
legend style={
  fill opacity=0.8,
  draw opacity=1,
  text opacity=1,
  at={(0.03,0.3)},
  anchor=south west,
  draw=lightgray204
},
log basis y={10},
tick align=outside,
tick pos=left,
title={NVAC (y-axis) vs. $\log_{10}(\sigma)$ (x-axis)},
x grid style={darkgray176},
xlabel={},
xmajorgrids,
xmin=-367.45, xmax=16.45,
xtick style={color=black},
extra x ticks = {-240},
extra x tick labels = {$\displaystyle {-240}$},
extra x tick style={tick label style={black, above, yshift=0.5ex,xshift = 1ex}}, 
y grid style={darkgray176},
ylabel={},
ymajorgrids,
ymin=48116445600.4435, ymax=1403874506518.26,
ymode=log,
ytick style={color=black},
ytick={1000000000,10000000000,100000000000,1000000000000,10000000000000,100000000000000},
yticklabels={
  \(\displaystyle {10^{9}}\),
  \(\displaystyle {10^{10}}\),
  \(\displaystyle {10^{11}}\),
  \(\displaystyle {10^{12}}\),
  \(\displaystyle {10^{13}}\),
  \(\displaystyle {10^{14}}\)
}
]
\addplot [semithick, steelblue31119180]
table {%
-350 1202575100826.28
-340 1169818942720.83
-330 1137061330873.22
-320 1104302208430.04
-310 1071541530768.81
-300 1038779111372.96
-290 1006014939873.49
-280 973248798430.389
-270 940480642154.419
-260 907710209903.461
-250 874937437302.547
-240 842162129414.239
-230 809384036228.212
-220 776602955711.312
-210 743818580854.64
-200 711030680988.774
-190 678238857771.857
-180 645442752962.077
-170 612641891522.535
-160 579835749336.43
-150 547023659477.148
-140 514204930773.753
-130 481378618610.674
-120 448543622981.957
-110 415698623198.75
-100 382841854946.518
-90 349971187563.575
-80 317083715449.454
-70 284175612968.699
-60 251241527568.391
-50 218273692964.384
-40 185260220592.061
-30 152181450701.918
-20 119001229281.451
-10 85639909035.2751
-1.30102999566398 56372957173.8613
-1 55334196723.3854
};
\addlegendentry{Corollary~28}
\addplot [semithick, darkorange25512714]
table {%
-350 819158966931.956
-340 819158966931.956
-330 819158966931.956
-320 819158966931.956
-310 819158966931.956
-300 819158966931.956
-290 819158966931.956
-280 819158966931.956
-270 819158966931.956
-260 819158966931.956
-250 819158966931.956
-240 819158966931.956
-230 819158966931.956
-220 819158966931.956
-210 819158966931.956
-200 819158966931.956
-190 819158966931.956
-180 819158966931.956
-170 819158966931.956
-160 819158966931.956
-150 819158966931.956
-140 819158966931.956
-130 819158966931.956
-120 819158966931.956
-110 819158966931.956
-100 819158966931.956
-90 819158966931.956
-80 819158966931.956
-70 819158966931.956
-60 819158966931.956
-50 819158966931.956
-40 819158966931.956
-30 819158966931.956
-20 819158966931.956
-10 819158966931.956
-1.30102999566398 819158966931.956
-1 819158966931.956
};
\addlegendentry{Lipschitzness-based}
\end{axis}

\end{tikzpicture}}
    \subfloat{
\begin{tikzpicture}[scale=0.47]
\definecolor{darkgray176}{RGB}{176,176,176}
\definecolor{darkorange25512714}{RGB}{255,127,14}
\definecolor{lightgray204}{RGB}{204,204,204}
\definecolor{steelblue31119180}{RGB}{31,119,180}

\begin{axis}[
legend cell align={left},
legend style={
  fill opacity=0.8,
  draw opacity=1,
  text opacity=1,
  at={(0.03,0.97)},
  anchor=north west,
  draw=lightgray204
},
tick align=outside,
tick pos=left,
title={0-1 Loss (y-axis) vs. $\sigma$ (x-axis)},
x grid style={darkgray176},
xlabel={},
xmajorgrids,
xmin=-0.025, xmax=0.525,
xtick style={color=black},
y grid style={darkgray176},
ylabel={},
ymajorgrids,
ymin=-0.0153042395669967, ymax=0.376948904487863,
ytick style={color=black}
]
\addplot [ultra thick, steelblue31119180]
table {%
0 0.0025254487991333
0.05 0.0036021186709404
0.1 0.00300220310688015
0.15 0.00606749057769773
0.2 0.0108754231333733
0.25 0.0291773204505443
0.3 0.0758824227154254
0.35 0.130637828558683
0.4 0.220969605840743
0.45 0.281140586957335
0.5 0.348085655480623
};
\addlegendentry{Train}
\addplot [ultra thick, darkorange25512714]
table {%
0 0.0215000510215759
0.05 0.0239086238741875
0.1 0.0216412252783775
0.15 0.0245171236395836
0.2 0.0283142243027688
0.25 0.0444091235399247
0.3 0.0896480229496955
0.35 0.157704420901835
0.4 0.226935819379985
0.45 0.297850917667151
0.5 0.359119216121733
};
\addlegendentry{Test}
\node at (axis cs: 0.08,0.05) {Gap =  0.018};
\node at (axis cs: 0.28,0.18) {Gap =  0.013};
\end{axis}
\end{tikzpicture}}
    \caption{The left two graphs depict NVAC of different generalization bounds as a function of the number of hidden layers and width of the network. 
    The Norm-based approach is excluded because of its excessively high NVAC (see Appendix~\ref{app:nvacgraphs}).
    The third graph plots NVAC against $\log_{10}(\sigma)$ ($\sigma$ is standard deviation of noise) for the two best approaches. The rightmost graph plots the train/test 0-1 losses for different values of $\sigma$. The gaps between the train and test losses are shown for $\sigma=0,0.3$.}
    \label{fig:nvac_all}
\end{figure}

{\bf Limitations and Future Work.}
Our analysis is based on the assumption that the activation function is bounded. Therefore, extending the results to ReLU neural networks is not immediate, and is left for future work. Also, our empirical analysis is preliminary and is mostly used as a sanity check. Further empirical evaluations can help to better understand the role of noise in training neural networks.

\medskip

\bibliography{refs.bib}



\appendix

\section{Miscellaneous facts}
\begin{lemma}[Data processing inequality for TV distance]\label{lemma:DPI}
Given two random variables $\rv{x_1},\rv{x_2} \in \rv{\cX}$, and a (random) Borel function $f:\cX\rightarrow\cY$, 
\begin{equation*}
    d_{TV}(f(\rv{x_1}),f(\rv{x_2}))\leq d_{TV}(\rv{x_1},\rv{x_2}).
\end{equation*}
\end{lemma}

The next theorem, bounds the total variation distance between two Gaussian random variables.
\begin{theorem}[Total variation distance between Gaussians with same covariance]\label{thm:tv_gaussian}
Let $\cN(\mu_1,\sigma^2I_{d})$ and $\cN(\mu_2,\sigma^2I_{d})$ be two Gaussian random variables, where $I_d$ is the $d$-by-$d$ identity matrix. Then we have,
\begin{equation*}
    d_{TV}(\cN(\mu_1,\sigma^2I_{d}),\cN(\mu_2,\sigma^2I_{d}))\leq \frac{1}{2\sigma}\left\| \mu_1-\mu_2\right\|_2.
\end{equation*}
\end{theorem}
\begin{proof}
Form Pinsker's inequality we know that for any two distributions $P$ and $Q$ we have 
\begin{equation}\label{eq:a0}
    d_{TV}(P,Q)\leq \sqrt{\frac{1}{2}d_{KL}(P,Q)},
    \end{equation}
where $d_{KL}(P,Q)$ is the Kullback-Liebler (KL) divergence between $P$ and $Q$. We can then find the KL divergence between $\cN(\mu_1,\sigma^2I_{d})$ and $\cN(\mu_2,\sigma^2I_{d})$ as (see e.g., \citet{diakonikolas2019robust}) 
\begin{equation}\label{eq:a1}
    d_{KL}\left(\cN(\mu_1,\sigma^2I_{d}),\cN(\mu_2,\sigma^2I_{d})\right)\leq \frac{1}{2\sigma^2}\|\mu_1-\mu_2\|_2^2.
\end{equation}
Combining Equations~\ref{eq:a0} and \ref{eq:a1} concludes the result.
\end{proof}

\begin{lemma}\label{lemma:chi}
Let $Y\sim\chi_n^2$ be a chi-squared random variable with $n$ degrees of freedom. Then we have \citep{laurent2000adaptive}
\begin{equation*}
    \bP[Y-n\geq 2\sqrt{nt}+2t]\leq e^{-t}.
\end{equation*}
\end{lemma}
\begin{lemma}\label{lemma:norm_gaussian}
Let $x=\sum_{i=1}^m w_ig_i$ be a random variable, where $g_i$ are $d$-dimensional Gaussian random variables with means $\mu_i\in[-B,B]^d$ and covariance matrices of $\sigma^2I_d$. We have
\begin{equation*}
    \bP\left[\|x\|_2\geq (B+\sigma)\sqrt{d}+\sigma\sqrt{2t}\right] \leq e^{-t}.
\end{equation*}
\end{lemma}
\begin{proof}
We know that for any $R\in\bR$
\begin{equation*}
\begin{aligned}
     \bP\left[\|x\|_2^2\geq R^2\right] = \sum_{i=1}^m w_i\bP\left[\|g_i\|_2^2 \geq R^2\right] = \sum_{i=1}^m w_i\bP\left[\|\sigma y_i+\mu_i\|_2^2 \geq R^2\right]= \sum_{i=1}^m w_i\bP\left[\|\sigma y_i+\mu_i\|_2 \geq R\right],
\end{aligned}
\end{equation*}
where $y_i\sim \cN(0,I_d)$ are standard normal random variables. Using triangle inequality we can rewrite the above equation as
\begin{equation*}
    \begin{aligned}
      \bP\left[\|x\|_2^2\geq R^2\right] \leq \sum_{i=1}^m w_i\bP\left[\|\sigma y_i\|_2+\|\mu_i\|_2 \geq R\right]\leq \sum_{i=1}^m w_i\bP\left[\|\sigma y_i\|_2+ B\sqrt{d} \geq R\right].
    \end{aligned}
\end{equation*}
We can, therefore, conclude that
\begin{equation*}
    \bP\left[\|x\|_2^2\geq R^2\right] \leq \bP\left[\|y_i\|_2^2  \geq \left(\frac{R-B\sqrt{d}}{\sigma}\right)^2\right].
\end{equation*}
Setting $R=(B+\sigma)\sqrt{d}+\sigma\sqrt{2t}$, we can write
\begin{equation*}
    \begin{aligned}
    &\bP\left[\|x\|_2\geq (B+\sigma)\sqrt{d}+\sigma\sqrt{2t}\right]\\
     &=\bP\left[\|x\|_2^2\geq \left((B+\sigma)\sqrt{d}+\sigma\sqrt{2t}\right)^2\right]\\
     &\leq \bP\left[\|y_i\|_2^2  \geq (\sqrt{d}+\sqrt{2t})^2\right]\\
     &\leq  \bP\left[\|y_i\|_2^2  \geq d+2t+2\sqrt{dt}\right]\\
     &\leq e^{-t}.
    \end{aligned}
\end{equation*}
\end{proof}
\section{Proofs of propositions in Section~\ref{sec:notations}}\label{app:props}
\subsection{Proof of Proposition~\ref{prop7}}
\begin{proof}
Fix an input set $S=\{x_1,\ldots,x_m\}$. Let $C=\{\hat{f_i}_{|S}\mid\hat{f}_i\in\cF,i\in[r_1]\}$ be $0.5$-cover for $\cF_{|S}$ with respect to $\rho^{\infty}$. Therefore, given any $f_{|S}\in\cF_{|S}$ there exists $\hat{f_i}_{|S}\in C$ such that
\begin{equation}\label{eq:prop1}
    \rho^{\infty}\left( (f(x_1),\ldots,f(x_m)), (\hat{f_i}(x_1),\ldots,\hat{f_i}(x_m))\right)\leq 0.5
\end{equation}
Since $\rho\left(f(x),\hat{f_i}(x)\right)=\indicator\{f(x)\neq \hat{f_i}(x)\}$, Equation~\ref{eq:prop1} suggests that $f(x_k)=\hat{f_i}(x_k)$ for any $k\in[m]$. Let $S'=\{\hat{f_i}(x_k)\mid i\in[r_1],\,k\in[m]\}$ and $C'=\{\hat{h_j}_{|S'}\mid \hat{h}_j\in\cH,j\in[r_2]\}$ be an $\epsilon$-cover for $\cH_{|S'}$ with respect to $\|.\|_2^{\infty}$. We know that $|S'|\leq mr_1.$ Denote $\hat{\cQ}=\{\hat{h}_j\circ \hat{f_i}\mid i\in[r_1],\,j\in[r_2]\}$. We will prove that $\hat{\cQ}_{|S}$ is an $\epsilon$-cover for $(\cH\circ\cF)_{|S}$ with respect to $\|.\|_2^{\infty}$. Consider $(h\circ f)_{|S}=\left( h(f(x_1)),\ldots,h(f(x_m))\right)\in (\cH\circ\cF)_{|S}$. Since $C$ is a $0.5$-cover for $\cF_{|S}$, from equation \ref{eq:prop1}, we know that there exists $\hat{f_i}\in\cF$ such that $f(x_k)=\hat{f_i}(x_k)$ for any $k\in[m]$. On the other hand, for any $k\in[m]$, $\hat{f_i}(x_k)$ is an element of $S'$, consequently, there exists $\hat{h_j}\in\cH$ such that
\begin{equation*}
\begin{aligned}
     &\left\|\left( h(\hat{f_i}(x_1)),\ldots,h(\hat{f_i}(x_m))) - (\hat{h_j}(\hat{f_i}(x_1)),\ldots,\hat{h_j}(\hat{f_i}(x_m)))\right)\right\|_2^{\infty}\\
     &=\left\|\left( h(f(x_1)),\ldots,h(f(x_m)))-(\hat{h_j}(\hat{f_i}(x_1)),\ldots,\hat{h_j}(\hat{f_i}(x_m)))\right)\right\|_2^{\infty}\\
    & \leq \epsilon
\end{aligned}
\end{equation*}
From the above equation, we can conclude that $(\cH\circ\cF)_{|S}$ is $\epsilon$-covered by $\hat{\cQ}_{|S}$. Clearly, $\left|\hat{\cQ}_{|S}\right|\leq r_1r_2$ and we know that $mr_1\leq mN_1$. As a result, $N(\epsilon,\cH_{|S'},\|.\|_2^{\infty})\leq N_U(\epsilon,\cH,mr_1,\|.\|_2^{\infty})\leq N_U(\epsilon,\cH,mN_1,\|.\|_2^{\infty})$. This result holds for any input set $S\subset \cX^m$ with $|S|=m$, therefore, it follows that
\begin{equation*}
    N_U(\epsilon,\cH\circ\cF,m,\|.\|_2^{\infty})\leq N_1.N_U(\epsilon,\cH,mN_1,\|.\|_2^{\infty}).
\end{equation*}
\end{proof}
\subsection{Proof of Proposition~\ref{prop8}}
\begin{proof}
The proof for the bound of $N_U(\epsilon,\cF,m,\|.\|_2^{\ell_2})$ can be found under Theorem 3 in \citet{zhang2002covering}. Since $\cH$ is a singleton class, it is easy to verify $N_U(\epsilon,\cH,m,\|.\|_2^{\infty})=1$. We prove that the covering number of $\cH\circ\cF$ is unbounded by contradiction. Let $S=\{x_1,\ldots,x_m\}\in (0,1)^m$ be an input set where $0< x_1\leq\ldots\leq x_m$. Denote $C=\{(h\circ \hat{f_i})_{|S}=(\frac{1}{\hat{w_i}x_1},\ldots,\frac{1}{\hat{{w_i}x_m}})\mid \hat{f_i}\in\cF,i\in[r]\}$ to be an $\epsilon$-cover for $(\cH\circ\cF)_{|S}$ where $|C|=r_1$ is finite. We know that $\hat{w}_i>0$ for $i\in[r]$. Denote $w^*=\min_{i\in[r]} \hat{w_i}$. Take any $\displaystyle w< \frac{1}{\frac{1}{w^*}+x_1\epsilon} \leq \frac{1}{\frac{1}{\hat{w}_i}+x_1\epsilon}$ and denote the corresponding function by $f\in\cF$, i.e., $f(x)=wx$. we know that for every $i\in[r]$
\begin{equation*}
    \frac{1}{wx_1} > \frac{1}{\hat{w}_ix_1} + \epsilon.
\end{equation*}
This means that
\begin{equation*}
\begin{aligned}
     &\left\|\left(\frac{w}{x_1},\ldots,\frac{w}{x_m}\right) - \left(\frac{\hat{w}_i}{x_1},\ldots,\frac{\hat{w}_i}{x_m}\right) \right\|_2\\
     &=\sqrt{\frac{1}{m}\sum_{i=1}^m \left(\frac{w}{x_1}-\frac{\hat{w}_i}{x_1}\right)^2}\geq \epsilon
\end{aligned}
\end{equation*}
Therefore, there is no $(h\circ\hat{f_i})_{|S}\in C$ such that $\left \|(h\circ\hat{f_i})_{|S} - (h\circ f)_{|S}\right\|_2^{\ell_2}\leq \epsilon$, which contradicts with the assumption that $C$ is an $\epsilon$-cover for $(\cH\circ\cF)_{|S}$.
\end{proof}
\subsection{Proof of Proposition~\ref{prop9}}
\begin{proof}
Let $\cF_{\gamma,\epsilon}$ denote the class of all functions $f_{\gamma,\epsilon}$ from $\cX$ to $\bR$ such that $|f(x)-x|\leq \gamma$ for any $x\in\cX$, where $\gamma \leq \epsilon/2$. Fix an input set $S=\{x_1,\ldots,x_m\}$. We know that given any $f_{\gamma,\epsilon},f'_{\gamma,\epsilon}\in\cF_{\gamma,\epsilon}$ and $i\in[m]$, 
\begin{equation*}
\begin{aligned}
 &\|f_{\gamma,\epsilon}(x_i)-f'_{\gamma,\epsilon}(x_i)\|\leq\|f_{\gamma,\epsilon}(x_i)-x_i\|+\|x_i-f'_{\gamma,\epsilon}(x_i)\|\leq\epsilon.
\end{aligned}
\end{equation*}
Therefore, it is easy to conclude that $N_U(\epsilon,\cF_{\gamma,\epsilon},m,\|.\|_2^{\infty})=1$. Let $\cH$ to be the class of all threshold functions $h_{a}$ from $\bR$ to $[0,1]$, where $h_a(x)=\indicator\{x\geq a\}$. Consider an input set $S=\{x_1,\ldots,x_m\}$ where $x_1\leq\ldots\leq x_m$. Given any $k\in[m]$ we can find $a\in\bR$ such that $x_i < a$ for $1\leq i\leq k$ and $x_i\geq a$ for $k<i\leq m$, e.g., set $a=(x_k+x_{k+1})/2$. We also know that for any $i,j\in[m]$, $h_a(x_i)\neq h_a(x_j)$ only if $x_i<a\leq x_j$. Therefore, it is easy to verify that $\cH_{|S}=m+1$ and that for any ${h_a}_{|S}$ and ${h_{a'}}_{|S}$ in $\cH_{|S}$ we have $\left\|{h_{a'}}_{|S}-{h_a}_{|S}\right\|_2\geq 1$. We can therefore conclude that $N_U(\epsilon,\cH,m,\|.\|_2^{\infty})=m+1$. Next, consider the class $\cH\circ\cF_{\gamma,\epsilon}$. We prove that $N_U(\epsilon',\cH\circ\cF_{\gamma,\epsilon},m,\|.\|_2^{\infty})=2^m$. 

We first mention the fact that given any $(y_1,\ldots,y_m)$ and $(y'_1,\ldots,y'_m)$ in $\{0,1\}^m$ if there exists $i\in[m]$ such that $y_i\neq y'_i$, then $\left\|(y'_1,\ldots,y'_m)-(y_1,\ldots,y_m)\right\|_2\geq 1$. Also, the range of the functions in $\cH\circ\cF$ is $[0,1]$, therefore, we are only interested in $\epsilon'<1$. In the following, we prove that for any $m$ there exists a set $S'$ with $|S'|=m$ such that the restriction of $\cH\circ\cF$ to set $S'$ has $2^m$ elements and the result follows.

Consider the input set $S'=\{z_1,\ldots,z_m\}$ such that $0\leq z_1<\ldots<z_m\leq \epsilon/2$. Given any $(y_1,\ldots,y_m)\in\{0,1\}^m$ we map $(z_1,\ldots,z_m)$ to $(e_1,\ldots,e_m)$ as follows: for any $i\in[m]$ if $y_i=1$ we define $e_i=z_i+\epsilon/2$, otherwise we define $e_i=z_i-\epsilon/2$. This mapping can be done by some function $f_{\gamma,\epsilon}$ from $\cF_{\gamma,\epsilon}$ since for any $i\in[m]$ we have $|e_i-z_i|=\epsilon/2$. Let $a=\epsilon/4$. We know that $h_a(e_i)$ is 1 if $y_i=1$ and 0 otherwise. Therefore, we can conclude that for every element $(y_1,\ldots,y_m)$ in $\{0,1\}^m$, there exists $(h_a\circ f_{\gamma,\epsilon})_{|S'
}$ in $(\cH\circ\cF_{\gamma,\epsilon})_{|S'}$ such that $(\cH\circ\cF_{\gamma,\epsilon})_{|S'}=(y_1,\ldots,y_m)$. Since $|\{0,1\}^m|=2^m$ and for any two distinct elements $(y_1,\ldots,y_m)$ and $(y'_1,\ldots,y'_m)$ in $(\cH\circ\cF_{\gamma,\epsilon})_{|S'}$ we have $\left\|(y_1,\ldots,y_m)-(y'_1,\ldots,y'_m)\right\|_2\geq 1$, we can say that $N(\epsilon',(\cH\circ\cF_{\gamma,\epsilon})_{|S},\|.\|_2^{\infty,m})=2^m$. Therefore,
\begin{equation*}
    N_U(\epsilon',\cH\circ\cF_{\gamma,\epsilon},m,\|.\|_2^{\infty})=\sup_{|S|=m}\left\{N(\epsilon',(\cH\circ\cF_{\gamma,\epsilon})_{|S}),\|.\|_2^{\infty,m})\right\}\geq 2^m.
\end{equation*}
\end{proof}
\section{Proofs of theorems and lemmas in Section \ref{sec:results}}\label{app:sec:results}
\paragraph{Notation.} \label{not.} For a (random) function $f$ and an input set $S=\{x_1,\ldots,x_m\}$, we define the restriction of $f$ to $S$ as $f_{|S}=(f(x_1),\ldots,f(x_m))$. Therefore, the restriction of the class $\cF$ to $S$ can be denoted as $\cF_{|S}=\{f_{|S}:f\in\cF\}$. We also denote by $\mD(\rv{x})$ the probability density functions of the random variable $\rv{x}$. For two Borel functions $f_1$ and $f_2$, we denote by $\pi^*(f_1(\rv{x}),f_2(\rv{x}))$ a coupling between random variables $f_1(\rv{x}),f_2(\rv{x})$ such that 
\begin{equation*}
    \mM_{\pi^*}(A)=\begin{cases}
    \mM_{\rv{x}}(B) & \exists B\subset \cB(\cX) \text{ such that } A=f_1(B)\times f_2(B) \\
    0 & \text{otherwise},
    \end{cases}
\end{equation*}
where $\cB(\cX)$ is the set of all Borel sets over $\cX$, $\mM_{\pi^*}(A)$ is the measure that $\pi^*$ assigns to the Borel set $A$, and $\mM_{\rv{x}}(B)$ is the measure that random variable $\rv{x}$ assigns to Borel set $B$.

\subsection{Proof of Theorem \ref{thm:tv_to_ell2}}
\begin{proof}
It is easy to verify that $N_U(\epsilon,\rv{\cF},m,d_{TV}^{\infty},\rv{\Delta_d}) \leq N_U(\epsilon,\rv{\cF},m,d_{TV}^{\infty},\rv{\cX_d})$. Since we know that $\rv{\Delta_d}\subset\rv{\cX_d}$, we have
\begin{equation}\label{thm90}
    N_U(\epsilon,\rv{\cF},m,d_{TV}^{\infty},\rv{\Delta_d}) = \sup_{\substack{\rv{S}\subset\rv{\Delta_d} \\ |\rv{S}|=m}}\left\{ N(\epsilon,\rv{\cF}_{|\rv{S}},d_{TV}^{\infty})\right\} \leq \sup_{\substack{\rv{S}\subset\rv{\cX_d} \\ |\rv{S}|=m}}\left\{ N(\epsilon,\rv{\cF}_{|\rv{S}},d_{TV}^{\infty})\right\} = N_U(\epsilon,\rv{\cF},m,d_{TV}^{\infty},\rv{\cX_d}).
\end{equation}
Let $S=\{x_1,\ldots,x_m\}\subset\bR^d$ be an input set. Denote $\rv{S}=\{\rv{\delta_{x_1}},\ldots,\rv{\delta_{x_m}}\}\subset \rv{\Delta_d}$ and let $C=\{\rv{\hat{{f_1}}}_{|\rv{S}},\ldots,\rv{\hat{{f_r}}}_{|\rv{S}} \mid \rv{\hat{f}_r} \in \rv{\cF}, i \in [r]\}$ be an $\epsilon$-cover for $\rv{\cF}_{|\rv{S}}$ with respect to $d_{TV}^{\infty}$. Define a new set of non-random functions $\hat{\cH}=\left\{\hat{h}_i(x)=\expects{\rv{\hat{f_i}}}{~\rv{\hat{f_i}}({x})} \mid i \in [r]\right\}$.

Given any random function $\rv{f}\in\rv{\cF}$ and considering the fact that $C$ is an $\epsilon$-cover for $\rv{\cF}_{|\rv{S}}$ and that $\rv{f}_{|\rv{S}} \in \rv{\cF}_{|\rv{S}}$, we know there exists $\rv{\hat{f_i}},\,i\in[r
]$ such that 
\begin{equation}\label{eq:th91}
    d_{TV}^{\infty}\left(\rv{\hat{f_i}}_{|\rv{S}},\rv{f}_{|\rv{S}}\right)=d_{TV}^{\infty}\left((\rv{\hat{f_i}}(\rv{\delta_{x_1}}),\ldots,\rv{\hat{f_i}}(\rv{\delta_{x_m}})),(\rv{f}(\rv{\delta_{x_1}}),\ldots,\rv{f}(\rv{\delta_{x_m}}))\right)\leq \epsilon.
\end{equation}
From Equation~\ref{eq:th91} we can conclude that for any $k \in [m]$, $d_{TV}\left(\rv{\hat{f_i}}(\rv{\delta_{x_k}}),\rv{f}(\rv{\delta_{x_k}})\right)\leq \epsilon$. Further, for the corresponding $ h,\hat{h}_i \in \cH$, we know that
\begin{equation*}
    \begin{aligned}
           &\hat{h}_i(x_k) = \expects{\rv{\hat{f_i}}}{~\rv{\hat{f_i}}({\rv{\delta_{x_k}}})} = \int_{\bR^d}  x\mD(\rv{\hat{f_i}}({\rv{\delta_{x_k}}}))(x)dx,\\
            &h(x_k) = \expects{\rv{f}}{~\rv{f}({\rv{\delta_{x_k}}})}=\int_{\bR^d}  x\mD(\rv{f}({\rv{\delta_{x_k}}}))(x)dx.
    \end{aligned}
\end{equation*}
 Denote $I = \mD(\rv{f}({\rv{\delta_{x_k}}}))$ and $\hat{I} = \mD(\rv{\hat{f_i}}({\rv{\delta_{x_k}}}))$. Define two new density functions $I_{diff}$ and $\hat{I}_{diff}$ as
\begin{equation*}
    \begin{aligned}
      I_{diff}(x) &= \left\{
    \begin{array}{ll}
       \displaystyle \frac{I(x) - \hat{I}(x)}{d_{TV}(I,\hat{I})} &  I(x)\geq \hat{I}(x)\\\\
        0 & \text{otherwise,} 
    \end{array}
    \right. \\
     \hat{I}_{diff}(x) &= \left\{
    \begin{array}{ll}
       \displaystyle \frac{\hat{I}(x) - I(x)}{d_{TV}(I,\hat{I})} &  \hat{I}(x)\geq I(x)\\\\
        0 & \text{otherwise.} 
    \end{array}
    \right.
    \end{aligned}
\end{equation*}
Also, we define $I_{min}$ as 
\begin{equation*}
    I_{min}(x) = \frac{\min\{I(x),\hat{I}(x)\}}{\int\min\{I(x),\hat{I}(x)\}dx} = \frac{\min\{I(x),\hat{I}(x)\}}{1-d_{TV}(I,\hat{I})}.
\end{equation*}
It is easy to verify that 
\begin{equation*}
    \begin{aligned}
     I(x) &= \left(1-d_{TV}(I,\hat{I})\right)I_{min}(x) + d_{TV}(I,\hat{I}).I_{diff}(x)\\
     \hat{I}(x) &= \left(1-d_{TV}(I,\hat{I})\right)I_{min}(x) + d_{TV}(I,\hat{I}).\hat{I}_{diff}(x).
    \end{aligned}
\end{equation*}
We can then find the $\ell_2$ distance between $\hat{h}_i(x_k)$ and $h(x_k)$ by
\begin{equation*}
    \begin{aligned}
           &\left\|\hat{h}_i(x_k)-h(x_k)\right\|_2\\
           &=\left\|\int_{\bR^d}  x\hat{I}(x)dx - \int_{\bR^d}  xI(x)dx \right\|_2\\
           &=\left\|\int_{\bR^d}  x\left[\left(1-d_{TV}(I,\hat{I})\right)I_{min}(x) + d_{TV}(I,\hat{I}).\hat{I}_{diff}(x)\right]\right.\\ 
           & \left. - x\left[\left(1-d_{TV}(I,\hat{I})\right)I_{min}(x) + d_{TV}(I,\hat{I}).I_{diff}(x)\right]dx \right\|_2\\
           & = \left\|\int_{\bR^d}  xd_{TV}(I,\hat{I})\left[\hat{I}_{diff}(x) - I_{diff}(x)\right]dx\right\|_2\\
           & = d_{TV}(I,\hat{I})\left\|\int_{\bR^d}x\left[\hat{I}_{diff}(x) - I_{diff}(x)\right]dx\right\|_2 
           \\
           &\leq 2B\sqrt{p}\,d_{TV}\left(\rv{f}(\rv{\delta_{x_k}}),\rv{\hat{f_i}}(\rv{\delta_{x_k}})\right) && \text{(Bounded domain $[-B,B]^p$ and triangle inequality)}\\
           &\leq 2B\epsilon\sqrt{p}.
    \end{aligned}
\end{equation*}
Since this result holds for any $k \in [m]$, we have 
\begin{equation}\label{thm92}
    \left\| \hat{h_i}_{|S}-h_{|S}\right\|_2^{\infty}=\left\|(\hat{h}_i(x_1),\ldots,\hat{h}_i(x_m))- (h(x_1),\ldots,h(x_m)) \right\|_2^{\infty} \leq 2B\epsilon\sqrt{p}.  
\end{equation}
In other words, for any $h_{|S} \in \cH_{|S}$ there exists a $\hat{h_i}_{|S} \in \hat{\cH}_{|S}$ such that $\left\| \hat{h_i}_{|S}-h_{|S}\right\|_2^{\infty} \leq 2B\epsilon\sqrt{p}$. Therefore, $\hat{\cH}_{|S}$ is a $2B\epsilon\sqrt{p}$ cover for $\cH_{|S}$ with respect to $\|.\|_2^{\infty}$ and $|\hat{\cH}_{|S}|=r$. 

The bound in Equation~\ref{thm92} holds for any subset $S$ of $\bR^d$ with $|S|=m$. Therefore,
\begin{equation}\label{thm93}
     N_U(2B\epsilon\sqrt{p},\cH,m,\|.\|_2^{\infty}) \leq N_U(\epsilon,\rv{\cF},m,d_{TV}^{\infty},\rv{\Delta_d}).
\end{equation}
Putting Equations~\ref{thm90} and \ref{thm93} together, we conclude 
\begin{equation*}
      N_U(2B\epsilon\sqrt{p},\cH,m,\|.\|_2^{\infty}) \leq N_U(\epsilon,\rv{\cF},m,d_{TV}^{\infty},\rv{\Delta_d}) \leq N_U(\epsilon,\rv{\cF},m,d_{TV}^{\infty},\rv{\cX_d}).
\end{equation*}
To prove the second part that involves covering number with respect to $\|.\|_2^{\ell_2}$, we can follow the same steps. Similarly, we know that
\begin{equation*}
    N_U(\epsilon,\rv{\cF},m,d_{TV}^{\ell_2},\rv{\Delta_d}) = \sup_{\substack{\rv{S}\subset\rv{\Delta_d} \\ |\rv{S}|=m}}\left\{ N(\epsilon,\rv{\cF}_{|\rv{S}},d_{TV}^{\ell_2})\right\} \leq \sup_{\substack{\rv{S}\subset\rv{\cX_d} \\ |\rv{S}|=m}}\left\{ N(\epsilon,\rv{\cF}_{|\rv{S}},d_{TV}^{\ell_2})\right\} = N_U(\epsilon,\rv{\cF},m,d_{TV}^{\ell_2},\rv{\cX_d}).
\end{equation*}
Consider the same input sets $S$ and $\rv{S}$ and let $\tilde{C}=\{\rv{\tilde{{f_1}}}_{|\rv{S}},\ldots,\rv{\tilde{{f_r}}}_{|\rv{S}} \mid \rv{\tilde{f}_t} \in \rv{\cF}, i \in [t]\}$ be an $\epsilon$-cover for $\rv{\cF}_{|\rv{S}}$ with respect to $d_{TV}^{\ell_2}$. Define a new set of non-random functions $\tilde{\cH}=\left\{\tilde{h}_i(x)=\expects{\rv{\tilde{f_i}}}{~\rv{\tilde{f_i}}({x})} \mid i \in [r]\right\}$.

Similarly, consider $f_{|\rv{S}}$ and $\tilde{f_i}_{|\rv{S}}$ such that
\begin{equation*}
    d_{TV}^{\ell_2}\left(\rv{\tilde{f_i}}_{|\rv{S}},\rv{f}_{|\rv{S}}\right)=d_{TV}^{\ell_2}\left((\rv{\tilde{f_i}}(\rv{\delta_{x_1}}),\ldots,\rv{\tilde{f_i}}(\rv{\delta_{x_m}})),(\rv{f}(\rv{\delta_{x_1}}),\ldots,\rv{f}(\rv{\delta_{x_m}}))\right)\leq \epsilon.
\end{equation*}
Using the same analysis as before, we can conclude that for any $k \in [m]$,
\begin{equation*}
           \left\|\tilde{h}_i(x_k)-h(x_k)\right\|_2
           \leq 2B\sqrt{p}\,d_{TV}\left(\rv{f}(\rv{\delta_{x_k}}),\rv{\tilde{f_i}}(\rv{\delta_{x_k}})\right). 
\end{equation*}
We can then conclude that 
\begin{equation*}
    \begin{aligned}
     &\|\tilde{h_i}_{|S}-h_{|S}\|_2^{\ell_2}\\
    & = \sqrt{\frac{1}{m}\sum_{i=1}^k\left\|\tilde{h_i}(x_k)-h(x_k)\right\|_2^2}\\
    & \leq \sqrt{\frac{1}{m}\sum_{i=1}^k (2B\sqrt{p})^2\left(d_{TV}\left(\rv{f}(\rv{\delta_{x_k}}),\rv{\tilde{f_i}}(\rv{\delta_{x_k}})\right)\right)^2}\\
    & \leq 2B\sqrt{p} \sqrt{\frac{1}{m}\sum_{i=1}^k \left(d_{TV}\left(\rv{f}(\rv{\delta_{x_k}}),\rv{\tilde{f_i}}(\rv{\delta_{x_k}})\right)\right)^2}\\
    &\leq 2B\sqrt{p}\, d_{TV}^{\ell_2}\left(\rv{\tilde{f_i}}_{|\rv{S}},\rv{f}_{|\rv{S}}\right)\\
    & \leq 2B\epsilon\sqrt{p}.
    \end{aligned}
\end{equation*}
We can then say that $\tilde{\cH}_{|S}$ is a $2B\epsilon\sqrt{p}$ cover for $\cH_{|S}$ with respect to $\|.\|_2^{\ell_2}$ and $|\hat{\cH}_{|S}|=t$. It follows that
\begin{equation*}
      N_U(2B\epsilon\sqrt{p},\cH,m,\|.\|_2^{\ell_2}) \leq N_U(\epsilon,\rv{\cF},m,d_{TV}^{\ell_2},\rv{\Delta_d}) \leq N_U(\epsilon,\rv{\cF},m,d_{TV}^{\ell_2},\rv{\cX_d}).
\end{equation*}
\end{proof}

\subsection{Proof of Lemma \ref{lemma:compose_tv}}
\begin{proof}
Denote $\rv{\cQ}=\rv{\cH} \circ \rv{\cF}$. Consider an input set of random variables $\rv{S}=\{\rv{x_1},\ldots,\rv{x_m}\}\subset\rv{\cX_d}$. Denote $r_1=N(\epsilon,\rv{\cF}_{|\rv{S}},d_{TV}^{\infty})$ and let $\rv{C}=\{{\rv{\hat{f_1}}}_{|\rv{S}},\ldots,\rv{\hat{f}_{r_1}}_{|\rv{S}} \mid \rv{\hat{f_i}}\in\rv{\cF}, i \in [r_1]\}$ be an $\epsilon$-cover for $\rv{\cF}_{|\rv{S}}$ with respect to $d_{TV}^{\infty}$ and $\rv{S'} = \{\rv{\hat{f_i}}(\rv{x_k})\mid i\in [r_1], k \in [m]\}$. Clearly, $|\rv{S'}|\leq mr_1$. Also, let $\rv{C'}=\{{\rv{\hat{h}_1}}_{|\rv{S'}},\ldots,\rv{\hat{h}_{r_2}}_{|\rv{S'}}\mid\rv{\hat{h}_j}\in\rv{\cH},j\in [r_2]\}$ be an $\epsilon'$-cover for $\rv{\cH}_{|\rv{S'}}$ with respect to $d_{TV}^{\infty}$ metric, where $r_2=N(\epsilon',\rv{\cH}_{|\rv{S'}},d_{TV}^{\infty})$ is the cardinality of the cover set $\rv{C'}$. Denote $\rv{\hat{\cQ}}=\{\rv{\hat{h}_j}\circ {\rv{\hat{f}_i} \mid i \in [r_1], j \in [r_2]}\}$. We claim that $\rv{\hat{\cQ}}_{|\rv{S}}$ is an $(\epsilon+\epsilon')$-cover for $\rv{\cQ}_{|\rv{S}}$ with respect to $d_{TV}^{\infty}$. Since the cardinality of $\rv{\hat{\cQ}}_{|\rv{S}}$ is no more than $r_1r_2$, we can conclude that $N(\epsilon,\rv{\cQ}_{|\rv{S}},d_{TV}^{\infty}) \leq N(\epsilon,\rv{\cF}_{|\rv{S}},d_{TV}^{\infty})N(\epsilon',\rv{\cH}_{|\rv{S'}},d_{TV}^{\infty})$.

Consider ${(\rv{h}\circ \rv{f})}_{|\rv{S}}=\left(\rv{h}(\rv{f}(\rv{x_1}),\ldots,\rv{h}(\rv{f}(\rv{x_m})\right)\in \rv{\cQ}_{|\rv{S}}$, where $\rv{f}\in\rv{\cF}$ and $\rv{h}\in\rv{\cH}$. Since $\rv{\cF}_{|\rv{S}}$ is $\epsilon$-covered by $\rv{C}$, we know that there exists $\rv{\hat{f}_i}\in\rv{\cF}$ such that
\begin{equation*}
    d_{TV}^{\infty}\left((\rv{\hat{f}_i}(\rv{x_1}),\ldots,\rv{\hat{f}_i}(\rv{x_m})),(\rv{f}(\rv{x_1}),\ldots,\rv{f}(\rv{x_m}))\right) \leq \epsilon.
\end{equation*}
By data processing inequality for total variation distance (Lemma \ref{lemma:DPI}), we conclude that 
\begin{equation*}
    d_{TV}\left(\rv{h}(\rv{\hat{f}_i}(\rv{x_k})),\rv{h}(\rv{f}(\rv{x_k}))\right)\leq \epsilon
\end{equation*} 
for $k\in[m]$. Therefore,
\begin{equation}\label{eq:A1}
    d_{TV}^{\infty}\left((\rv{h}(\rv{\hat{f}_i}(\rv{x_1})),\ldots,\rv{h}(\rv{\hat{f}_i}(\rv{x_m}))),(\rv{h}(\rv{f}(\rv{x_1})),\ldots,\rv{h}(\rv{f}(\rv{x_m})))\right)\leq \epsilon.
\end{equation}
Since $\rv{\hat{f}_i}_{|\rv{S}}=(\rv{\hat{f}_i}(\rv{x_1}),\ldots,\rv{\hat{f}_i}(\rv{x_m}))\in \rv{C}$, we know that $\rv{\hat{f}_i}(\rv{x_k}) \in \rv{S'}$ for $k\in[m]$. We also know that $\rv{\cH}_{|\rv{S'}}$ is $\epsilon'$-covered by $\rv{C'}$, therefore, there exists ${\rv{\hat{h}_j}}\in\rv{\cH}$ such that 
\begin{equation}\label{eq:A2}
    d_{TV}^{\infty}\left( (\rv{\hat{h}_j}(\rv{\hat{f}_i}(\rv{x_1})),\ldots,\rv{\hat{h}_j}(\rv{\hat{f}_i}(\rv{x_m}))),(\rv{h}(\rv{\hat{f}_i}(\rv{x_1})),\ldots,\rv{h}(\rv{\hat{f}_i}(\rv{x_m})))  \right)\leq \epsilon'
\end{equation}

Combining Equations~\ref{eq:A1} and $\ref{eq:A2}$ and by using triangle inequality for total variation distance, we conclude that
\begin{equation*}
       d_{TV}^{\infty}\left( (\rv{\hat{h}_j}(\rv{\hat{f}}_i(\rv{x_1})),\ldots,\rv{\hat{h}_j}(\rv{\hat{f}_i}(\rv{x_m}))),\left(\rv{h}(\rv{f}(\rv{x_1})),\ldots,\rv{h}(\rv{f}(\rv{x_m}))\right)  \right)\leq \epsilon+\epsilon',
\end{equation*}
which suggests that for any $(\rv{h}\circ \rv{f})_{|\rv{S}}\in \rv{\cQ}_{|\rv{S}}$, there exists $(\rv{\hat{h}_j}\circ \rv{\hat{f}_i})_{|\rv{S}}\in\rv{\hat{\cQ}}_{|\rv{S}}$ such that 
\begin{equation*}
    d_{TV}^{\infty}\left({(\rv{h}\circ \rv{f})_{|\rv{S}},(\rv{\hat{h}_j}\circ \rv{\hat{f}_i})_{|\rv{S}}}\right)\leq \epsilon+\epsilon'.
\end{equation*}
In other words, $\rv{\cQ}_{|\rv{S}}$ is $(\epsilon+\epsilon')$-covered by $\rv{\hat{\cQ}}_{|\rv{S}}$. 

Let $N_1=N_U(\epsilon,\rv{\cF},m,d_{TV}^{\infty},\rv{\cX_d})$. We know that $mr_1\leq mN_1$ and, therefore,  $N(\epsilon',\rv{\cH}_{|\rv{S'}},d_{TV}^{\infty})\leq N_U(\epsilon',\rv{\cH},mr_1,d_{TV}^{\infty},\rv{\cX_d})\leq N_U(\epsilon',\rv{\cH},mN_1,d_{TV}^{\infty},\rv{\cX_d})$. Since the result holds for any input $\rv{S}\subset\rv{\cX_d}$ of cardinality $m$ and we know that $ r_1\leq N_U(\epsilon,\rv{\cF},m,d_{TV}^{\infty},\rv{\cX_d})$, it follows that
\begin{equation*}
    N_U(\epsilon+\epsilon',\rv{\cQ},m,d_{TV}^{\infty},\rv{\cX_d}) \leq N_U(\epsilon',\rv{\cH},mN_1,d_{TV}^{\infty},\rv{\cX_d}).N_U(\epsilon,\rv{\cF},m,d_{TV}^{\infty},\rv{\cX_d}).
\end{equation*}

The bound for $\rv{\Delta_d}$ is almost exactly the same as that of $\rv{\cX_d}$. The only difference is that 
 $\rv{S}=\{\rv{{\delta}_{x_1}},\ldots,\rv{{\delta}_{x_m}}\}\subseteq\rv{\Delta_d}$, and we have a uniform $\epsilon$-covering number with respect to $\rv{\Delta_d}$. We conclude that
\begin{equation*}
      N_U\left(\epsilon+\epsilon',\rv{\cH}\circ\rv{\cF},m,d_{TV}^{\infty},\rv{\Delta_d}\right) \leq  N_U\left(\epsilon',\rv{\cH},mN_2,d_{TV}^{\infty},\rv{\cX_d}\right).N_U(\epsilon,\rv{\cF},m,d_{TV}^{\infty},\rv{\Delta_d}).
\end{equation*}

The bound with respect to $d_{TV}^{\ell_2}$ follows the same analysis. Consider a new set $\rv{S_z}=\{\rv{\delta{z_1}},\ldots,\rv{\delta{z_m}}\}\subset\rv{\Delta_d}$. Denote $t_1=N(\epsilon,\rv{\cF}_{|\rv{S_z}},d_{TV}^{\ell_2})$ and let $\rv{C_z}=\{{\rv{\tilde{f}_1}}_{|\rv{S_z}},\ldots,\rv{\tilde{{f}_{t_1}}}_{|\rv{S_z}}\mid\rv{\tilde{f}_i}\in\rv{\cF}, i \in [t_1]\}$ be an $\epsilon$-cover for $\rv{\cF}_{|\rv{S_z}}$ with respect to $d_{TV}^{\ell_2}$ and $\rv{S_z'} = \{\rv{\tilde{f}_i}(\rv{\delta{z_k}})\mid i\in [t_1], k \in [m]\}$. Clearly, $|\rv{S_z'}|\leq mt_1$. Let $\rv{C_z'}=\{{\rv{\tilde{h}_1}}_{|\rv{S_z'}},\ldots,\rv{\tilde{h}_{t_2}}_{|\rv{S_z'}}\mid\rv{\tilde{h}_j}\in\rv{\cH},j\in [t_2]\}$ be an $\epsilon'$-cover for $\rv{\cH}_{|\rv{S_z'}}$ with respect to $d_{TV}^{\infty}$ metric, where $t_2=N(\epsilon',\rv{\cH}_{|\rv{S_z'}},d_{TV}^{\infty})$ is the cardinality of the cover set $\rv{C_z'}$. Denote $\rv{\tilde{\cQ}}=\{\rv{{\tilde{h}_j}}\circ {\rv{\tilde{f}_i} \mid i \in [t_1], j \in [t_2]}\}$. We claim that $\rv{\tilde{\cQ}}_{|\rv{S_z}}$ is an $(\epsilon+\epsilon')$-cover for $\rv{\cQ}_{|\rv{S_z}}$ with respect to $d_{TV}^{\ell_2}$. We can then conclude that $N(\epsilon,\rv{\cQ}_{|\rv{S_z}},d_{TV}^{\ell_2}) \leq N(\epsilon,\rv{\cF}_{|\rv{S_z}},d_{TV}^{\ell_2}).N(\epsilon',\rv{\cH}_{|\rv{S_z'}},d_{TV}^{\infty})$.

Consider ${(\rv{h}\circ \rv{f})}_{|\rv{S_z}}=\left(\rv{h}(\rv{f}(\rv{\delta_{z_1}}),\ldots,\rv{h}(\rv{f}(\rv{\delta_{z_m}})\right)\in \rv{\cQ}_{|\rv{S_z}}$, where $\rv{f}\in\rv{\cF}$ and $\rv{h}\in\rv{\cH}$. Since $\rv{\cF}_{|\rv{S_z}}$ is $\epsilon$-covered by $\rv{C_z}$, we know that there exists $\rv{\tilde{f}_i}\in\rv{\cF}$ such that
\begin{equation*}
\begin{aligned}
     &d_{TV}^{\ell_2}\left((\rv{\tilde{f}_i}(\rv{\delta_{z_1}}),\ldots,\rv{\tilde{f}_i}(\rv{\delta_{z_m}})),(\rv{f}(\rv{\delta_{z_1}}),\ldots,\rv{f}(\rv{\delta_{z_m}}))\right) \\
    &  = \sqrt{\frac{1}{m}\sum_{k=1}^m \left(d_{TV}(\rv{\tilde{f}_i}(\rv{\delta_{z_k}}),\rv{f}(\rv{\delta_{z_k}}))\right)^2}\leq \epsilon.
\end{aligned}
\end{equation*}
Similarly, by data processing inequality, we conclude that $d_{TV}\left(\rv{h}(\rv{\tilde{f}_i}(\rv{\delta_{z_k}})),\rv{h}(\rv{f}(\rv{\delta_{z_k}}))\right)\leq d_{TV}\left(\rv{\tilde{f}_i}(\rv{\delta_{z_k}}),\rv{f}(\rv{\delta_{z_k}})\right)$ for $k\in[m]$. Therefore,
\begin{equation}\label{eq:B3}
\begin{aligned}
     &d_{TV}^{\ell_2}\left((\rv{h}(\rv{\tilde{f}_i}(\rv{\delta_{z_1}})),\ldots,\rv{h}(\rv{\tilde{f}_i}(\rv{\delta_{z_m}}))),(\rv{h}\rv{f}(\rv{\delta_{z_1}})),\ldots,\rv{h}(\rv{f}(\rv{\delta_{z_m}})))\right)\\
      &  = \sqrt{\frac{1}{m}\sum_{k=1}^m \left(d_{TV}(\rv{h}(\rv{\tilde{f}_i}(\rv{\delta_{z_k}})),\rv{h}(\rv{f}(\rv{\delta_{z_k}})))\right)^2}\\
       &  \leq \sqrt{\frac{1}{m}\sum_{k=1}^m \left(d_{TV}(\rv{\tilde{f}_i}(\rv{\delta_{z_k}}),\rv{f}(\rv{\delta_{z_k}}))\right)^2}\leq \epsilon.
\end{aligned}
\end{equation}
Now, using the fact that $\rv{\tilde{f}_i}_{|\rv{S_z}}=(\rv{\tilde{f}_i}(\rv{\delta_{z_1}}),\ldots,\rv{\tilde{f}_i}(\rv{\delta_{z_m}}))\in \rv{C_z}$, we know that $\rv{\tilde{f}_i}(\rv{\delta_{z_k}}) \in \rv{S_z'}$ for $k\in[m]$. We also know that $\rv{\cH}_{|\rv{S_z'}}$ is $\epsilon'$-covered by $\rv{C_z'}$ with respect to $d_{TV}^{\infty}$. Therefore, there exists $\rv{\tilde{h}_j}\in\rv{\cH}$ such that 
\begin{equation}\label{eq:B4}
    d_{TV}^{\infty}\left( (\rv{\tilde{h}_j}(\rv{\tilde{f}_i}(\rv{\delta_{z_1}})),\ldots,\rv{\tilde{h}_j}(\rv{\tilde{f}_i}(\rv{\delta_{z_m}}))),(\rv{h}(\rv{\tilde{f}_i}(\rv{\delta_{z_1}})),\ldots,\rv{h}(\rv{\tilde{f}_i}(\rv{\delta_{z_m}})))  \right)\leq \epsilon'.
\end{equation}
From Equation~\ref{eq:B4} we can conclude that $d_{TV}\left((\rv{\tilde{h}_j}(\rv{\tilde{f}_i}(\rv{\delta_{z_k}})) , (\rv{h}(\rv{\tilde{f}_i}(\rv{\delta_{z_k}}))\right)\leq \epsilon'$ for $k \in [m]$. Using triangle inequality for total variation distance, we can write
\begin{equation}\label{eq:B5}
\begin{aligned}
     &d_{TV}\left((\rv{\tilde{h}_j}(\rv{\tilde{f}_i}(\rv{\delta_{z_k}})) , (\rv{h}(\rv{f}(\rv{\delta_{z_k}}))\right)\\
     &\leq  d_{TV}\left((\rv{\tilde{h}_j}(\rv{\tilde{f}_i}(\rv{\delta_{z_k}})) , (\rv{h}(\rv{\tilde{f}_i}(\rv{\delta_{z_k}}))\right) + d_{TV}\left(\rv{h}(\rv{\tilde{f}_i}(\rv{\delta_{z_k}})),\rv{h}(\rv{f}(\rv{\delta_{z_k}}))\right) \\
     &\leq d_{TV}\left(\rv{h}(\rv{\tilde{f}_i}(\rv{\delta_{z_k}})),\rv{h}(\rv{f}(\rv{\delta_{z_k}}))\right) + \epsilon'.
\end{aligned}
\end{equation}
We can then conclude that
\begin{equation*}
    \begin{aligned}
     &d_{TV}^{\ell_2}\left((\rv{\tilde{h}_j}(\rv{\tilde{f}_i}(\rv{\delta_{z_1}})),\ldots,\rv{\tilde{h}_j}(\rv{\tilde{f}_i}(\rv{\delta_{z_m}}))),(\rv{h}(\rv{f}(\rv{\delta_{z_1}})),\ldots,\rv{h}(\rv{f}(\rv{\delta_{z_m}})))\right)\\
      &  = \sqrt{\frac{1}{m}\sum_{k=1}^m \left(d_{TV}(\rv{\tilde{h}_j}(\rv{\tilde{f}_i}(\rv{\delta_{z_k}})),\rv{h}(\rv{f}(\rv{\delta_{z_k}})))\right)^2}\\
       &  \leq \sqrt{\frac{1}{m}\sum_{k=1}^m \left(d_{TV}(\rv{h}(\rv{\tilde{f}_i}(\rv{\delta_{z_k}})),\rv{h}(\rv{f}(\rv{\delta_{z_k}})))+\epsilon'\right)^2} && \text{(From Equation~\ref{eq:B5})}\\
        &  \leq \sqrt{\frac{1}{m}\sum_{k=1}^m \left(d_{TV}(\rv{h}(\rv{\tilde{f}_i}(\rv{\delta_{z_k}})),\rv{h}(\rv{f}(\rv{\delta_{z_k}})))\right)^2+\frac{1}{m}\sum_{k=1}^m\epsilon'^2} \\
        &  \leq \sqrt{\frac{1}{m}\sum_{k=1}^m \left(d_{TV}(\rv{h}(\rv{\tilde{f}_i}(\rv{\delta_{z_k}})),\rv{h}(\rv{f}(\rv{\delta_{z_k}})))\right)^2}+\sqrt{\frac{1}{m}\sum_{k=1}^m\epsilon'^2}\\
        & \leq \epsilon + \epsilon'. && \text{(From Equation~\ref{eq:B3})}
    \end{aligned}
\end{equation*}
As a result, $\rv{\cQ}_{|\rv{S_z}}$ is $(\epsilon+\epsilon')$-covered by $\rv{\tilde{\cQ}}_{|\rv{S_z}}$. Let $N_3=N_U(\epsilon,\rv{\cF},m,d_{TV}^{\ell_2},\rv{\Delta_d})$. Since $mt_1\leq mN_3$, we can write $N(\epsilon',\rv{\cH}_{|\rv{S'}},d_{TV}^{\infty})\leq N_U(\epsilon',\rv{\cH},mt_1,d_{TV}^{\infty},\rv{\cX_d})\leq N_U(\epsilon',\rv{\cH},mN_3,d_{TV}^{\infty},\rv{\cX_d})$. We know that the result holds for any input $\rv{S_z}\subset\rv{\Delta_d}$ of cardinality $m$ and $ t_1\leq N_U(\epsilon,\rv{\cF},m,d_{TV}^{\ell_2},\rv{\Delta_d})$, therefore, it follows that
\begin{equation*}
    N_U(\epsilon+\epsilon',\rv{\cQ},m,d_{TV}^{\ell_2},\rv{\Delta_d}) \leq N_U(\epsilon',\rv{\cH},mN_3,d_{TV}^{\infty},\rv{\cX_d}).N_U(\epsilon,\rv{\cF},m,d_{TV}^{\ell_2},\rv{\Delta_d}).
\end{equation*}

\end{proof}
\subsection{TV distance of composition of a class with noise}
The following lemma, which is used in bounding the total variation distance by Wasserstein distance, is borrowed from \citet{chae2020wasserstein}. This lemma will be used in proof of the remaining lemmas in Section~\ref{sec:results}.
\begin{lemma}[Bounding TV distance by Wasserstein distance]\label{lemma:wasser_conv}
Given a density function $K$ over $\bR^d$ and two probability measures $\mu,\nu$ over $\cX$ with probability density functions $I_{\mu}$ and $I_{\nu}$, respectively, we have
\begin{equation*}
    \|K*I_{\mu}-K*I_{\nu}\|_1\leq\sup_{y\neq z}\left\{\frac{\|K(x-y)-K(x-z)\|_1}{\|y-z\|_2
    }\right\}d_{\cW}(\mu,\nu)
\end{equation*}
\end{lemma}
\begin{proof}
For any coupling $\pi$ of $\mu$ and $\nu$, we have
\begin{equation*}
    K*I_{\mu}(x)-K*I_{\nu}(x)=\int(K(x-y)-K(x-z))d\pi(y,z).
\end{equation*}
Therefore, 
\begin{equation*}
\begin{aligned}
       \|K*(I_{\mu}-I_{\nu})\|_1&=\int\left|\int\left((K(x-y)-K(x-z)\right)d\pi(y,z)\right|dx\\
       &\leq \int\int\left|(K(x-y)-K(x-z)\right|d\pi(y,z)dx &&\text{(By Jensen's inequality)}\\
       &= \int\left\|K(x-y)-K(x-z)\right\|_1d\pi(y,z)&&\text{(By Fubini's theorem)}\\
       &\leq \sup_{y\neq z}\left\{\frac{\left\|K(x-y)-K(x-z)\right\|_1}{\|y-z\|_2}\right\}\int\|y-z\|_2d\pi(y,z)
\end{aligned}
\end{equation*}
Since this holds for any coupling $\pi$ of $\mu$ and $\nu$ we conclude that
\begin{equation*}
\|K*(I_{\mu}-I_{\nu})\|_1\leq \sup_{y\neq z}\left\{\frac{\left\|K(x-y)-K(x-z)\right\|_1}{\|y-z\|_2}\right\}d_{\cW}(\mu,\nu)
\end{equation*}
\end{proof}
\subsection{Proof of Theorem \ref{thm:wasser_tv}}
\begin{proof}
Fix an input set $\rv{S}=\{\rv{x_1},\ldots,\rv{x_m}\}\subset\bR^d$. Let $\rv{C}=\{{\rv{\hat{f}_1}}_{|\rv{S}},\ldots,\rv{\hat{f}_r}_{|\rv{S}}:\rv{\hat{f}_i}\in\rv{\cF}, i \in [r]\}$ be an $\epsilon$-cover for $\rv{\cF}_{|\rv{S}}$ with respect to $d_{\cW}^{\infty}$ metric. Denote $\rv{\cQ}=\rv{\cG_{\sigma}} \circ \rv{\cF}$. We define a new class of random functions $\rv{\hat{\cQ}}=\{\rv{g_{\sigma}} \circ \rv{\hat{f}_i} \mid i \in [r]\}$. We show that $\rv{\cQ}_{|\rv{S}}$ is $(\frac{\epsilon}{2\sigma})$-covered by $\rv{\hat{\cQ}}_{|\rv{S}}$ and since $|\rv{\hat{Q}}_{|\rv{S}}| = r$, the result follows.

Let $I_{\sigma}$ denote the probability density function of $\cN(\mathbf{0},\sigma^2 I_d)$. For any $\rv{f}\in\rv{\cF}$, we have $\rv{g_{\sigma}}(\rv{f}(x))=\rv{f}(x)+\rv{z}$, where $\rv{z}$ is a random variable with probability density function $I_{\sigma}$, therefore, we know that $\mD(\rv{g_{\sigma}}(\rv{f}(x))=\mD(\rv{f}(x))*I_{\sigma}$.

Given $(\rv{g_{\sigma}} \circ \rv{f})_{|\rv{S}}=(\rv{g_{\sigma}}(\rv{f}(\rv{x_1})),\ldots,\rv{g_{\sigma}}(\rv{f}(\rv{x_m})))\in\rv{\cQ}_{|\rv{S}}$, we know that $\rv{f}_{|\rv{S}}=(\rv{f}(\rv{x_1}),\ldots,\rv{f}(\rv{x_m}))$ is in $\rv{\cF}_{|\rv{S}}$. Therefore, there exists $\rv{\hat{f}_i}\in\rv{\cF}$ such that $d_{\cW}^{\infty}(\rv{\hat{f}_i}_{|\rv{S}},\rv{f}_{|\rv{S}})\leq \epsilon$, i.e.,
\begin{equation}\label{eq:A3}
    d_{\cW}^{\infty}\left( (\rv{\hat{f}_i}(\rv{x_1}),\ldots,\rv{\hat{f}_i}(\rv{x_m})),(\rv{f}(\rv{x_1}),\ldots,\rv{f}\rv{x_m}))\right) \leq \epsilon.
\end{equation}
From Equation~\ref{eq:A3}, we know that $d_\cW(\rv{\hat{f}_i}(\rv{x_k}),\rv{f}(\rv{x_k}))\leq\epsilon$ for all $k \in [m]$. From Lemma~\ref{lemma:wasser_conv}, we can conclude that for all $k\in[m]$,
\begin{equation}\label{eq:A4}
\begin{aligned}
        &\frac{1}{2}\left\|I_{\sigma}*\mD(\rv{\hat{f}_i}(\rv{x_k})-I_{\sigma}*\mD(\rv{f}(\rv{x_k})))\right\|_1\\ & \leq \frac{1}{2}\left(\sup_{y\neq z}\left\{\frac{\left\|I_{\sigma}(x-y)-I_{\sigma}(x-z)\right\|_1}{\|y-z\|_2}\right\}\right)d_\cW(\rv{\hat{f}_i}(\rv{x_k}),\rv{f}\rv{x_k}))\\
        &\leq \frac{\epsilon}{2} \left(\sup_{y\neq z}\left\{\frac{\left\|I_{\sigma}(x-y)-I_{\sigma}(x-z)\right\|_1}{\|y-z\|_2}\right\}\right).
\end{aligned}
\end{equation}
Moreover, $I_{\sigma}*\mD(\rv{\hat{f}_i}(\rv{x_k})$ and $I_{\sigma}*\mD(\rv{f}(\rv{x_k}))$ are probability density functions of $\rv{g_{\sigma}}(\rv{\hat{f}_i}(\rv{x_k}))$ and $\rv{g_{\sigma}}(\rv{f}(\rv{x_k}))$, respectively. Therefore, from Equation~\ref{eq:A4},
\begin{equation}\label{eq:A5}
    d_{TV}\left(\rv{g_{\sigma}}(\rv{\hat{f}_i}(\rv{x_k})),\rv{g_{\sigma}}(\rv{f}(\rv{x_k}))\right)\leq \frac{\epsilon}{2} \left(\sup_{y\neq z}\left\{\frac{\left\|I_{\sigma}(x-y)-I_{\sigma}(x-z)\right\|_1}{\|y-z\|_2}\right\}\right).
\end{equation}
Since Equation~\ref{eq:A5} holds for all $k \in [m]$, it follows that
\begin{equation*}
     d_{TV}^{\infty}\left(\rv{g_{\sigma}}(\rv{\hat{f}_i}(\rv{x_k})),\rv{g_{\sigma}}(\rv{f}(\rv{x_k}))\right)\leq \frac{\epsilon}{2} \left(\sup_{y\neq z}\left\{\frac{\left\|I_{\sigma}(x-y)-I_{\sigma}(x-z)\right\|_1}{\|y-z\|_2}\right\}\right).
\end{equation*}
This shows that for any $(\rv{g_{\sigma}} \circ \rv{f})_{|\rv{S}} \in \rv{\cQ}_{|\rv{S}}$ there exists $(\rv{g_{\sigma}} \circ \rv{\hat{f}_i})_{|\rv{S}} \in \rv{\hat{\cQ}}_{|\rv{S}}$ such that 
\begin{equation}\label{eq:A6}
    d_{TV}^{\infty}\left((\rv{g_{\sigma}} \circ \rv{f})_{|\rv{S}},(\rv{g_{\sigma}} \circ \rv{\hat{f}_i})_{|\rv{S}}\right) \leq \frac{\epsilon}{2} \left(\sup_{y\neq z}\left\{\frac{\left\|I_{\sigma}(x-y)-I_{\sigma}(x-z)\right\|_1}{\|y-z\|_2}\right\}\right).
\end{equation}
It is only left to bound the supremum term in Equation~\ref{eq:A6}.

Based on Theorem~\ref{thm:tv_gaussian}, we know that for two Gaussian distributions $\cN(\mu_1,\sigma^2I)$ and $\cN(\mu_2,\sigma^2I)$ their total variation distance can be bounded by
\begin{equation}\label{eq:A7}
    d_{TV}\left(\cN(\mu_1,\sigma^2I),\cN(\mu_2,\sigma^2I)\right) \leq \frac{1}{2\sigma}\|\mu_1-\mu_2\|_2.
\end{equation}
We also know that $\left\|I_{\sigma}(x-y)-I_{\sigma}(x-z)\right\|_1=2d_{TV}(\cN(y,\sigma^2I),\cN(z,\sigma^2I))$. Combining Equations~\ref{eq:A6} and $\ref{eq:A7}$, we can write
\begin{equation}\label{eq:A8}
    \begin{aligned}
            d_{TV}^{\infty}\left((\rv{g_{\sigma}} \circ \rv{f})_{|\rv{S}},(\rv{g_{\sigma}} \circ \rv{\hat{f}_i})_{|\rv{S}}\right) \leq  \frac{\epsilon}{2} \left(\sup_{y\neq z}\left\{ \frac{\frac{1}{\sigma}\|y-z\|_2}{\|y-z\|_2 }\right\}\right)\leq \frac{\epsilon}{2\sigma}.
    \end{aligned}
\end{equation}
From Equation~\ref{eq:A8} it follows that $\rv{\cQ}_{|\rv{S}}$ is $(\frac{\epsilon}{2\sigma})$-covered by $\rv{\hat{\cQ}}_{|\rv{S}}$. Since the result holds for any subset $\rv{S}$ of $\rv{\cX_d}$ with cardinality $m$, we can conclude that
\begin{equation*}
    N_U\left(\frac{\epsilon}{2\sigma},\rv{\cG_{\sigma}}\circ\rv{\cF},m,d_{TV}^{\infty},\rv{\cX_d}\right) \leq  N_U(\epsilon,\cF,m,d_\cW^{\infty},\rv{\cX_d}).
\end{equation*}

The second part of the proof is similar. We consider a set of inputs $\rv{S_z} = \{\rv{\delta_{z_1}}, \ldots, \rv{\delta_{z_m}}\}\subset \rv{\Delta_d}$. We can then consider an $\epsilon$-cover $\rv{C_z} = \{\rv{\tilde{f}_1}_{|\rv{S}},\ldots,\rv{\tilde{f}_t}_{|\rv{S}}:\rv{\tilde{f}_i}\in\rv{\cF}, i \in [t]\}$ for $\rv{\cF}_{|\rv{S_z}}$. We will then construct a class of functions $\rv{\tilde{\cQ}} = \{\rv{g_{\sigma}}\circ\rv{\tilde{f}_i} \mid i \in [t]\}$ and show that $\rv{\cQ}_{|\rv{S_z}}$ is $(\frac{\epsilon}{2\sigma})$-covered by $\rv{\tilde{\cQ}}_{|\rv{S_z}}$. The proof follows the same steps as the previous part. Particularly, let $\rv{\tilde{f}_i} \in \rv{\cF}$ be such that $d_{\cW}^{\infty}(\rv{\tilde{f}_i}_{|\rv{S_z}}, \rv{f}_{|\rv{S_z}})\leq \epsilon$. For any $k \in [m]$, we can write that
\begin{equation}
\begin{aligned}
        &\frac{1}{2}\left\|I_{\sigma}*\mD(\rv{\tilde{f}_i}(\rv{\delta_{z_k}})-I_{\sigma}*\mD(\rv{f}(\rv{\delta_{z_k}})))\right\|_1\\ & \leq \frac{1}{2}\left(\sup_{y\neq z}\left\{\frac{\left\|I_{\sigma}(x-y)-I_{\sigma}(x-z)\right\|_1}{\|y-z\|_2}\right\}\right)d_\cW(\rv{\tilde{f}_i}(\rv{\delta_{z_k}}),\rv{f}(\rv{\delta_{z_k}}))\\
        &\leq \frac{\epsilon}{2} \left(\sup_{y\neq z}\left\{\frac{\left\|I_{\sigma}(x-y)-I_{\sigma}(x-z)\right\|_1}{\|y-z\|_2}\right\}\right).
\end{aligned}
\end{equation}
Using the same arguments as the previous part, we will have that
\begin{equation*}
     d_{TV}\left(\rv{g_{\sigma}}(\rv{\tilde{f}_i}(\rv{\delta_{z_k}})),\rv{g_{\sigma}}(\rv{f}(\rv{\delta_{z_k}}))\right)\leq \frac{\epsilon}{2} \left(\sup_{y\neq z}\left\{\frac{\left\|I_{\sigma}(x-y)-I_{\sigma}(x-z)\right\|_1}{\|y-z\|_2}\right\}\right)\leq \frac{\epsilon}{2\sigma}.
\end{equation*}
Therefore, we can conclude that for any $\rv{f}\in\rv{\cF}$ there exists $\rv{\tilde{f}_i},\, i\in[t]$ such that
\begin{equation*}
                d_{TV}^{\infty}\left((\rv{g_{\sigma}} \circ \rv{f})_{|\rv{S_z}},(\rv{g_{\sigma}} \circ \rv{\tilde{f}_i})_{|\rv{S_z}}\right) \leq \frac{\epsilon}{2\sigma},
\end{equation*}
which means that $\rv{\cQ}_{\rv{S_z}}$ is $(\frac{\epsilon}{2\sigma})$-covered by $\rv{\tilde{\cQ}}_{|\rv{S_z}}$. Since the result holds for every $\rv{S_z} \subset \rv{\Delta_d}$ of cardinality m, we can conclude that 
\begin{equation*}
    N_U\left(\frac{\epsilon}{2\sigma},\rv{\cG_{\sigma}}\circ\rv{\cF},m,d_{TV}^{\infty},\rv{\Delta_d}\right) \leq  N_U(\epsilon,\rv{\cF},m,d_\cW^{\infty},\rv{\Delta_d}).
\end{equation*}
\end{proof}
\subsection{Proof of Corollary \ref{coll:ell2_to_tv}}
\begin{proof}
First, from Proposition~\ref{prop12}, we can conclude that
\begin{equation}\label{eq:col1}
    N_U(\epsilon,\cF,d_{\cW}^{\infty},m,\rv{\Delta_d})=N_U(\epsilon,\cF,m,\|.\|_2^{\infty}).
\end{equation}

Then, consider an input set $\rv{S_z}=\{\rv{\delta_{x_1}},\ldots,\rv{\delta_{x_m}}\}\subset \rv{\Delta_d}$. Let $\rv{C_z}=\{\hat{{f_1}}_{|\rv{S_z}},\ldots,\hat{{f_{r}}}_{|\rv{S_z}}\mid \hat{f_i} \in \cF, i \in [r]\}$ be an $\epsilon$-cover for $\cF_{|\rv{S_z}}$ with respect to $d_{\cW}^{\infty}$, then for a given $f_{|\rv{S_z}}\in\cF_{|\rv{S_z}}$ and $\hat{f_i}_{|\rv{S_z}}\in \rv{C_z}$, where $d_{\cW}^{\infty}(f_{|\rv{S_z}},\hat{f_i}_{|\rv{S_z}})\leq \epsilon$, from Equations~\ref{eq:A4} and \ref{eq:A8}, we know that for all $k\in[m]$
\begin{equation*}
\begin{aligned}
            &d_{TV}(\rv{g_{\sigma}}(\hat{f}_i(\rv{\delta_{x_k}})),\rv{g_{\sigma}}(f(\rv{\delta_{x_k}})))=d_{TV}\left(\cN(\hat{f}_i(x_k),\sigma^2I_p),\cN(f(x_k),\sigma^2I_p)\right)\\
            &\leq \frac{1}{2\sigma}\|\hat{f}_i(x_k)-f(x_k)\|_2 \leq \frac{1}{2\sigma}d_{\cW}\left(\hat{f_i}(\rv{\delta_{x_k}}),f(\rv{\delta_{x_k}})\right)\\
            &\leq \frac{\epsilon}{2\sigma}.
\end{aligned}
\end{equation*}
Therefore, we can conclude that
\begin{equation*}
\begin{aligned}
       &d_{TV}^{\infty}\left((\rv{g_{\sigma}} \circ \hat{f}_i)_{|\rv{S_z}},(\rv{g_{\sigma}} \circ f)_{|\rv{S_z}}\right)\\
       &=d_{TV}^{\infty}\left( (\rv{g_{\sigma}}(\hat{f}_i(\rv{\delta_{x_1}})),\ldots,\rv{g_{\sigma}}(\hat{f}_i(\rv{\delta_{x_m}}))),(\rv{g_{\sigma}}(f(\rv{\delta_{x_1}})),\ldots,\rv{g_{\sigma}}(f(\rv{\delta_{x_m}})))\right)\\
       &\leq \frac{\epsilon}{2\sigma},
\end{aligned}
\end{equation*}
It follows that for any $(\rv{g_{\sigma}}\circ f)_{|\rv{S_z}} \in (\rv{\cG_{\sigma}}\circ\cF)_{|\rv{S_z}}$, there exists $\hat{f_i}_{|\rv{S_z}}\in \rv{C_z}$ such that $\displaystyle d_{TV}^{\infty}\left((\rv{g_{\sigma}} \circ \hat{f}_i)_{|\rv{S_z}},(\rv{g_{\sigma}} \circ f)_{|\rv{S_z}}\right)\leq  \frac{\epsilon}{2\sigma}$. Therefore,
\begin{equation*}
    N(\frac{\epsilon}{2\sigma},(\rv{\cG_{\sigma}}\circ\cF)_{|\rv{S_z}},d_{TV}^{\infty}) \leq N(\epsilon, \cF_{|\rv{S_z}}, d_{\cW}^{\infty}).
\end{equation*}
Since this results holds for any $\rv{S_z}\subset \rv{\Delta_d}$, we can conclude that
\begin{equation*}
    N_U(\frac{\epsilon}{2\sigma},\rv{\cG_{\sigma}}\circ\cF,m,d_{TV}^{\infty},\rv{\Delta_d}) \leq N_U(\epsilon, \cF,m,d_{\cW}^{\infty},\rv{\Delta_d}) = N_{U}(\epsilon,\cF,m,\|.\|_2^{\infty}).
\end{equation*}

The proof of the second part again follows from Proposition~\ref{prop12}. We can write that 
\begin{equation*}
    N_{U}(\epsilon,{\cF},d_\cW^{\ell_2}, m, \rv{\Delta_d}) = N_{U}(\epsilon,\cF,m,\|.\|_2^{\ell_2}).
\end{equation*}

Consider the input set $\rv{S_z}\subset \rv{\Delta_{d}}$ as defined above and let $\rv{\tilde{C}_z}=\{\tilde{{f_1}}_{|\rv{S_z}},\ldots,\tilde{{f_{t}}}_{|\rv{S_z}}\mid \tilde{f_i} \in \cF, i \in [t]\}$ be an $\epsilon$-cover for $\cF_{|\rv{S_z}}$ with respect to $d_{\cW}^{\ell_2}$. Now, for a given $f_{|\rv{S_z}} \in \cF_{|\rv{S_z}}$ and the corresponding $\tilde{f_i}_{|\rv{S_z}} \in \rv{\tilde{C}_z}$, where $d_{\cW}^{\ell_2}(f_{|\rv{S_z}},\tilde{f_i}_{|\rv{S_z}})\leq\epsilon$, we know that for all $k \in [m]$
\begin{equation*}
\begin{aligned}
            &d_{TV}(\rv{g_{\sigma}}(\tilde{f}_i(\rv{\delta_{x_k}})),\rv{g_{\sigma}}(f(\rv{\delta_{x_k}})))=d_{TV}\left(\cN(\tilde{f}_i(x_k),\sigma^2I_p),\cN(f(x_k),\sigma^2I_p)\right)\\
            &\leq \frac{1}{2\sigma}\|\tilde{f}_i(x_k)-f(x_k)\|_2 \leq \frac{1}{2\sigma} d_{\cW}\left(\tilde{f_i}(\rv{\delta_{x_k}}),f(\rv{\delta_{x_k}})\right).
\end{aligned}
\end{equation*}
Therefore,
\begin{equation*}
    \begin{aligned}
    & d_{TV}^{\ell_2}\left((\rv{g_{\sigma}} \circ \tilde{f}_i)_{|\rv{S_z}},(\rv{g_{\sigma}} \circ f)_{|\rv{S_z}}\right)\\
       &=\sqrt{\frac{1}{m}\sum_{k=1}^m \left(d_{TV}\left(\rv{g_{\sigma}}(\tilde{f}_i(\rv{\delta_{x_k}})),\rv{g_{\sigma}}(f(\rv{\delta_{x_k}}))\right)\right)^2}\\
       & \leq \sqrt{\frac{1}{m}\sum_{k=1}^m\frac{\left(\,d_{\cW}\left(\tilde{f_i}(\rv{\delta_{x_k}}),f(\rv{\delta_{x_k}})\right)\right)^2}{(2\sigma)^2}}\\
       &\leq \frac{1}{2\sigma} d_{\cW}^{\ell_2}(\tilde{f_i}_{|S_z},f_{|S_z}) \leq \frac{\epsilon}{2\sigma}.
    \end{aligned}
\end{equation*}
Therefore, for any $(\rv{g_{\sigma}}\circ f)_{|\rv{S_z}} \in (\rv{\cG_{\sigma}}\circ\cF)_{|\rv{S_z}}$, there exists $\tilde{f_i}_{|\rv{S_z}}\in \rv{C_z}$ such that $\displaystyle d_{TV}^{\ell_2}\left((\rv{g_{\sigma}} \circ \tilde{f}_i)_{|\rv{S_z}},(\rv{g_{\sigma}} \circ f)_{|\rv{S_z}}\right)\leq \frac{\epsilon}{2\sigma}$. As a result,
\begin{equation*}
    N(\frac{\epsilon}{2\sigma},(\rv{\cG_{\sigma}}\circ\cF)_{|S_z},d_{TV}^{\ell_2}) \leq N(\epsilon, \cF_{|S_z}, d_{\cW}^{\ell_2}).
\end{equation*}
Since this results holds for any $\rv{S_z}\subset \rv{\Delta_d}$, we can conclude that
\begin{equation*}
    N_U(\frac{\epsilon}{2\sigma},\rv{\cG_{\sigma}}\circ\cF,m,d_{TV}^{\ell_2},\rv{\Delta_d}) \leq N_U(\epsilon, \cF,m,d_{\cW}^{\ell_2},\rv{\Delta_d}) = N_{U}(\epsilon,\cF,m,\|.\|_2^{\ell_2}).
\end{equation*}
\end{proof}

\subsection{Proof of Theorem \ref{thm:global_ell2_to_tv}}
\begin{proof}
Let $\rv{\cQ}=\rv{\cG_\sigma} \circ \cF$. Denote by $r=N_U(\epsilon,\cF,\infty,\|.\|_2^{\infty})$. Let $C=\{\hat{{f_i}}(x) \mid \hat{f}_i\in\cF,\forall x\in\bR^d, i\in [r]\}$ be a global $\epsilon$-cover for $\cF$ with respect to $\|.\|_2$ metric. We will show that for all $(\rv{g_\sigma}\circ f)_{|\rv{\cX_{B,d}}},\,f\in\cF$, there exists $\hat{f}_i\in \cF$ such that $d_{TV}^{\infty}\left((\rv{g_\sigma}\circ f)_{|\rv{\cX_{B,d}}},(\rv{g_\sigma}\circ \hat{f}_i)_{|\rv{\cX_{B,d}}}\right)\leq \frac{\epsilon}{2\sigma}$. Clearly, $|C|\leq r$ and the result follows.

Since $C$ covers the restriction of $\cF$ to $\bR^d$, for any $f\in \cF$, there exists $\hat{{f_i}}$ such that $\|f(x)-\hat{f}_i(x)\|_2\leq \epsilon$ for every $x\in\bR^d$. Next, for any $\rv{x}\in\rv{\cX_{B,d}}$ and for the coupling $\pi^*(f(\rv{x}),\hat{f_i}(\rv{x}))$ as defined in Notations we can write
\begin{equation*}
    \begin{aligned}
     \int_{\bR^d \times \bR^d}\|x-y\|_2d\pi^*(x,y)\leq \epsilon \int_{\bR^d \times \bR^d}d\pi^*(x,y)\leq \epsilon,
    \end{aligned}
\end{equation*}
which comes from the fact that $\hat{f}_i$ is ``globally close'' to $f$ with respect to $\|.\|_2$ distance. We, therefore, know that
\begin{equation*}
\begin{aligned}
     d_{\cW}(f(\rv{x}),\hat{f_i}(\rv{x}))=&\inf_{\pi\in\Pi(f(\rv{x}),\hat{f_i}(\rv{x}))}\int_{\bR^d \times \bR^d}\|x-y\|_2d\pi(x,y)\\
     &\leq \int_{\bR^d \times \bR^d}\|x-y\|_2d\pi^*(x,y)\leq \epsilon.
\end{aligned}
\end{equation*}
Since this holds for any $\rv{x}\in\rv{\cX_{B,d}}$, we can conclude that
\begin{equation*}
    d_{\cW}^{\infty}\left (f_{|\rv{\cX_{B,d}}},\hat{f_i}_{|\rv{\cX_{B,d}}}\right )\leq \epsilon.
\end{equation*}
Next, from the arguments in Theorem~\ref{thm:wasser_tv} and Equation~\ref{eq:A8}, we know that 
\begin{equation*}
    d_{TV}^{\infty}\left( (\rv{g}_\sigma\circ f)_{|\rv{\cX_{B,d}}},(\rv{g}_\sigma \circ \hat{f_i})_{|\rv{\cX_{B,d}}}\right)\leq \frac{1 }{2\sigma}d_{\cW}^{\infty}\left (f_{|\rv{\cX_{B,d}}},\hat{f_i}_{|\rv{\cX_{B,d}}}\right )\leq\frac{\epsilon}{2\sigma},
\end{equation*}

which is exactly what we wanted to prove. Therefore, the size of the TV cover for $\rv{\cG_{\sigma}}\circ \cF$ can be bounded by the size of $\|.\|_2$ cover of $\cF$,
\begin{equation*}
      N_U(\frac{\epsilon}{2\sigma},\rv{\cG_\sigma}\circ\cF,\infty,d_{TV}^{\infty},\rv{\cX_{B,d}})\leq N_U(\epsilon,\cF,\infty,\|.\|_2^{\infty}).
\end{equation*}
\end{proof}

\section{Proof of theorem in Section \ref{sub:NN}}\label{app:sub:NN}
\paragraph{Notation.} For a vector $V\in\bR^d$, we denote its angle by $\angle V$. By $\angle (V_1,V_2)$, we are referring to the angle between two vectors $V_1$ and $V_2$. Also, we denote by $\indicator\{x=a\}$ the indicator function that outputs $1$ if $x=a$ and $0$ if $x\neq a$. We also denote by $\langle V_1,V_2\rangle$ the inner product between vectors $V_1$ and $V_2$. We denote by $\mD(\rv{x})$ the probability density function of the random variable $\rv{x}$. For two Borel functions $f_1$ and $f_2$, we denote by $\pi^*(f_1(\rv{x}),f_2(\rv{x}))$ a coupling between random variables $f_1(\rv{x}),f_2(\rv{x})$ such that 
\begin{equation*}
    \mM_{\pi^*}(A)=\begin{cases}
    \mM_{\rv{x}}(B) & \exists B\subset \cB(\cX) \text{ such that } A=f_1(B)\times f_2(B) \\
    0 & \text{otherwise},
    \end{cases}
\end{equation*}
where $\cB(\cX)$ is the set of all Borel sets over $\cX$, $\mM_{\pi^*}(A)$ is the measure that $\pi^*$ assigns to the Borel set $A$, and $\mM_{\rv{x}}(B)$ is the measure that random variable $\rv{x}$ assigns to Borel set $B$. We also denote by $Ball_d(x,R)$ the $d$ dimensional ball of radius $R$ centered at $x$.

\subsection{Proof of Theorem~\ref{lemma:net}}
In the following we state a stronger version of Theorem~\ref{lemma:net} which presents a uniform covering number bound for neural network classes that have a general activation function that is Lipschitz continuous, monotone, and has a bounded domain.

\begin{theorem}[Stronger version of Theorem~\ref{lemma:net}] \label{lemma:strong_net}
Consider the class $\net{d}{p}{L}$ of single-layer neural networks, where the activation function is Lipschitz continuous with Lipschtiz factor $L$, monotone, and has a bounded output in $[-B,B]^p$. The global covering number of $\rv{\cG_{\sigma}}\circ\net{d}{p}{L}$ with respect to total variation distance is bounded by
\begin{align*}
   &N_U(\epsilon,\rv{\cG_{\sigma}}\circ\net{d}{p}{L},\infty,d_{TV}^{\infty},\rv{\cG_\sigma}\circ\rv{\cX_{B,d}}) \\
   &\leq  \left(\frac{4(4+B)^{3/2}}{(2\pi)^{1/4}}\frac{d^{5/2}L\sqrt{Bu}}{\epsilon^{3/2}\sigma^2}\ln\left(\frac{(4+B)Bd}{\epsilon\sigma}\right)\right)^{p(d+1)},
\end{align*}
where $u = \max\left\{\left|\phi^{-1}\left(B-\sigma \epsilon/((4+B)d)\right)\right|,\left|\phi^{-1}\left(-B+\sigma \epsilon/((4+B)d)\right)\right|\right\}$.
\end{theorem}
Note that Theorem~\ref{lemma:net} is a special case of the above theorem where the activation function is the sigmoid function with Lipschitz continuity factor of 1 and a bounded domain in $[0,1]^p$. In the case of sigmoid function, we can also conclude that 
\begin{equation*}
    \begin{aligned}
     u &= \max\left\{\left|\phi^{-1}\left(1-\epsilon\sigma /((4+B)d)\right)\right|,\left|\phi^{-1}\left(\epsilon\sigma /((4+B)d)\right)\right|\right\}\\
     &= \left|\phi^{-1}\left(1- \epsilon\sigma /((4+B)d)\right)\right|\\
     & = \ln\left(((4+B)d-\epsilon\sigma)/(\epsilon\sigma)\right)\\
     & \leq \ln\left((5d-\epsilon\sigma)/(\epsilon\sigma)\right).
    \end{aligned}
\end{equation*}
\begin{proof}
We bound the global covering number of class $\net{d}{p}{L}=\{f:\bR^d\rightarrow\bR^p \mid f(x)=\Phi(W\transpose x)\}$ with respect to Wasserstein distance by constructing a grid for the weights $V_i\in\bR^{d}$ of $W\transpose=[V_1\transpose \ldots V_p\transpose]$. Then, we find the TV covering number using Theorem~\ref{thm:wasser_tv}. To construct the grid, we consider two cases for each $V_i$ based on its $\ell_2$ norm. In case $\|V_i\|_2\leq B_v$, we construct the grid based on $\|V_i\|_2$ and its angle, while for the case that $\|V_i\|_2>B_v$, we prove that only a grid on the angle of $V_i$ is sufficient. Further, we choose $B_v$ based on $\epsilon$ and $\sigma$. We then show that for each matrix $W\transpose=[V_1\transpose \ldots V_p\transpose]$, there exists $\hat{W}\transpose=[\hat{V}_1\transpose \ldots \hat{V}_p\transpose]$ in the grid such that $d_{\cW}\left(\Phi(W\transpose \rv{x}),\Phi(\hat{W}\transpose \rv{x})\right)$ is bounded for all $\rv{x}\in \rv{\cG_{\sigma}}\circ \rv{\cX_{B,d}}$.

Denote $r=\lceil \frac{2B_v}{\delta} \rceil$ and
\begin{equation}\label{eq:As}
    A=\{-B_v+i\delta\mid \in [r]\}^d.
\end{equation}
Define a new set
\begin{equation*}
    A_S=\left\{(a_1,\ldots,a_d) \in A \mid \left(\sum_{i=1}^d \indicator\{a_i=B_v\}+\sum_{i=1}^d\indicator\{a_i=-B_v\}\right)\geq 1\right\}.
\end{equation*}
Informally, $A_S$ is the grid of points on sides of a $d$-dimensional hypercube. For any point $b=(b_1,\ldots,b_d)\in A_S$, we define the following set of vectors
\begin{equation*}
    P_b=\{\frac{i\zeta}{B_v}[b_1 \ldots b_d]\in\bR^d\mid i\in[\lceil \frac{B_v}{\zeta}\rceil]\}.
\end{equation*}
Note that the way we defined $A_S$ in Equation~\ref{eq:As}, implies that for any $(b_1,\ldots,b_d)\in A_S$, there exists at least one $b_i$ such that $|b_i|=B_v$. Therefore, whenever $i=\lceil \frac{B_v}{\zeta}\rceil$, we know that $\|\frac{i\zeta}{B_v}[b_1 \ldots b_d]\|_2\geq B_v$.

Now, we can define the grid of vectors $V\in\bR^d$ in the following way
\begin{equation*}\label{grid}
    C=\bigcup_{b\in A_S}P_b.
\end{equation*}
Informally speaking, we are discretizing the norms in $\lceil \frac{B_v}{\zeta} \rceil$ values and then for each vector from origin to gird points on the sides of the hypercube, we use $\lceil \frac{B_v}{\zeta} \rceil$ vectors with the same angle and different norms as our grid. Clearly, the size of grid $|C|$ is upper bounded by $\lceil \frac{B_v}{\zeta} \rceil\lceil \frac{2B_v}{\delta}\rceil^d$.

Next, we turn into proving that given any vector $V$ in $\bR^d$, there exists a vector $\hat{V}$ in $C$ such that for any $\rv{z}\in\rv{\cG_\sigma} \circ \rv{\cX_{B,d}}$, $d_\cW(\phi(V\transpose \rv{z}),\phi(\hat{V}\transpose \rv{z}))\leq (B+4)\epsilon$.

\paragraph{Case 1.} In this case, we consider vectors $V\in\bR^d$ such that $\|V\|_2\leq B_v$.
The way that we constructed the set of vectors $C$ implies that given any vector there exists a $b\in A_S$ and the set of aligned vectors $P_b$ such that the angle between $V$ and vectors in set $P_b$ can be bounded. More specifically, for any $V'\in P_b$, we know that
\begin{equation*}
    \angle (V,V') \leq \arcsin{\frac{\delta}{B_v}},
\end{equation*}
since $\arcsin$ is a monotone increasing functions over $[-1,1]$ and we know that $\left\|[b_1 \ldots b_d]\right\|_2\geq B_v$.
Let $\theta=\arcsin{\frac{\delta}{B_v}}$. Moreover, since $\|V\|_2\leq B_v$, we know that there exists $\hat{V}\in P_b$ such that 
\begin{equation*}
    \left| \|V\|_2-\|\hat{V}\|_2\right|\leq \frac{\zeta}{B_v}\|[b_1\ldots b_d]\|_2\leq \frac{\zeta}{B_v}\sqrt{d}B_v\leq \sqrt{d}\zeta.
\end{equation*}
Without loss of generality, let $ \|V\|_2 \leq \|\hat{V}\|_2$. We can then write
\begin{equation*}
    \frac{\|\hat{V}\|_2}{\|V\|_2}\leq 1+\frac{\sqrt{d}\zeta}{\|V\|_2}
\end{equation*}
Denote $\hat{V}_{\bot}=\|\hat{V}\|_2 sin(\angle (V,\hat{V}))V_{\bot}$ and $\hat{V}_{\parallel}=\|\hat{V}\|_2cos(\angle (V,\hat{V}))\frac{V}{\|V\|_2}$, where $V_{\bot}$ is a normalized vector orthogonal to $V$.
Denote $B_z=(B+\sigma)\sqrt{d}+\sigma \sqrt{2\ln \frac{B}{\epsilon}}$. For any $x\in\bR^d$ such that $\|x\|_2\leq B_z$, we can write
\begin{equation*}
\begin{aligned}
  &\langle \hat{V},x\rangle = \langle \hat{V}_{\bot},x\rangle + \langle \hat{V}_{\parallel},x\rangle=\langle \hat{V}_{\bot},x\rangle + \langle V,x\rangle \frac{\|\hat{V}_{\parallel}\|_2}{\|V\|_2}\\
  &=\|\hat{V}_{\bot}\|_2\|x\|_2cos(\angle (\hat{V}_{\bot},x))+ \langle V,x\rangle \frac{\|\hat{V}_{\parallel}\|_2}{\|V\|_2}\\
  &\leq \|\hat{V}_{\bot}\|_2\|x\|_2+ \langle V,x\rangle \frac{\|\hat{V}_{\parallel}\|_2}{\|V\|_2}\\
   &\leq \|\hat{V}\|_2\|x\|_2 sin(\angle(V,\hat{V}))+ \langle V,x\rangle \frac{\|\hat{V}\|_2cos(\angle(V,\hat{V}))}{\|V\|_2}\\
  & \leq \|\hat{V}\|_2\|x\|_2 \frac{\delta}{B_v} + \langle V,x\rangle \frac{\|\hat{V}\|_2}{\|V\|_2}\\
   & \leq \sqrt{d}B_v\|x\|_2 \frac{\delta}{B_v} + \langle V,x\rangle (1+\frac{\sqrt{d}\zeta}{\|V\|_2}).
\end{aligned}
\end{equation*}
Therefore, we can conclude that
\begin{equation}\label{eq:190}
    \begin{aligned}
     &\langle \hat{V},x\rangle -\langle V,x\rangle\leq \sqrt{d}B_v\|x\|_2 \frac{\delta}{B_v} + \|V\|_2\|x\|_2(\frac{\sqrt{d}\zeta}{\|V\|_2})\\
     &\leq(\sqrt{d}\delta+ \sqrt{d}\zeta)\|x\|_2\\
     &\leq (\sqrt{d}\delta+ \sqrt{d}\zeta)\left((B+\sigma)\sqrt{d}+\sigma \sqrt{2\ln \frac{B}{\epsilon}}\right).
    \end{aligned}
\end{equation}
Now, for any $\rv{z}\in\rv{\cG_\sigma}\circ \rv{\cX_{B,d}}$, by Lemma~\ref{lema:gmm}, we know that we can find a mixture of $m=\lceil\frac{B}{\eta}\rceil^d$ $d$-dimensional Gaussian random variables $\rv{h}=\sum_{i=1}^m w_ig_i$ with bounded means in $[-B,B]^d$ and covariance matrices $\sigma^2I_d$ such that $d_{TV}(\rv{h},\rv{z})\leq 2\sqrt{d}\eta/\sigma$. Let $\rv{\cH}$ be the class of all such mixtures.

From Lemma~\ref{lemma:norm_gaussian}, we know that
\begin{equation}
    \bP\left[\|x\|_2^2\geq (B+\sigma)\sqrt{d}+\sigma \sqrt{2t} \right]\leq e^{-t}.
\end{equation}
Setting $t=\ln\frac{B}{\epsilon}$ and $\delta=\zeta=\epsilon/(2dL \ln\frac{B}{\epsilon})$, we can conclude that
\begin{equation}\label{eq:191}
    \bP\left[\|x\|_2\geq B_z\right]=\bP\left[\|x\|_2\geq (B+\sigma)\sqrt{d}+\sigma \sqrt{2\ln \frac{B}{\epsilon}} \right]\leq \frac{\epsilon}{B}.
\end{equation}
Therefore, from Equations~\ref{eq:190} and \ref{eq:191}, we can conclude that for the random variable $\rv{h}=\sum_{i=1}^m w_ig_i$ with $\mD(h)=I_h$ and for the coupling $\pi^*$ of $\phi(V\transpose \rv{h})$ and $\phi(\hat{V}\transpose \rv{h})$ as defined in notations we can write
\begin{equation}\label{eq:192}
    \begin{aligned}
     &\int_{\bR^d \times \bR^d}\|x-y\|_2d\pi^*\left(x,y\right)\\
     &\leq \int_{Ball_d(0,B_z)}L\sqrt{d}(\delta+\zeta)\left((B+\sigma)\sqrt{d}+\sigma \sqrt{2\ln \frac{B}{\epsilon}}\right)
     \,dI_h\\
     &+\int_{\bR^d\setminus Ball_d(0,B_z)}2B\, dI_h\\
     &\leq \frac{(B+\sigma)\epsilon}{2\ln \frac{B}{\epsilon}} + \frac{\epsilon\sigma}{\sqrt{2d\ln\frac{B}{\epsilon}}} + 2\epsilon,
    \end{aligned}
\end{equation}
where we used the fact that for any $x\in\bR^d$, we know that $\|V\transpose x-\hat{V}\transpose x\|_2$ is bounded and the activation function $\phi(x)$ is Lipschitz continuous with Lipschitz constant $L$. Here, we assume that the variance of noise is always smaller than 1, i.e., $\sigma\leq 1$. We know that $d\geq 1$ and assuming that $\ln\frac{B}{\epsilon}\geq 1$ (*), we can rewrite Equation~\ref{eq:192} as
\begin{equation*}
    \begin{aligned}
     &\int_{\bR^d \times \bR^d}\|x-y\|_2d\pi^*\left(x,y\right) \leq (B+1) \epsilon + \epsilon +2\epsilon \leq (B+4)\epsilon,
    \end{aligned}
\end{equation*}
Then, we have
\begin{equation*}
\begin{aligned}
      &d_{\cW}\left(\phi(V\transpose \rv{h}),\phi(\hat{V}\transpose \rv{h})\right)=\inf_{\pi\in\Pi\left(\phi(V\transpose \rv{h}),\phi(\hat{V}\transpose  \rv{h})\right)}\int_{\bR^d \times \bR^d}\|x-y\|_2d\pi\left(x,y\right)\\
      &\leq \int_{\bR^d \times \bR^d}\|x-y\|_2d\pi^*\left(x,y\right)\leq (B+4)\epsilon.
\end{aligned}
\end{equation*}
Therefore, we have proved that for any $V\in\bR^d$ such that $\|V\|_2\leq B_v$, there exists a vector $\hat{V}$ in $C$ such that for any $\rv{z}\in\rv{\cG_\sigma}\circ \rv{\cX_{B,d}}$ and its estimation with a mixture $\rv{h}$ of Gaussian random variables, we have
\begin{equation*}
    d_{\cW}\left(\phi(V\transpose \rv{h}),\phi(\hat{V}\transpose \rv{h})\right)\leq (B+4)\epsilon.
\end{equation*}
\paragraph{Case 2} Now, we turn to analyze the case where we have vectors $V$ in $\bR^d$ such that $\|V\|_2>B_v$. We assume that the function $\phi$ is invertible. Taking into account that $\phi$ is also bounded in $[-B,B]$, denote $u=\max\left\{|\phi^{-1}(B-\epsilon)|,|\phi^{-1}(-B+\epsilon)|\right\}$. For a given vector $V\in\bR^d$, select $b\in A_S$ such that for all $V'\in P_b$, we have $\angle(V,V')\leq \theta$, where $\theta$ is defined the same as case 1. From all vectors in $P_b$, select $\hat{V}$ such that it has the maximum $\ell_2$ norm, i.e., the one on the side of the hypercube. It is obvious that $\|\hat{V}\|_2\geq B_v$. We will show that for any $\rv{h}\in\rv{\cH}$, the Wasserstein distance between $\phi(V\transpose \rv{h})$ and $\phi(\hat{V}\transpose \rv{h})$ is bounded.

Define following two sets
\begin{equation}
    \begin{aligned}
     S_1=\{x\in\bR^d\mid |\langle V,x\rangle|\leq u\},\\
     S_2=\{x\in\bR^d\mid |\langle \hat{V},x\rangle|\leq u\}.
    \end{aligned}
\end{equation}
Given any $x\in \bR^d\setminus S_1\cup S_2$ such that $\|x\|_2\leq B_z$, we show that both of  $\langle V,x\rangle$ and $\langle \hat{V},x\rangle$ are either smaller than $-u$ or larger than $u$. Assume that $\langle \hat{V},x\rangle>u$. Denote $\alpha=\angle(\hat{V},x)$ and $\beta=\angle(V,\hat{V})$. From the fact that $\langle \hat{V},x \rangle=\|\hat{V}\|_2\|x\|_2\cos\alpha\geq u$, we conclude that $\cos \alpha\geq 0$. On the other hand, to conclude that $\langle V,x \rangle$ is also larger than $u$, we only need to prove that $\langle V,x \rangle\geq 0$ since $x\in\bR^d\setminus S_1\cup S_2$ and we already know that $|\langle V,x \rangle|\geq u$. Therefore, we want to prove that   $\langle V,x \rangle=\|V\|_2\|x\|_2\cos(\alpha\pm\beta)\geq 0$. It implies that we need to prove $\cos\alpha\geq\sin\beta$. But we know that
\begin{equation*}
\begin{aligned}
   \cos\alpha&\geq \frac{u}{\|\hat{V}\|_2\|x\|_2}\\
   &\geq \frac{u}{\|\hat{V}\|_2B_z} && \text{(Since $\|x\|_2\leq B_z$)}\\
   & \geq \frac{u}{\sqrt{d}B_vB_z} && \text{(Since $\hat{V}\in P_b$ and $\|\hat{V}\|_2\leq \sqrt{d}B_v$)}\\
   &\geq \frac{B-\epsilon}{LB_vB_z\sqrt{d}}\\
   &\geq \frac{\delta}{B_v}\geq \sin\theta\geq\sin\beta,
\end{aligned}
\end{equation*}
where we used the fact that the function $\phi$ is Lipschitz continuous and we know that $|\phi(u)-\phi(-u)|\leq 2Lu$. The last line follows from the fact that $B_z\leq \left((B-\epsilon)/\epsilon\right)\left(2\sqrt{d}\ln(B/\epsilon)\right)$ (**). It is easy to verify in the same way that if $\langle \hat{V},x\rangle\leq -u$, then $\langle V,x\rangle\leq -u$. 

Next, since $\phi$ is monotone, we can conclude that for any $x\in\bR^d\setminus S_1\cup S_2$ such that $\|x\|_2\leq B_z$, we have either both $V\transpose x,\hat{V}\transpose x$ in $[B-\epsilon,B]$ or both $V\transpose x,\hat{V}\transpose x$ in $[-B,-B+\epsilon]$, which means that $|V\transpose x-\hat{V}\transpose x|\leq \epsilon$. Setting $B_v^2 = 4Bu/(\epsilon\sigma\sqrt{2\pi})$, for any mixture of Gaussian random variables $\rv{h}\in\rv{\cH}$ and for the coupling $\pi^*$ of $\phi(V\transpose \rv{h})$ and $\phi(\hat{V}\transpose \rv{h})$, we can write
\begin{equation*}
    \begin{aligned}
     &\int_{\bR^d \times \bR^d}\|x-y\|_2d\pi^*\left(x,y\right)\\
     &\leq \int_{Ball_d(0,B_z)\setminus S_1\cup S_2}\epsilon
     dI_h + \int_{S_1\cup S_2} 2B dI_h + \int_{\bR^d\setminus Ball_d(0,B_z)} 2B dI_h\\
     &\leq \epsilon + 4B\frac{u}{\sqrt{2\pi}\sigma B_v^2}+2\epsilon\\
     & \leq 4\epsilon ,
    \end{aligned}
\end{equation*}
where we used the union bound and the fact that $x\in S_1$ is similar to the probability that $|x|\leq  u/B_v$ for the zero mean Gaussian random variable $x$ with variance equal to $(\sigma B_v)^2$. We can, again, write that
\begin{equation*}
\begin{aligned}
      &d_{\cW}\left(\phi(V\transpose \rv{h}),\phi(\hat{V}\transpose \rv{h})\right)=\inf_{\pi\in\Pi\left(\phi(V\transpose \rv{h}),\phi(\hat{V}\transpose  \rv{h})\right)}\int_{\bR^d \times \bR^d}\|x-y\|_2d\pi\left(x,y\right)\\
      &\leq \int_{\bR^d \times \bR^d}\|x-y\|_2d\pi^*\left(x,y\right)\leq 4\epsilon.
\end{aligned}
\end{equation*}
So far, we proved that for any $V\in\bR^d$ there exists a $\hat{V}\in C$ such that $d_{\cW}\left(\phi(V\transpose \rv{h}),\phi(\hat{V}\transpose \rv{h})\right)\leq (4+B)\epsilon$ for all mixtures $\rv{h}\in\cH$, which comes from the fact that $4\epsilon\leq (4+B)\epsilon$.
Now, we turn to covering functions in \net{d}{p}{L}. Note that the output of $\phi(V\transpose x)$ is real-valued. We also know that $\Phi$ is applied element-wise. Consider the set
\begin{equation*}
    C_{W}=\{[V_1\transpose \ldots V_p\transpose]\transpose\mid V_i\in C \text{ for } i\in[p]\}.
\end{equation*}
We know that for any $W=[V_1\transpose \ldots V_p\transpose]\transpose$ there exists $\hat{W}\transpose=[\hat{V}_1\transpose \ldots \hat{V}_p\transpose]\transpose$ such that for every $i \in [p]$, we have $d_{\cW}\left(\phi(V_i\transpose \rv{h}),\phi(\hat{V}_i\transpose \rv{h})\right)\leq (4+B)\epsilon$. Therefore, since we keep the coupling the same $\pi^*$ for every $i\in[p]$, we can conclude that $d_{\cW}\left(\Phi(W\transpose \rv{h}),\Phi(\hat{W}_i\transpose \rv{h})\right)\leq (4+B)\epsilon d$.

Now, using Theorem~\ref{thm:wasser_tv}, we get that
\begin{equation}\label{eq:193}
d_{TV}\left(  \rv{g}_{\sigma}(\Phi(W\transpose \rv{h})),\rv{g}_{\sigma}(\Phi(\hat{W}\transpose \rv{h}))\right)\leq \frac{(4+B)\epsilon d}{2\sigma}
\end{equation}
Consequently, for any $\rv{z}\in\rv{\cG_{\sigma}}\circ\rv{\cX}_{B,d}$, we can write
\begin{equation}\label{eq:194}
\begin{aligned}
 & d_{TV}\left(  \rv{g}_{\sigma}(\Phi(W\transpose \rv{z})),\rv{g}_{\sigma}(\Phi(\hat{W}\transpose \rv{z}))\right)\\
  &\leq d_{TV}\left(  \rv{g}_{\sigma}(\Phi(W\transpose \rv{z})),\rv{g}_{\sigma}(\Phi(W\transpose \rv{h}))\right)\\
   &+ d_{TV}\left(  \rv{g}_{\sigma}(\Phi(W\transpose \rv{h})),\rv{g}_{\sigma}(\Phi(\hat{W}\transpose \rv{h}))\right)\\
   &+ d_{TV}\left(  \rv{g}_{\sigma}(\Phi(\hat{W}\transpose \rv{h})),\rv{g}_{\sigma}(\Phi(\hat{W}\transpose \rv{z}))\right)\\
   &\leq \frac{4\sqrt{d}\eta}{\sigma}+(4+B)\frac{\epsilon d}{2\sigma},
\end{aligned}
\end{equation}
where we used data processing inequality and Equation~\ref{eq:193}. Equation \ref{eq:194} implies that $C_W$ is a global cover for $\rv{\cG_{\sigma}}\circ\net{d}{p}{L}$ with respect to $d_{TV}$ metric. Clearly,
    \begin{equation*}
    |C_W|\leq \left(\frac{(B_v)^{d+1}}{\delta^d\zeta}\right)^p = \left(\frac{2B_vdL\ln\frac{B}{\epsilon}}{\epsilon}\right)^{p(d+1)}.
\end{equation*} 
Therefore, setting $\eta=\sqrt{d}(4+B)\epsilon/8$ and $\epsilon' =  \epsilon\sigma/((4+B)d)$ we conclude that
\begin{equation}\label{eq:195}
    \begin{aligned}
   & N_U\left(\epsilon,\rv{\cG_{\sigma}}\circ\net{d}{p}{L},\infty,d_{TV},\rv{\cG_\sigma}\circ\rv{\cX_{B,d}}\right) \leq \left(\frac{2(4+B)d^{2}LB_v}{\epsilon\sigma}\ln\left(\frac{(4+B)Bd}{\epsilon\sigma}\right)\right)^{p(d+1)}
    \\
   &\leq \left(\frac{4(4+B)^{3/2}}{(2\pi)^{1/4}}\frac{d^{5/2}L\sqrt{Bu'}}{\epsilon^{3/2}\sigma^2}\ln\left(\frac{(4+B)Bd}{\epsilon\sigma}\right)\right)^{p(d+1)}
   ,
\end{aligned}
\end{equation}
where 
\begin{equation*}
\begin{aligned}
  u' &= \max\left\{|\phi^{-1}(B-\epsilon')|,|\phi^{-1}(-B+\epsilon')|\right\}\\
  &=\max\left\{\left|\phi^{-1}\left(B- \epsilon\sigma/((4+B)d)\right)\right|,\left|\phi^{-1}\left(-B+ \epsilon\sigma/((4+B)d)\right)\right|\right\},
\end{aligned}
\end{equation*}
and 
\begin{equation*}
    \sigma \leq \frac{(4+B)Bd}{\epsilon}.
\end{equation*}
Note that we always use $\sigma\leq 1$. In that case, having $\sigma > (4+B)Bd/\epsilon$ means that $\epsilon > (4+B)Bd>B\sqrt{d}$. On the other hand, the domain of the output of $\Phi$ is in $[-B,B]^d$ and, therefore, in this case the covering number would be simply one and no further analysis is required. Furthermore, the assumption (*) always holds since in order to obtain an $\epsilon$-cover for the single-layer neural network, we will need to bound the Wassestein distance between $\phi(V\transpose\bar{h})$ and $\phi(\hat{V}\transpose\bar{h})$ by $(4+B)\epsilon'$. In this case we have
\begin{equation*}
    \begin{aligned}
    & \ln \frac{B}{\epsilon'} \geq 1\\
     \Leftrightarrow \, & \frac{B}{\epsilon'} \geq e\\
     \Leftrightarrow \, & \frac{B}{e} \geq \frac{\epsilon\sigma}{(4+B)d} \\
     \Leftrightarrow \, & \frac{(4+B)d}{e\sigma}B \geq \epsilon,
    \end{aligned}
\end{equation*}
which holds since we consider $\sigma \leq 1$ and $\epsilon\leq B\sqrt{d}$. Moreover, for assumption (**) to hold, we need
\begin{equation*}
    \begin{aligned}
   & B_z\leq \left(\frac{B-\epsilon'}{\epsilon}\right)2\sqrt{d}\ln(\frac{B}{\epsilon'})\\
     \Leftrightarrow \,& (B+\sigma)\sqrt{d}+\sigma\sqrt{2\ln\frac{B}{\epsilon'}} \leq \left(\frac{B-\epsilon'}{\epsilon'}\right)2\sqrt{d}\ln(\frac{B}{\epsilon'})\\
     \Leftrightarrow \,&  \frac{B+1}{\sqrt{\ln \frac{B}{\epsilon'}}}+ \frac{\sqrt{2}}{\sqrt{d}} \leq 2\left(\frac{B-\epsilon'}{\epsilon'}\right) \sqrt{\ln\frac{B}{\epsilon'}}\\
     \Leftrightarrow \,&  \frac{B+1}{(\ln \frac{B}{\epsilon'})^{1/4}}+ \frac{\sqrt{2}}{\sqrt{d\ln(\frac{B}{\epsilon'})}}\leq 2\left(\frac{B-\epsilon'}{\epsilon'}\right)\\
       \Leftrightarrow \,&  \left(\frac{B+1}{\ln \frac{B}{\epsilon'}}+ \frac{\sqrt{2}}{\sqrt{d\ln(\frac{B}{\epsilon'})}}\right)\frac{\epsilon'}{2}\leq B-\epsilon'\\
      \Leftrightarrow \,&  \left(\frac{B+1+\sqrt{2}}{2}+1\right)\epsilon'\leq B\\
      \Leftrightarrow \,&  \left(\frac{B+3+\sqrt{2}}{2}\right)\left(\frac{\epsilon\sigma}{(4+B)d}\right)\leq B\\
      \Leftrightarrow \,&  \epsilon \leq \frac{2(4+B)d}{(B+3+\sqrt{2})\sigma}B,
    \end{aligned}
\end{equation*}
which is always true if $\sigma\leq 1$. Note that in both (*) and (**) we were interested in values of $\epsilon$ that are smaller than $ B\sqrt{d}$; Otherwise, the covering number would be one.
\end{proof}

We can also simplify the constants and write Equation~\ref{eq:195} as
\begin{equation*}
    \begin{aligned}
   & N_U\left(\epsilon,\rv{\cG_{\sigma}}\circ\net{d}{p}{L},\infty,d_{TV},\rv{\cG_\sigma}\circ\rv{\cX_{B,d}}\right) \leq \left(2.6(4+B)^{3/2}\frac{d^{5/2}L\sqrt{Bu'}}{\epsilon^{3/2}\sigma^2}\ln\left(\frac{(4+B)Bd}{\epsilon\sigma}\right)\right)^{p(d+1)}.
\end{aligned}
\end{equation*}
Also since $\phi$ is a monotone function, we can approximate $u'$ by
\begin{equation*}
\begin{aligned}
  u' \leq \max\left\{\left|\phi^{-1}\left(B-\frac{\sigma \epsilon}{(4+B)d}\right)\right|,\left|\phi^{-1}\left(-B+\frac{\sigma \epsilon}{(4+B)d}\right)\right|\right\}.
\end{aligned}
\end{equation*}

\section{Proofs of theorems and and lemmas in Section \ref{sec:applications}}\label{app:sec:application}
\subsection{Proof of Theorem~\ref{thm:neuralnet}}
\begin{proof}
We will prove the theorem for the stronger case where the output of single-layer neural network classes and $\cH$ is in $[-B,B]^{p_T}$. In the case of sigmoid function, $\phi(x)$, the output is in $[0,1]^{p_T}$. Since adding a constant to the output of functions in a class does not change its covering number, we can replace the sigmoid activation function in the class of single-layer neural networks with $\phi(x)-1/2$. Therefore, we can assume $B=1/2$ and consider outputs to be in $[-1/2,1/2]^{p_T}$. Consider two consecutive classes $\net{p_i-1}{p_{i}}{L_{i}}$ and $\net{p_i}{p_{i+1}}{L_{i+1}}$. From Lemma~\ref{lemma:compose_tv} we know that
     \begin{equation}\label{d:eq1}
     \begin{aligned}
        & N_U\left(\frac{2\epsilon}{2BT\sqrt{p_T}},\rv{\cG_{\sigma}}\circ \net{p_i}{p_{i+1}}{L_{i+1}}\circ\rv{\cG_{\sigma}}\circ\net{p_{i-1}}{p_i}{L_i},\infty,d_{TV}^{\infty},\rv{\cG_\sigma}\circ\rv{\cX_{B,p_{i-1}}}\right)\\
        & \leq N_U\left(\frac{\epsilon}{2BT\sqrt{p_T}},\rv{\cG_{\sigma}}\circ\net{p_{i-1}}{p_i}{L_i},\infty,d_{TV}^{\infty},\rv{\cG_\sigma}\circ\rv{\cX_{B,p_{i-1}}}\right) \\
        &. N_U\left(\frac{\epsilon}{2BT\sqrt{p_T}},\rv{\cG_{\sigma}}\circ\net{p_i}{p_{i+1}}{L_{i+1}},\infty,d_{TV}^{\infty},\rv{\cG_\sigma}\circ\rv{\cX_{B,p_{i}}}\right)=N_i.N_{i+1}.
     \end{aligned}
      \end{equation}
Let
\begin{equation*}
        \rv{\cQ}=\rv{\cG_{\sigma}}\circ\net{p_{T-1}}{p_T}{L_T}     \circ \ldots \circ \rv{\cG_{\sigma}}\circ\net{p_1}{p_2}{L_2}.
    \end{equation*}
It is clear that $\rv{\cF} = \rv{\cQ} \circ    \rv{\cG_{\sigma}}\circ\net{d}{p_1}{L_1}.$
Equation \ref{d:eq1} is true for every $2\leq i \leq T$. Therefore, we can conclude that
\begin{equation*}
    N_U\left(\frac{(T-1)\epsilon}{2BT\sqrt{p_T}},\rv{\cQ},\infty,d_{TV}^{\infty},\rv{\cG_\sigma}\circ\rv{\cX_{B,p_{1}}}\right) \leq \prod_{i=2}^T N_i.
\end{equation*}
Moreover, corollary \ref{coll:ell2_to_tv} suggests that
\begin{equation*}
    N_U\left(\frac{\epsilon}{2BT\sqrt{p_T}},\rv{\cG_{\sigma}}\circ\net{d}{p_1}{L_1},\infty,d_{TV}^{\ell_2},\rv{\Delta_d}\right) \leq N_U\left(\frac{2\sigma \epsilon}{2BT\sqrt{p_T}},\net{d}{p_1}{L_1},\infty,\|.\|_2^{\ell_2},\rv{\Delta_d}\right)
\end{equation*}
Using Lemma~\ref{lemma:compose_tv}, we can again write that
\begin{equation*}
    \begin{aligned}
      & N_U\left(\frac{\epsilon}{2B\sqrt{p_T}},\rv{\cF},m,d_{TV}^{\ell_2},\rv{\Delta_d}\right) \\
      & \leq  N_U\left(\frac{(T-1)\epsilon}{2BT\sqrt{p_T}},\rv{\cQ},\infty,d_{TV}^{\infty},\rv{\cG_\sigma}\circ\rv{\cX_{B,p_{1}}}\right) .  N_U\left(\frac{\epsilon}{2BT\sqrt{p_T}},\rv{\cG_{\sigma}}\circ\net{d}{p_1}{L_1},\infty,d_{TV}^{\ell_2},\rv{\Delta_d}\right)\\
      & \leq \prod_{i=1}^T N_i.
    \end{aligned}
\end{equation*}
Finally, from Theorem~\ref{thm:tv_to_ell2} and the fact that $\rv{\cF}$ is a class of functions from $\bR^{d}$ to $[-B,B]^p$, we can conclude that
\begin{equation*}
    \begin{aligned}
     & N_U\left(\epsilon,\cH,m,\|.\|_2^{\ell_2}\right) \leq  N_U\left(\frac{\epsilon}{2B\sqrt{p_T}},\rv{\cF},m,d_{TV}^{\ell_2},\rv{\Delta_d}\right) \leq  \prod_{i=1}^T N_i.
    \end{aligned}
\end{equation*}
\end{proof}

\subsection{A technique to build deeper networks from networks with bounded covering number}
The following lemma is a technique that can be used to ``break'' networks in two parts. Then one can find a $\|.\|_2$ covering number for the first few layers and use Theorem~\ref{lemma:net} for the rest. It is a useful technique that enables the use of existing networks with bounded $\|.\|_2$ covering number to create deeper networks while controlling the capacity. Another possible application of the following lemma is that it gives us the opportunity to get tighter bounds on the covering number in special settings. One example of such settings would be networks that have small norms of weights in the first few layers and potentially large weights in the final layers. In this case, it is possible to use $\|.\|_2$ covering numbers that are dependent on the norms of weights for the first few layers and Theorem~\ref{lemma:net} for the rest, which does not depend on the norms of weights.
\begin{lemma}\label{lemma:compos_old}
 Let $\cQ$ be a class of functions (e.g., neural networks) from $\bR^d$ to $\bR^{p_0}$ and $\net{p_0}{p_1}{L_1}$, $\net{p_1}{p_2}{L_2},\ldots,\net{p_{T-1}}{p_T}{L_T}$ be $T$ classes of neural networks. Denote the composition of the $T$-layer neural network and $\cQ$ as  
    \begin{equation*}
        \rv{\cF}=\rv{\cG_{\sigma}}\circ\net{p_{T-1}}{p_T}{L_T}     \circ \ldots \circ \rv{\cG_{\sigma}}\circ\net{p_1}{p_2}{L_2}     \circ    \rv{\cG_{\sigma}}\circ\net{p_0}{p_1}{L_1}\circ \rv{\cG_{\sigma}}\circ \cQ,
    \end{equation*}
    and let $\cH=\{h:\bR^d\to [-B,B]^{p_T} \mid h(x)=\expects{\rv{f}}{~\rv{f}({x})}, \rv{f}\in \rv{\cF} \}$.
    Define the uniform covering numbers of composition of neural network classes with the Gaussian noise class (with respect to $d_{TV}^{\infty}$) as
    \begin{equation*}
        N_i=N_U\left(\frac{\epsilon}{4BT\sqrt{p_T}},\rv{\cG_{\sigma}}\circ\net{p_{i-1}}{p_i}{L_i},\infty,d_{TV}^{{\infty}},\rv{\cG_\sigma}\circ\rv{\cX_{B,p_{i-1}}}\right), \,\, 1\leq i\leq T,
    \end{equation*}
        and define the uniform covering number of class $\cQ$ as
    \begin{equation*}
        N_0 = N_U\left(\frac{\sigma\epsilon}{2B\sqrt{p_T}}, \cQ, m, \|.\|_2^{\ell_2}\right).
    \end{equation*}
    Then we have,
    \begin{equation*}
        N_U\left(\epsilon,\cH,m,\|.\|_2^{\ell_2}\right)\leq \prod_{i=0}^T N_i.
    \end{equation*}
\end{lemma}

\begin{proof}
From Corollary~\ref{coll:ell2_to_tv}, we can conclude that
\begin{equation*}
    N_U(\frac{\epsilon}{4B\sqrt{p_T}},\rv{\cG_{\sigma}}\circ \cQ,m,d_{TV}^{\ell_2},\rv{\Delta_d})\leq N_U(\frac{\sigma\epsilon}{2B\sqrt{p_T}},\cQ,m,\|.\|_2^{\ell_2})=N_0.
\end{equation*}
Same as proof of Theorem~\ref{thm:neuralnet}, by using Lemma~\ref{lemma:compose_tv}, we can say that for two consecutive classes $\net{p_i-1}{p_{i}}{L_{i}}$ and $\net{p_i}{p_{i+1}}{L_{i+1}}$
     \begin{equation*}
     \begin{aligned}
        & N_U\left(\frac{2\epsilon}{4BT\sqrt{p_T}},\rv{\cG_{\sigma}}\circ \net{p_i}{p_{i+1}}{L_{i+1}}\circ\rv{\cG_{\sigma}}\circ\net{p_{i-1}}{p_i}{L_i},\infty,d_{TV}^{\infty},\rv{\cG_\sigma}\circ\rv{\cX_{B,p_{i-1}}}\right)\\
        & \leq N_U\left(\frac{\epsilon}{4BT\sqrt{p_T}},\rv{\cG_{\sigma}}\circ\net{p_{i-1}}{p_i}{L_i},\infty,d_{TV}^{\infty},\rv{\cG_\sigma}\circ\rv{\cX_{B,p_{i-1}}}\right) \\
        &. N_U\left(\frac{\epsilon}{4BT\sqrt{p_T}},\rv{\cG_{\sigma}}\circ\net{p_i}{p_{i+1}}{L_{i+1}},\infty,d_{TV}^{\infty},\rv{\cG_\sigma}\circ\rv{\cX_{B,p_{i}}}\right)=N_i.N_{i+1}
     \end{aligned}
      \end{equation*}
Let
\begin{equation*}
        \rv{\cE}=\rv{\cG_{\sigma}}\circ\net{p_{T-1}}{p_T}{L_T}     \circ \ldots \circ \rv{\cG_{\sigma}}\circ\net{p_0}{p_1}{L_1}.
    \end{equation*}
It is clear that $\rv{\cF} = \rv{\cE} \circ \rv{\cG_{\sigma}}\circ \cQ.$
Now, from Lemma~\ref{lemma:compose_tv}, we can conclude that
\begin{equation*}
\begin{aligned}
  & N_U\left(\frac{\epsilon}{2B\sqrt{p_T}},\rv{\cF},m,d_{TV}^{\ell_2},\rv{\Delta_d}\right)\\
  & \leq N_U\left(\frac{\epsilon}{4B\sqrt{p_T}},\rv{\cE},\infty,d_{TV}^{\infty},\rv{\cG_\sigma}\circ\rv{\cX_{B,p}}\right). N_U(\frac{\epsilon}{4B\sqrt{p_T}},\rv{\cG_{\sigma}}\circ \cQ,m,d_{TV}^{\ell_2},\rv{\Delta_d})\\
  &\leq \prod_{i=0}^T N_i.
\end{aligned}
\end{equation*}
Lastly, from Theorem~\ref{thm:tv_to_ell2}, we can conclude that
\begin{equation*}
    N_U(\epsilon, \cH,m,\|.\|_2^{\ell_2})\leq N_U\left(\frac{\epsilon}{2B\sqrt{p_T}},\rv{\cF},\infty,d_{TV}^{\ell_2},\rv{\Delta_d}\right)\leq \prod_{i=0}^T N_i.
\end{equation*}
\end{proof}

\section{Uniform convergence by bounding the covering number}\label{app:dudley}
In this section we provide some technical backgrounds that are related to estimating NVAC and finding valid $GB$s. Specifically, we discuss how to turn a bound on $\|.\|_2^{\ell_2} $covering number to a bound on generalization gap with respect to ramp loss.

{\bf Preliminaries.} For any $x\in\bR$, the ramp function $r_{\gamma}$ with respect to a margin $\gamma$ is defined as
\begin{equation*}
    r_{\gamma}(x)=\begin{cases}
    0 & x\leq -\gamma,\\
    1 + \frac{x}{\gamma} & [-\gamma,0],\\
    1 &\gamma > 0.
    \end{cases}
\end{equation*}
Let $x=[x^{(1)},\ldots,x^{(k)}]\transpose\in\bR^k$ be a vector and $\cY=[k]$. The margin function $\cM:\bR^k \times \cY \rightarrow\bR$ is defined as $\cM(x,i):=x^{(i)} - \max_{j\neq i}x^{(j)}$. Next, we define the ramp loss for classification.

\begin{definition}[Ramp loss]
Let $f:\cX\rightarrow \bR^{k}$ be a function and let $\cD$ be a distribution over $\cX\times\cY$ where $\cY=[k]$. We define the ramp loss of function $f$ with respect to margin parameter $\gamma$ as $l_{\gamma}(f)=\expects{(x,y)\sim \cD}{r_{\gamma}(-\cM\left(f(x),y)\right)}$. We also define the empirical counterpart of ramp loss on an input set $S\sim\cD^m$ by $\hat{l}_{\gamma}(f)=\frac{1}{m}\sum_{(x,y)\in S}r_{\gamma}(-\cM(f(x),y))$.
\end{definition}
It is worth mentioning that using (surrogate) ramp loss is a natural case for classification tasks; see e.g., \citet{boucheron2005theory,bartlett2006convexity}.

Next, we define the composition of a hypothesis class with the ramp loss function.
\begin{definition}[Composition with ramp loss)]
Let $\cF$ be a hypothesis class from $\cX$ to $\bR^k$ and $\cY=[k]$. We denote the class of its composition with the ramp loss function by $\cF_{\gamma}:\cX \times \cY \rightarrow [0,1]$ and define it as $\cF_{\gamma}=\left\{(f_{\gamma}(x,y)=r_{\gamma}\left(-\cM(f(x),y)\right): f\in\cF\right\}$. 
\end{definition}
The following lemma states that we can always bound the 0-1 loss by the ramp loss.
\begin{lemma}\label{lemma:zero_one}
Let $\cD$ be a distribution over $\cX\times\cY$, where $\cY=[k]$ and let $f$ be a function from $\cX$ to $\bR^{k}$. We have
\begin{equation*}
    \expects{(x,y)\sim \cD}{l^{0-1}(f(x),y)} \leq \expects{(x,y)\sim \cD}{r_{\gamma}(-\cM(f(x),y))} = l_{\gamma}(f).
\end{equation*}
\end{lemma}
For a proof of Lemma~\ref{lemma:zero_one} see section A.2 in \citet{bartlett2017spectrally}.

One way to bound the generalization gap of a learning algorithm is to find the rate of uniform convergence for class $\cF_{\gamma}$. We define uniform convergence in the following.
\begin{definition}[Uniform convergence]
Let $\cF$ be a hypothesis class and $l$ be a loss function. We say that $\cF$ has uniform convergence property if there exists some function $m_{UC}:(0,1)^2\rightarrow \bN$ such that for every distribution $\cD$ over $\cX\times\cY$ and any sample $S\sim\cD^m$ if $m\geq m_{UC}(\epsilon,\delta)$ with probability at least $1-\delta$ (over the randomness of $S$) for every hypothesis $f\in\cF$ we have
\begin{equation*}
    \left|\expects{(x,y)\sim \cD}{l(f(x),y)} -\frac{1}{m}\sum_{(x,y)\in S} l(f(x),y) \right|\leq \epsilon.
\end{equation*}
\end{definition}
An standard approach for finding the rate of uniform convergence is by analyzing the Rademacher complexity of $\cF_{\gamma}$. We now define the empirical Rademacher complexity.
\begin{definition}[Empirical Rademacher complexity]
Let $\cF$ be a class of hypotheses from $\cZ$ to $\bR$ and $\cD$ be a distribution over $\cZ$. The empirical Rademacher complexity of class $\cF$ with respect to sample $S=\left\{z_1,\ldots,z_m\right\}\sim \cD^m$ is denoted by $\hat{\fR}(\cF_{|S})$ and is defined as
\begin{equation*}
    \hat{\fR}(\cF_{|S}) = \expects{\sigma}{\sup_{f\in\cF} \sum_{i=1}^m\sigma_i f(z_i)}
\end{equation*}
where $\sigma=(\sigma_1,\ldots,\sigma_m)$ and $\sigma_i$ are i.i.d. Rademacher random variables uniformly drawn from $\{0,1\}$. 
\end{definition}

The following theorem relates the Rademacher complexity of $\cF_{\gamma}$ to its rate of uniform convergence and provides a generalization bound for the ramp loss and its empirical counterpart on a sample $S$.
\begin{theorem}\label{thm:rademacher}
Let $\cF$ be a class of functions from $\cX$ to $\bR^k$ and $\cD$ be a distribution over $\cX\times\cY$ where $\cY=[k]$. Let $S\sim\cD^m$ denote a sample. Then, for every $\delta$ and every $f\in\cF$, with probability at least $1-\delta$ (over the randomness of $S$) we have
\begin{equation*}
    l_{\gamma}(f) \leq \hat{l}_{\gamma}(f) + 2\hat{\fR}({\cF_{\gamma}}_{|S}) + 3 \sqrt{\frac{\ln(2/\delta)}{2m}}
\end{equation*}
\end{theorem}
Theorem~\ref{thm:rademacher} is an immediate result of standard generalization bounds based on Rademacher complexity (see e.g. Theorem 3.3 in \citet{mohri2018foundations}) once we realize that $\expects{(x,y)\sim \cD}{f_{\gamma}}=l_{\gamma}(f)$ and $\frac{1}{m}\sum_{(x,y)\in S}f_{\gamma}(x,y)=\hat{l}_{\gamma}(f)$.

We will use Dudley entropy integral \citep{dudley2010universal} for chaining to bound the Rademacher complexity by covering number; see \citet{shalev2014understanding} for a proof.
\begin{theorem}[Dudley entropy integral]\label{thm:dudely}
Let $\cF$ be a class of hypotheses with bounded output in $[0,c_x]$. Then
\begin{equation*}
    \hat{\fR}(\cF_{|S}) \leq \inf_{\epsilon\in [0,c_x/2]}\left\{4\epsilon + \frac{12}{\sqrt{m}}\int_{\epsilon}^{c_x/2}\sqrt{\ln N_U(\nu,\cF,m,\|.\|_2^{\ell_2})}\,d\nu\right\}.
\end{equation*}
\end{theorem}
Putting Theorems~\ref{thm:rademacher}, \ref{thm:dudely}, and Lemma~\ref{lemma:zero_one} together, we are now ready to state the following theorem to bound the 0-1 loss based on the covering number of $\cF_{\gamma}$ and empirical ramp loss.
\begin{theorem}\label{thm:gen_cover}
Let $\cF$ be a class of functions from $\cX$ to $\bR^k$ and $\cD$ be a distribution over $\cX\times \cY$ where $\cY=[k]$. Let $S\sim\cD^m$ be a sample. Then, with probability at least $1-\delta$ (over the randomness of $S$) for every $f\in\cF$ we have
\begin{equation*}
\begin{aligned}
        &\expects{(x,y)\sim \cD}{l^{0-1}(f(x),y)}\leq \\
       &l_{\gamma}(f) \leq \hat{l}_{\gamma}(f) +\inf_{\epsilon\in [0,1/2]}\left\{2\left[4\epsilon + \frac{12}{\sqrt{m}}\int_{\epsilon}^{1/2}\sqrt{\ln N_U(\nu,\cF_{\gamma},m,\|.\|_2^{\ell_2})}\,d\nu\right]\right\}+ 3 \sqrt{\frac{\ln(2/\delta)}{2m}}.
\end{aligned}
\end{equation*}
\end{theorem}
We will use above theorem in the next appendix to estimate NVAC based on $\|.\|_2^{\ell_2}$ covering number of composition of a class with ramp loss.
\section{Estimating NVAC using the covering number}\label{app:nvacfromCover}
In this appendix, we will use Theorem~\ref{thm:gen_cover} to establish a way of approximating NVAC from a covering number bound. In Remark~\ref{rem:nvac} we state the technique used to approximate NVAC and in the following we will justify why this would be a good approximation.
\begin{remark}\label{rem:nvac}
Let $\cF$ be a hypothesis class from $\cX$ to $\bR^{k}$, $S$ be a sample of size $m$ and $\hat{h}\in\cF$. We find $n^*$ such that the following holds
\begin{equation}\label{eq:e0}
     \frac{6}{\sqrt{mn^*}}\sqrt{\ln N_U(\epsilon, \cF_{\gamma}, mn^*, \|.\|_2^{\ell_2})} \leq \epsilon, \quad \epsilon = \frac{1 - \hat{l}_{\gamma}(\hat{h})}{10},
\end{equation}
and choose $mn^*$ as an approximation of NVAC. Here, $\hat{l}_{\gamma}(\hat{h})$ is the empirical ramp loss of $\hat{h}$ on sample $S$. In Appendix~\ref{app:nvacgraphs}, where we empirically compare NVAC of different covering number bounds, we choose $S$ to be the MNIST dataset and $\hat{h}$ as the trained neural network (from a class $\cF$ of all neural networks with a certain architecture) on this dataset.
\end{remark}
In the following we discuss why this choice of $mn^*$ is a good estimate of NVAC. First, let $S^n \in (\cX \times \cY)^{mn}$ be an input set and $\cD$ be a distribution over $(\cX\times\cY)$, where $mn$ is larger than $mn^*$ as found in Remark~\ref{rem:nvac}. From Theorem~\ref{thm:gen_cover} and using the fact that the ramp loss is in $[0,1]$ we can write
\begin{equation}\label{eq:e1}
\begin{aligned}
     &\expects{(x,y)\sim \cD}{l^{0-1}(\hat{h}(x),y)}\leq \hat{l}_{\gamma}(\hat{h}) + 2\fR({\cF_{\gamma}}_{|S^n}) + 3 \sqrt{\frac{\ln(2/\delta)}{2mn}}\\
     & \leq \hat{l}_{\gamma}(\hat{h}) +  \inf_{\epsilon\in [0,1/2]}\left\{2\left[4\epsilon + \frac{12}{\sqrt{mn}}\int_{\epsilon}^{1/2}\sqrt{\ln N_U(\nu,\cF_{\gamma},mn,\|.\|_2^{\ell_2})}\,d\nu\right]\right\} + 3 \sqrt{\frac{\ln(2/\delta)}{2mn}}.
\end{aligned}
\end{equation}
Since $S^n$ consists of $n$ copies of the sample $S$, we can replace $\hat{l}_{\gamma}(\hat{h})$ on $S^n$ by the ramp loss of $\hat{h}$ on $S$ (this would be equal to the ramp loss of trained neural network when we empirically compare NVACs in Appendix~\ref{app:nvacgraphs}). Moreover, since the number of samples are very large and $\delta=0.01$, we can approximate the last term in the right hand side of Equation~\ref{eq:e1} with zero. Therefore, we can write that
\begin{equation}\label{eq:e2}
\begin{aligned}
     &\expects{(x,y)\sim \cD}{l^{0-1}(\hat{h}(x),y)}\\
     &\leq\ \hat{l}_{\gamma}(\hat{h}) +  \inf_{\epsilon\in [0,1/2]}\left\{2\left[4\epsilon + \frac{12}{\sqrt{mn}}\int_{\epsilon}^{1/2}\sqrt{\ln N_U(\nu,\cF_{\gamma},mn,\|.\|_2^{\ell_2})}\,d\nu\right]\right\}\\
     &\leq \hat{l}_{\gamma}(\hat{h}) +  2\left[4\epsilon + \frac{12}{\sqrt{mn^*}}\int_{\epsilon}^{1/2}\sqrt{\ln N_U(\nu,\cF_{\gamma},mn^*,\|.\|_2^{\ell_2})}\,d\nu\right] && (\forall \epsilon \in [0,1/2])\\
     &\leq\ \hat{l}_{\gamma}(\hat{h}) +  2\left[4\epsilon + \frac{6}{\sqrt{mn^*}}\sqrt{\ln N_U(\epsilon,\cF_{\gamma},mn^*,\|.\|_2^{\ell_2})}\right],
\end{aligned}
\end{equation}
where we used the fact that $N_U(\epsilon,\cF_{\gamma},mn,\|.\|_2^{\ell_2})$ decreases monotonically with $\epsilon$ and the integral is over $[\epsilon,1/2]$. Note that in the above equation we subtly used the fact that covering number grows at most polynomially with the number of samples and, therefore, increasing number of samples will always result in smaller right hand side term in Equation~\ref{eq:e2}. In Appendix~\ref{app:coverapproaches}, we will show why this is a valid assumption for the covering number bounds that we use in our experiments (see Remark~\ref{rem:sample_grow}). 

Since Equation~\ref{eq:e2} holds for any $\epsilon \in [0,1/2]$, we can set $\epsilon = (1 - \hat{l}_{\gamma}(\hat{h}))/10$ and conclude that
\begin{equation}\label{eq:e3}
\begin{aligned}
     &\expects{(x,y)\sim \cD}{l^{0-1}(\hat{h}(x),y)}\\
     &\leq\ \hat{l}_{\gamma}(\hat{h}) +  2\left[4\epsilon + \frac{6}{\sqrt{mn^*}}\sqrt{\ln N_U(\epsilon,\cF_{\gamma},mn^*,\|.\|_2^{\ell_2})}\right]\\
     & \leq \hat{l}_{\gamma}(\hat{h}) + 2\,\frac{5(1-\hat{l}_{\gamma}(\hat{h})) }{10}\\
     &\leq 1.
\end{aligned}
\end{equation}
From the above equation, we can conclude that by setting $mn$ to be larger that $mn^*$ as defined in Remark~\ref{rem:nvac}, we can provide the following valid generalization bound with respect to $l^{0-1}$ and $l_{\gamma}$:
\begin{equation*}
    GB(\hat{h},S^n) = 2\left[4\epsilon + \frac{6}{\sqrt{mn }}\sqrt{\ln N_U(\epsilon,\cH,mn,\|.\|_2^{\ell_2})}\right].
\end{equation*}
Moreover, for any $S^n$ such that $mn\geq mn^*$ we can conclude that the $GB$ defined above results in a non-vacuous bound, i.e.,
\begin{equation*}
    GB(\hat{h},S^{n}) + \hat{l}_{\gamma}(\hat{h}) \leq 1,
\end{equation*}
which concludes that $mn^*$ is a reasonable approximation for NVAC.

In the next appendix, we discuss different covering number bounds that were mentioned in Section~\ref{subsec:cover_bounds}. We state the exact form of these covering number bounds for a general $T$-layer network in Appendix~\ref{app:coverapproaches}. Finally in Appendix~\ref{app:nvacgraphs}, we present the settings of our experiments and the empirical results of NVAC found by Remark~\ref{rem:nvac}.
\section{Different approaches to bound the covering number}\label{app:coverapproaches}
In the following, we will state the covering number bounds that were compared in Sections~\ref{subsec:cover_bounds} and \ref{sec:experiments}. We first give two preliminary lemmas. Lemma~\ref{lemma:cover_to_margin} connects the covering number of a hypothesis class $\cF$ to the covering number of $\cF_{\gamma}$, which is used in Remark~\ref{rem:nvac} to obtain NVAC and generalization bounds. In Lemma~\ref{lemma:real_high} we will show a way to find the covering number of a class of functions from $\bR^d$ to $\bR^p$ from the covering number of real-valued classes that correspond to each dimension. We will use this lemma when we want to compare covering number bounds in the literature that are given for real-valued functions, i.e., Norm-based, Lipschitzness-based, and Pseudo-dim-based bounds.

In the following remark, we will discuss the motivation behind the choice of specific generalization bounds in
Section~\ref{sec:experiments}.
\begin{remark}[Choice of generalization bounds]
In our experiments in Section~\ref{sec:experiments} we have not assessed the PAC-Bayes bound in \citet{neyshabur2017pac} since it is always looser than the Spectral bound of \citep{bartlett2017spectrally}; see \citet{neyshabur2017pac} for a discussion. Furthermore, we exclude the generalization bounds that are proved in ``two steps''. For example, a naive two-step approach is to divide the training data into a large and a small subsets; one can then train the network using the large set and evaluate the resulting hypothesis using the small set. This will give a rather tight generalization bound since in the second step we are evaluating a single hypothesis. However, it does not explain why the learning worked well (i.e., how the learning model came up with a good hypothesis in the first step). 
More sophisticated two-step approaches such as \citet{DR17,arora2018stronger,zhou2019non} offer more insights on why the model generalizes. However, they do not fully explain why the first step works well (i.e., the prior distribution in \citet{DR17} or the uncompressed network in \cite{arora2018stronger,zhou2019non}. Therefore, we focus on bounds based on covering numbers (uniform convergence).
\end{remark}
Next, we state the preliminaries lemmas that we use for some of the covering number bounds in literature to relate them to covering numbers for the composition of neural networks with the ramp loss.
\begin{lemma}[From covering number of $\cF$ to covering number of $\cF_{\gamma}$]\label{lemma:cover_to_margin}
Let $\cF$ be a hypothesis class of functions from $\cX$ to $\bR^k$ and $\cF_{\gamma}:\cX \times \cY \rightarrow [0,1]$ be the class of its composition with ramp loss, where $\cY=[k]$. Then we have
\begin{equation*}
    N_U(\epsilon,\cF_{\gamma},m,\|.\|_2^{\ell_2}) \leq N_U(\frac{\gamma\epsilon}{2},\cF,m,\|.\|_2^{\ell_2}). 
\end{equation*}
\end{lemma}
\begin{proof}
First, it is easy to verify that $r_{\gamma}$ and $-\cM(x,y)$ (with respect to the first input) are Lipschitz continuous functions with respect to $\|.\|_2$ with Lipschitz factors of $1/\gamma$ and $2$, respectively; see e.g., section A.2 in \citet{bartlett2017spectrally}. Therefore, we can conclude that $r_{\gamma}\left(-\cM(f(x),y)\right)$ is Lipschitz continuous with Lipschitz factor of $2/\gamma$. 

Fix an input set $S=\{(x_1,y_1),\ldots,(x_m,y_m)\}\subset \cX \times \cY$ and let $C=\{\hat{f_i}_{|S}\mid \hat{f_i}\in\cF,\, i\in[r]\}$ be an $(\gamma\epsilon/2)$-cover for $\cF_{|S}$. In the following, we will denote the composition of $\hat{f_i}$ with ramp loss by $\hat{f}_{{\gamma,i}}$ for the simplicity of notation. Now, we prove that $C_{\gamma}=\{\hat{f}_{\gamma,i_{|S}}\mid \hat{f}_{{\gamma,i}}\in\cF_{\gamma},\, i\in[r]\}$ is also an $\epsilon$-cover for ${\cF_\gamma}_{|S}$.

Given any $f\in\cF$, there exists $\hat{f_i}_{|S}\in C$ such that
\begin{equation*}
    \left\| (\hat{f}_i(x_1),\ldots,\hat{f}_i(x_m))-(f(x_1),\ldots,f(x_m))\right\|_2^{\ell_2}\leq \frac{\gamma\epsilon}{2}.
\end{equation*}
We can then write that
\begin{equation}\label{eq:0}
    \begin{aligned}
          &\left\| (\hat{f}_{{\gamma,i}}(x_1),\ldots,\hat{f}_{{\gamma,i}}(x_m))-(f_{\gamma}(x_1),\ldots,f_{\gamma}(x_m))\right\|_2^{\ell_2}\\
      & = \sqrt{\frac{1}{m} \sum_{k=1}^{m} \left( \hat{f}_{{\gamma,i}}(x_k) - (f_{\gamma}(x_k) \right)^2}\\
      &\leq  \sqrt{\frac{1}{m} \sum_{k=1}^{m} \left( r_{\gamma}\left(-\cM(\hat{f}_i(x_k),y_k)\right) - r_{\gamma}(-\cM(f(x_k),y_k))\right)^2}
    \end{aligned}
\end{equation}
From the Lipschitz continuity of $r_{\gamma}\left(-\cM(x,y)\right)$ we can conclude that for any $(x,y)\in\cX \times \cY$
\begin{equation*}
    \begin{aligned}
     \left|r_{\gamma}\left(-\cM(f(x),y)\right) - r_{\gamma}(-\cM(\hat{f}_i(x),y))\right|\leq \frac{1}{\gamma} \|\cM(\hat{f}_i(x),y) - \cM(f(x),y)\|_2 \leq  \frac{2}{\gamma}\|\hat{f}_i(x) - f(x)\|_2 
    \end{aligned}.
\end{equation*}
Taking the above equation into account, we can rewrite Equation~\ref{eq:0} as
\begin{equation*}
    \begin{aligned}
      &\left\| (\hat{f}_{{\gamma,i}}(x_1),\ldots,\hat{f}_{{\gamma,i}}(x_m))-(f_{\gamma}(x_1),\ldots,f_{\gamma}(x_m))\right\|_2^{\ell_2}\\
      &\leq \frac{2}{\gamma} \sqrt{\frac{1}{m} \sum_{k=1}^{m} \left( (\hat{f}_i(x_k) - f(x_k))\right)^2}\\
      & \leq   \frac{2}{\gamma} \left\| (\hat{f}_i(x_1),\ldots,\hat{f}_i(x_m))-(f(x_1),\ldots,f(x_m))\right\|_2^{\ell_2}\\
      &\leq \frac{2}{\gamma}\frac{\gamma\epsilon}{2}\\
      &\leq \epsilon.
    \end{aligned}
\end{equation*}
In other words, for any $f_{\gamma_{|S}}\in{\cF_\gamma}_{|S}$ there exists $\hat{f}_{\gamma,i_{|S}}\in S$ such that $\left \|\hat{f}_{\gamma,i_{|S}} - f_{\gamma_{|S}}\right\|_2^{\ell_2}\leq \epsilon$ and, therefore, $C_{\gamma}$ is an $\epsilon$-cover for ${\cF_\gamma}_{|S}$ and the result follows.
\end{proof}
The following lemma finds a covering number for a class of functions with outputs in $\bR^p$ from the covering number of the classes of real-valued functions corresponding to each dimension.  
\begin{lemma}\label{lemma:real_high}
Let $\cF_{1},\ldots,\cF_{p}:\cX \rightarrow \bR$ be $p$ classes of real valued functions. Further let $\cF=\left \{f(x)=[f_1(x),\ldots,f_p(x)]\transpose \mid f_i\in\cF_i,\,i\in [p]\right \}$ be a class of functions from $\cX$ to $\bR^p$, where each dimension $i$ in their output comes from the output of a real-valued function in $\cF_i$. Then, we have
\begin{equation*}
    N_U(\epsilon,\cF,m,\|.\|_2^{\ell_2}) \leq \prod_{i=1}^p N_U(\frac{\epsilon}{\sqrt{p}},\cF_i,m,\|.\|_2^{\ell_2}).
\end{equation*}
\end{lemma}
\begin{proof}
Fix an input set $S=\{x_1,\ldots,x_m\}\subset \cX$. Let $C_1,\ldots,C_p$ be $(\epsilon/\sqrt{p})$-covers for ${\cF_1}_{|S},\ldots,{\cF_p}_{|S}$, respectively. We will construct the set $C$ as follows and prove that $C$ is an $\epsilon$-cover for $\cF_{|S}$
\begin{equation*}
    C=\left\{ [\hat{f}_1(x_k),\ldots,\hat{f_p}(x_k)]\transpose \mid \hat{f_i}_{|S}\in C_i, \, i \in [p],\, k \in [m]\right\}.
\end{equation*}
Particularly, from each class $\cF_i$, we are choosing all functions $\hat{f_i}$ such that $\hat{f_i}_{|S}$ is in $C_i$. We then use those functions as the dimension $i$ of the output to get functions $f\in\cF$. Then we put the restriction of these functions to set $S$ in $C$. Clearly, $|C|\leq \prod_{i=1}^p |C_i|$.

Let $f(x)=[f_1(x),\ldots,f_p(x)]\transpose$ be any function in $\cF$. Since $C_1,\ldots,C_p$ are $(\epsilon/\sqrt{p})$-covers for $\cF_1,\ldots,\cF_p$ we know that there exists another set of functions $\hat{f_i}\in\cF_i,\,i\in[p]$ such that $\hat{f_i}_{|S}\in C_i$ and
\begin{equation*}
    \left\| (\hat{f_i}(x_1),\ldots,\hat{f_i}(x_m)) - (f_i(x_1),\ldots,f_i(x_m))\right\|_2^{\ell_2}\leq \frac{\epsilon}{\sqrt{p}}, \quad \forall i\in [p].
\end{equation*}
Let $\hat{f}(x)=[\hat{f_1}(x),\ldots,\hat{f_p}(x)]\transpose$. We can then write that
\begin{equation*}
    \begin{aligned}
     \left \| f_{|S} - \hat{f}_{|S}\right\|_2^{\ell_2}= &\left \| (f(x_1),\ldots,f(x_m)) - (\hat{f}(x_1),\ldots,\hat{f}(x_m))\right\|_2^{\ell_2}\\
     & = \sqrt{\frac{1}{m}\sum_{k=1}^m \left\|f(x_k)-\hat{f}(x_k) \right\|_2^2}\\
     & \leq \sqrt{\frac{1}{m}\sum_{k=1}^m \sum_{i=1}^p\left( f_i(x_k) - \hat{f_i}(x_k)\right)^2}\\
     & \leq \sqrt{\sum_{i=1}^p \sum_{k=1}^m\frac{1}{m}\left( f_i(x_k) - \hat{f_i}(x_k)\right)^2}\\
     & \leq \sqrt{\sum_{i=1}^p \left( \left\|(f_i(x_1),\ldots,f_i(x_m))-(\hat{f_i}(x_1),\ldots\hat{f_i}(x_m))\right\|_2^{\ell_2}\right)^2}\\
      &\leq \sqrt{\sum_{i=1}^p\frac{\epsilon^2}{p}}\\
      &\leq \epsilon
    \end{aligned}
\end{equation*}
Therefore, we can conclude that $C$ is an $\epsilon$-cover for $\cF_{|S}$. Since $|C|\leq \prod_{i=1}^p |C_i|$ the result follows.
\end{proof}
In the following we will state the covering number bounds that are compared using their NVACs and generalization bounds in Section~\ref{sec:experiments}. 

{\bf Covering number bounds.} We first state the bound in Theorem~\ref{thm:neuralnet}, where we use the covering number of Theorem~\ref{lemma:net} for every but the first layer of the neural network and Lemma 14.7 in \cite{anthony1999neural} for the first layer. This theorem is almost the same as Corollary~\ref{thm:ours_main} in Section~\ref{sec:applications}. It only has one more step on relating the covering number of the class $\cH$ to $\cH_{\gamma}$.
\begin{theorem}[Covering number bound of Theorem~\ref{thm:neuralnet} for ramp loss]\label{thm:ours}
 Let $\net{d}{p_1}{L_1}$,$\net{p_1}{p_2}{L_2}$,$\ldots$, $\net{p_{T-1}}{p_T}{L_T}$ be $T$ classes of neural networks. Denote the $T$-layer noisy network by
    \begin{equation*}
        \rv{\cF}=\rv{\cG_{\sigma}}\circ\net{p_{T-1}}{p_T}{L_T}     \circ \ldots \circ \rv{\cG_{\sigma}}\circ\net{p_1}{p_2}{L_2}     \circ    \rv{\cG_{\sigma}}\circ\net{d}{p_1}{L_1},
    \end{equation*}
    and let $\cH=\{h:\bR^d\to [0,1]^{p_T} \mid h(x)=\expects{\rv{f}}{~\rv{f}({x})}, \rv{f}\in \rv{\cF} \}$.
    Then we have
    \begin{equation*}
    \begin{aligned}
            &\ln N_U\left(\epsilon,\cH_{\gamma},m,\|.\|_2^{\ell_2}\right)\\
            &\leq \sum_{i=2}^T p_i.p_{i-1}\ln\left(30\frac{(2T\sqrt{p_T})^{3/2}p_{i-1}^{5/2}\sqrt{\ln\left(\frac{\displaystyle (10/\gamma)T\sqrt{p_T}p_{i-1}-\epsilon\sigma}{\displaystyle \epsilon\sigma} \right)}}{(\gamma\epsilon)^{3/2}\sigma^2}\ln\left(\frac{10Tp_{i-1}\sqrt{p_T}}{\gamma\epsilon\sigma}\right)\right)\\
            & + dp_1\ln\left(\frac{Tem\sqrt{p_T}}{\gamma\epsilon\sigma}\right).
    \end{aligned}
    \end{equation*}
\end{theorem}
\begin{proof}
We first use Theorem~\ref{lemma:net} to find the covering number of $\net{p_{i-1}}{p_i}{}$. Particularly, for any $2\leq i\leq T$ we have,
\begin{equation*}
    \begin{aligned}
     & \ln N_i=\ln N_U\left(\frac{\epsilon}{T\sqrt{p_T}},\rv{\cG_{\sigma}}\circ\net{p_{i-1}}{p_i}{L_i},\infty,d_{TV}^{\infty},\rv{\cG_\sigma}\circ\rv{\cX_{1,p_{i-1}}}\right)\\
     & \leq p_i.p_{i-1}\ln\left(30\frac{(T\sqrt{p_T})^{3/2}p_{i-1}^{5/2}\sqrt{\ln\left(\frac{\displaystyle 5T\sqrt{p_T}p_{i-1}-\epsilon\sigma}{\displaystyle \epsilon\sigma} \right)}}{\epsilon^{3/2}\sigma^2}\ln\left(\frac{5Tp_{i-1}\sqrt{p_T}}{\epsilon\sigma}\right)\right).
    \end{aligned}
\end{equation*}
Moreover, we use Lemma 14.17 in \citet{anthony1999neural} to find a bound on $N_1$. This lemma provides a bound with respect to $\|.\|_2^{\infty}$, however, we know that $\|.\|_2^{\ell_2}$ is always smaller than $\|.\|_2^{\infty}$ (see Remark~\ref{rem:one_infty}). Therefore, we can bound $N_1$ as follows
\begin{equation*}
    \ln N_1 \leq dp_1\ln\left(\frac{Tem\sqrt{p_T}}{2\epsilon\sigma}\right).
\end{equation*}
From Theorem~\ref{thm:neuralnet} we know that $\ln N_U\left(\epsilon,\cH,m,\|.\|_2^{\ell_2}\right)\leq \sum_{i=1}^T \ln N_i$, therefore, we can write that
\begin{equation*}
    \begin{aligned}
     &\ln N_U\left(\epsilon,\cH,m,\|.\|_2^{\ell_2}\right)\\
     & \leq \sum_{i=1}^T p_i.p_{i-1}\ln\left(30\frac{(T\sqrt{p_T})^{3/2}p_{i-1}^{5/2}\sqrt{\ln\left(\frac{\displaystyle 5T\sqrt{p_T}p_{i-1}-\epsilon\sigma}{\displaystyle \epsilon\sigma} \right)}}{\epsilon^{3/2}\sigma^2}\ln\left(\frac{5Tp_{i-1}\sqrt{p_T}}{\epsilon\sigma}\right)\right)\\
     & + dp_1\ln\left(\frac{Tem\sqrt{p_T}}{2\epsilon\sigma}\right).
    \end{aligned}
\end{equation*}
Applying Lemma~\ref{lemma:cover_to_margin} to turn this covering number into a covering number for $\cH_{\gamma}$ concludes the result.
\end{proof}
{\bf Notation.} For a matrix $W\in\bR^{d\times p}$ we denote its $\|.\|_{s,t}$ norm as $\left\|\left( \|W_{:,1}\|_s,\ldots,\|W_{:,p}\|_s \right)\right\|_{t}$, where $W_{:,i}$ denotes the $i$th column of $W$ (e.g. for a weight matrix $W$, $\|W\transpose\|_{1,\infty}$ refers to the maximum of $\|.\|_1$ norm of incoming weights of a neuron). By $\|W\|_{\sigma}$ we denote the spectral norm of a matrix. For a matrix $X\in\bR^{d\times m}$ we denote its normalized Frobenious norm by$\|X\|_F$, which is defined as $\|X\|_F=\sqrt{\frac{1}{m}\sum x_{i,j}^2}$.

We would like to mention that, in the experiments, we use a slightly different form of sigmoid function for the activation function rather than the one in Definition~\ref{def:neuralnet}. Indeed, we will add a constant to the sigmoid function to turn it into an odd function in $[-1/2,1/2]$. In the following remark we will discuss the reason behind this choice and the fact that it does not change the covering number in Theorem~\ref{thm:ours}. 
\begin{remark}
The bound in the Spectral covering number requires the activation functions to output 0 at the origin. Therefore, in our experiments in Section~\ref{sec:experiments}, we set $\phi(x)= \frac{1}{1+e^{-x}}-\frac{1}{2}$ as activation functions for neurons of the network, so that $\phi(0)=0$ and $\phi(x)\in [-1/2,1/2]$. This will not affect the covering number bound of Theorem~\ref{thm:ours}. The bound in Theorem~\ref{thm:ours} is derived from the covering number bound of Theorem~\ref{lemma:net} for single-layer neural network classes. There are three sources of dependency on the activation function in Theorem~\ref{lemma:net}. The first one is the dependence on the range of output, which is 1 for both $\phi(x)=\frac{1}{1+e^{-x}}-\frac{1}{2}$ and the sigmoid function ($\phi(x)=\frac{1}{1+e^{-x}}$ defined in Definition~\ref{def:neuralnet}. The second dependecy is the Lipschitz factor which is 1 for both of the activation functions. The final dependency is on $u=\max \left\{\left |\phi^{-1}(B-\epsilon)\right|, \left |\phi^{-1}(-B+\epsilon)\right|\right\}$. It is easy to verify that the value of $u$ for $\phi(x)=\frac{1}{1+e^{-x}}-\frac{1}{2}$ is exactly the same as the value of $u$ for $\phi(x)=\frac{1}{1+e^{-x}}$. As a result, using both $\phi(x)=\frac{1}{1+e^{-x}}$ and $\phi(x)=\frac{1}{1+e^{-x}}-\frac{1}{2}$ will result in the same covering number bound in Theorem~\ref{thm:ours}. Generally, adding a constant to the output of functions in a class will not change its covering number.
\end{remark} 

We will now discuss the Norm-based bound from Theorem 14.17 in \citet{anthony1999neural}, which is a bound for real-valued networks. Therefore, we will apply Lemma~\ref{lemma:real_high} to relate it to a covering number for neural networks with $p$ output dimensions.
\begin{theorem}[Norm-based covering number]\label{thm:norm_based}
Let $\wnet{d}{p}{v}=\{f_W:\bR^d\to[0,1]^p \mid f_W(x)=\Phi(W\transpose x), W\in\bR^{d\times p}  \text{ and } \|W\transpose\|_{1,\infty}\leq v\}$ be the class of single-layer neural networks with $d$ inputs and $p$ outputs where $\|.\|_{1,\infty}$ norm of the layer is bounded by $v$. Let $\wnet{d}{p_1}{v_1},\ldots,\wnet{p_{T-1}}{p_T}{v_T}$ be $T$ classes of neural networks and denote the $T$-layer neural network by $\cF = \wnet{p_{T-1}}{p_T}{v_T}\circ\ldots\circ \wnet{d}{p_1}{v_1}$. Denote by $V$ the maximum of  $\|.\|_{1,\infty}$ among the layers of the network, i.e., $V=\max_{i} v_i$. Then we have
\begin{equation*}
    \log_2 N_U(\epsilon,\cF_{\gamma},m,\|.\|_2^{\ell_2}) \leq \frac{p_T}{2}(\frac{2\sqrt{p_T}}{\gamma\epsilon})^{2T}(2V)^{T(T+1)}\log_2(2d+2).
\end{equation*}
\end{theorem}
\begin{proof}
The proof simply follows from Theorem 14.17 in \citet{anthony1999neural} and Lemmas~\ref{lemma:cover_to_margin} and \ref{lemma:real_high} once we note that the sigmoid function is Lipschitz continuous with Lipschitz factor of 1.
\end{proof}
Next we state the Pseudo-dim-based bound. 

\begin{theorem}[Psuedo-dim-based covering number]\label{thm:pdim_based}
Let $\net{d}{p_1}{W_1},\ldots,\net{p_{T-1}}{p_T}{W_T}$ be $T$ classes of neural networks and $\cF = \net{p_{T-1}}{p_T}{W_T}\circ\ldots\circ \net{d}{p_1}{W_1}$. Denote by $\cF_i$ the class of real-valued functions corresponding to $i$-th dimension of output of functions in class $\cF$. Denote the total number of weights of the real-valued network $\cF_i$ by $W_{rvo}=dp_1+\sum_{i=2}^{T-1} p_{i-1}.p_i + p_{T-1}$ and the total number of neurons in all but the input layer of the real-valued network $\cF_i$ by $r_{rvo}=1+\sum_{i=1}^{T-1} p_i$. Furthermore, let $P$ be as follows
\begin{equation*}
    P = \left((W_{rvo}+2)r_{rvo}\right)^2 + 11(W_{rvo}+2)r_{rvo}\log_2\left(18(W_{rvo}+2)r_{rvo}^2\right).
\end{equation*} 
Then given that $m>P$ we have
\begin{equation*}
    \ln N_U(\epsilon,\cF_{\gamma},m,\|.\|_2^{\ell_2}) \leq p_TP\ln\left(\frac{2\sqrt{p_T}em}{P\gamma\epsilon}\right).
\end{equation*}
\end{theorem}
\begin{proof}
By Theorem 14.2 in \citet{anthony1999neural} we know that the pseudo dimension, $P_{\text{dim}}$, of $\cF_i$ is smaller or equal to $P$ (for a definition of pseudo dimension see for instance Chapter 11 in \citet{anthony1999neural}). Furthermore, from the standard analysis of covering number and pseudo dimension (see e.g., Theorem 12.2 in \citet{anthony1999neural}), we can write
\begin{equation*}
     \ln N_U(\epsilon,\cF_i,m,\|.\|_2^{\ell_2}) \leq P_{\text{dim}} \ln (\frac{em}{\epsilon P_{\text{dim}}}).
\end{equation*}
Combining the above equation with Lemmas~\ref{lemma:cover_to_margin} and \ref{lemma:real_high} concludes the result.
\end{proof}
Now we turn into presenting the Lipschitzness-based bound.

\begin{theorem}[Lipschitzness-based covering number]\label{thm:lipsch-based}
Let $\wnet{d}{p}{v}=\{f_W:\bR^d\to[0,1]^p \mid f_W(x)=\Phi(W\transpose x), W\in\bR^{d\times p} \text{ and } \|W\transpose\|_{1,\infty}\leq v\}$ be the class of single-layer neural networks with $d$ inputs and $p$ outputs where $\|.\|_{1,\infty}$ norm of the layer is bounded by $v$. Let $\wnet{d}{p_1}{v_1}$,$\ldots$, $\wnet{p_{T-1}}{p_T}{v_T}$ be $T$ classes of neural networks and denote the $T$-layer neural network by $\cF = \wnet{p_{T-1}}{p_T}{v_T}\circ\ldots\circ \wnet{d}{p_1}{v_1}$. Denote by $\cF_i$ the class of real-valued functions corresponding to $i$-th dimension of output of functions in class $\cF$. Let $V$ the maximum of $\|.\|_{1,\infty}$ among all but the first layer of the network, i.e., $V=\max_{2 \leq i \leq T} v_i$ and denote the total number of weights of the real-valued networks by $W_{rvo}=dp_1+\sum_{i=2}^{T-1} p_{i-1}.p_i+p_{T-1}$. Then we have
\begin{equation*}
    \ln N_U(\epsilon,\cF_{\gamma},m,\|.\|_2^{\ell_2}) \leq p_T W_{rvo}\ln \left( \frac{4em\sqrt{p_T}W_{rvo}V^{T}}{\gamma\epsilon(V-1)}\right).
\end{equation*}
\end{theorem}
\begin{proof}
The covering number follows from the bound in Theorem 14.5 in \citet{anthony1999neural}, which is a $\|.\|_2^{\infty}$ covering number, but we know that $\|.\|_2^{\ell_2}$ is always smaller than $\|.\|_2^{\infty}$. Therefore, from Theorem 14.5 in \citet{anthony1999neural}, Lemma~\ref{lemma:real_high}, and the fact that sigmoid is a Lipschitz continuous function with Lipschitz factor of 1 we know that
\begin{equation*}
        \ln N_U(\epsilon,\cF,m,\|.\|_2^{\ell_2}) \leq p_T W_{rvo}\ln \left( \frac{2em\sqrt{p_T}W_{rvo}V^{T}}{\epsilon(V-1)}\right).
\end{equation*}
Combining the above equation with Lemma~\ref{lemma:cover_to_margin} will result in the desired bound.
\end{proof}
Finally, we will present the Spectral bound in \citet{bartlett2017spectrally}.
\begin{theorem}[Spectral covering number]\label{thm:spectral}
Let $\snet{d}{p}{s}{b}=\{f_W:\bR^d\to[0,1]^p \mid f_W(x)=\Phi(W\transpose x), W\in\bR^{d\times p} \text{ and } \|W\transpose\|_{\sigma}\leq s, \|W\transpose\|_{2,1}\leq b\}$ be the class of single-layer neural networks with $d$ inputs and $p$ outputs where spectral and $\|.\|_{2,1}$ norms of the layer is bounded by $s$ and $b$, respectively. Let $\snet{d}{p_1}{s_1}{b_1},\ldots,\snet{p_{T-1}}{p_T}{s_T}{b_T}$ be $T$ classes of neural networks and denote the $T$-layer neural network by $\cF = \snet{p_{T-1}}{p_T}{s_T}{b_T}\circ\ldots\circ \snet{d}{p_1}{s_1}{b_1}$. For an input set $S=\{x_1,\ldots,x_m\}\subset \bR^d$ define $X=[x_1 \ldots x_m]\in\bR^{d\times m}$ as the collection of input samples. Finally, denote by $w$ the maximum number of neurons in all layers of the network (including the input layer). Then we have
\begin{equation*}
\ln N_U(\epsilon,\cF_{\gamma},m,\|.\|_2^{\ell_2}) \leq \frac{4\|X\|_{F}^2\ln(2w^2)}{\gamma^2\epsilon^2}\left(\prod_{i=1}^Ts_i^2\right)\left(\sum_{i=1}^T\left(\frac{b_i}{s_i}\right)^{2/3}\right)^3.
\end{equation*}
\end{theorem}
The original bound in \citet{bartlett2017spectrally} considers the input norm $\|X\|_F^2$ to be the sum of $\|.\|_2^2$ norms of input samples and adjusts the chaining technique of Theorem~\ref{thm:dudely} to account for this assumption. Here, for the sake of consistency, we consider the Forbenious norm to be normalized and use the conventional chaining technique, which applies to the $\|.\|_2^{\ell_2}$ metric.
\begin{remark}\label{rem:sample_grow}
Some of the bounds that we presented are dependent on the number of input samples, $m$. However, for all of them the logarithm of covering number has at most a logarithmic dependence on the number of samples. It is also worth mentioning that the Spectral bound is dependent on the normalized Frobenious norm and increasing the number of copies of $S$ in Equation~\ref{eq:e2} (i.e., $mn$) will not change this norm and, therefore, the Spectral bound. 
\end{remark}
\section{Empirical results}\label{app:nvacgraphs}
In this appendix we will discuss details of the learning settings for the empirical results that were stated in Section~\ref{sec:experiments}. We train fully connected neural networks on the publicly available MNIST dataset, which consists of handwritten digits ($28 \times 28$ pixel images) with $10$ labels. Our baseline architecture has 3 hidden layers each containing 250 neurons, one input layer, and one output layer. The input layer has $784$ neurons, which are pixels of each image in MNIST dataset. The output layer has $10$ neurons, corresponding to the $10$ labels. All the activation functions are the shifted variant of the sigmoid function as discussed in Appendix~\ref{app:coverapproaches}, i.e., $\phi(x)=\frac{1}{1+e^{-x}}-\frac{1}{2}$. The additional architecures that we use are as follows: (a) fully connected neural networks with one input layer, one output layer, and $2,4,5$ hidden layers each containing 250 neurons; (b) fully connected neural networks with one input layer, one output layer, and three hidden layers each containing $64,150,350,500,800,1000,1500$ neurons. All of the experiments are performed using NVIDIA Titan V GPU.

Networks are trained with SGD optimizer with a momentum of $0.9$ and a learning rate of $0.3$. For the purpose of training the loss is set to be the cross-entropy loss. For the rest of the experiments (e.g., to report the accuracy and NVACs) ramp loss with a margin of $\gamma=0.1$ is used. The size of training, validation, and test sets are $59000$, $1000$, and $10000$, respectively. In Corollary~\ref{thm:ours_main} we are considering noisy networks with its expectation as output. Therefore, for reporting results of Corollary~\ref{thm:ours_main} we compute the output $50$ times and take an average. Computing random outputs several times and averaging them yields in negligible error bars in the demonstrated results.

The results of NVAC as a function of depth and width are depicted in Figure~\ref{fig:app1}. All of the NVACs are derived according to Remark~\ref{rem:nvac}. In Figure~\ref{fig:app1}, we also include the Norm-based approach (Theorem~\ref{thm:norm_based}) which was omitted from the Figures in Section~\ref{sec:experiments} due to its large scale. As mentioned in Section~\ref{sec:experiments}, Corollary~\ref{thm:ours_main} outperforms other bounds. In the following, we will investigate this observation. 

  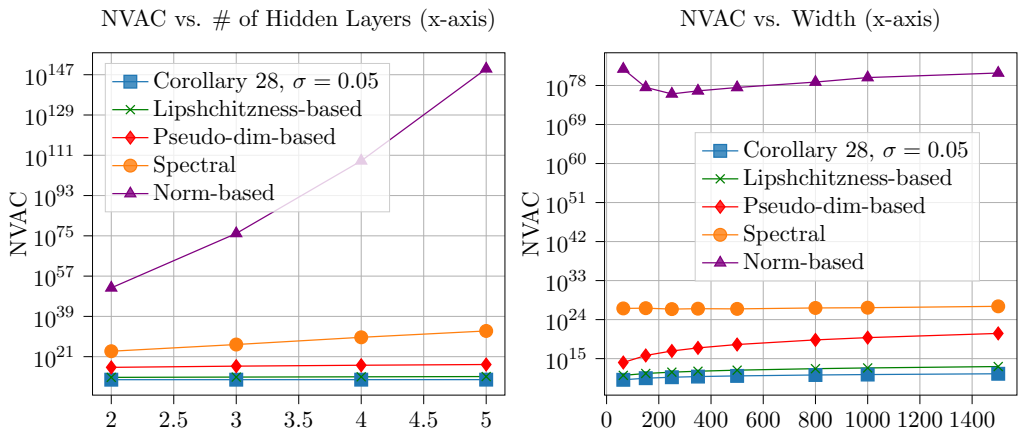
\begin{figure}
    \centering
    \subfloat{
\begin{tikzpicture}[scale=0.8]

\definecolor{darkgray176}{RGB}{176,176,176}
\definecolor{darkorange25512714}{RGB}{255,127,14}
\definecolor{green01270}{RGB}{0,127,0}
\definecolor{lightgray204}{RGB}{204,204,204}
\definecolor{purple}{RGB}{128,0,128}
\definecolor{steelblue31119180}{RGB}{31,119,180}

\begin{axis}[
legend cell align={left},
legend style={
  fill opacity=0.8,
  draw opacity=1,
  text opacity=1,
  at={(0.03,0.97)},
  anchor=north west,
  draw=lightgray204
},
log basis y={10},
tick align=outside,
tick pos=left,
title={NVAC vs. \# of Hidden Layers (x-axis)},
x grid style={darkgray176},
xlabel={},
xmajorgrids,
xmin=1.85, xmax=5.15,
xtick style={color=black},
y grid style={darkgray176},
ylabel={NVAC},
ymajorgrids,
ymin=6514.84759889463, ymax=6.47882009500106e+156,
ymode=log,
ytick style={color=black},
ytick={1e-15,1000,1e+21,1e+39,1e+57,1e+75,1e+93,1e+111,1e+129,1e+147,1e+165,1e+183},
yticklabels={
  \(\displaystyle {10^{-15}}\),
  \(\displaystyle {10^{3}}\),
  \(\displaystyle {10^{21}}\),
  \(\displaystyle {10^{39}}\),
  \(\displaystyle {10^{57}}\),
  \(\displaystyle {10^{75}}\),
  \(\displaystyle {10^{93}}\),
  \(\displaystyle {10^{111}}\),
  \(\displaystyle {10^{129}}\),
  \(\displaystyle {10^{147}}\),
  \(\displaystyle {10^{165}}\),
  \(\displaystyle {10^{183}}\)
}
]
\addplot [semithick, steelblue31119180, mark=square*, mark size=3, mark options={solid}]
table {%
2 51348186946.4345
3 51897254520.7146
4 52940188934.7658
5 54163313232.3868
};
\addlegendentry{Corollary~28, $\sigma=0.05$}
\addplot [semithick, green01270, mark=x, mark size=3, mark options={solid}]
table {%
2 561348889381.558
3 756692656869.744
4 998775160950.438
5 1300399857111.24
};
\addlegendentry{Lipshchitzness-based}
\addplot [semithick, red, mark=diamond*, mark size=3, mark options={solid}]
table {%
2 1.67727514437642e+16
3 5.81160033330929e+16
4 1.47368111391297e+17
5 3.11306486643974e+17
};
\addlegendentry{Pseudo-dim-based}
\addplot [semithick, darkorange25512714, mark=*, mark size=3, mark options={solid}]
table {%
2 2.70414353980701e+23
3 2.65862473449465e+26
4 4.48382303730719e+29
5 3.07972375913828e+32
};
\addlegendentry{Spectral}
\addplot [semithick, purple, mark=triangle*, mark size=3, mark options={solid}]
table {%
2 5.54205663826141e+51
3 1.13376887344502e+76
4 3.21099967436937e+108
5 4.3074257595747e+149
};
\addlegendentry{Norm-based}
\end{axis}

\end{tikzpicture}}
    \subfloat{
\begin{tikzpicture}[scale=0.8]

\definecolor{darkgray176}{RGB}{176,176,176}
\definecolor{darkorange25512714}{RGB}{255,127,14}
\definecolor{green01270}{RGB}{0,127,0}
\definecolor{lightgray204}{RGB}{204,204,204}
\definecolor{purple}{RGB}{128,0,128}
\definecolor{steelblue31119180}{RGB}{31,119,180}

\begin{axis}[
legend cell align={left},
legend style={
  fill opacity=0.8,
  draw opacity=1,
  text opacity=1,
  at={(0.91,0.55)},
  anchor=east,
  draw=lightgray204
},
log basis y={10},
tick align=outside,
tick pos=left,
title={NVAC vs. Width (x-axis)},
x grid style={darkgray176},
xlabel={},
xmajorgrids,
xmin=-7.8, xmax=1571.8,
xtick style={color=black},
xticklabel style={/pgf/number format/1000 sep=},
y grid style={darkgray176},
ylabel={NVAC},
ymajorgrids,
ymin=3820068.27640608, ymax=3.52196140113934e+85,
ymode=log,
ytick style={color=black},
ytick={0.001,1000000,1e+15,1e+24,1e+33,1e+42,1e+51,1e+60,1e+69,1e+78,1e+87,1e+96},
yticklabels={
  \(\displaystyle {10^{-3}}\),
  \(\displaystyle {10^{6}}\),
  \(\displaystyle {10^{15}}\),
  \(\displaystyle {10^{24}}\),
  \(\displaystyle {10^{33}}\),
  \(\displaystyle {10^{42}}\),
  \(\displaystyle {10^{51}}\),
  \(\displaystyle {10^{60}}\),
  \(\displaystyle {10^{69}}\),
  \(\displaystyle {10^{78}}\),
  \(\displaystyle {10^{87}}\),
  \(\displaystyle {10^{96}}\)
}
]
\addplot [semithick, steelblue31119180, mark=square*, mark size=3, mark options={solid}]
table {%
64 13205924003.3771
150 30686361777.6472
250 51216675135.9733
350 72268338983.3325
500 104064325981.589
800 168856734692.262
1000 211558469685.488
1500 322519292017.711
};
\addlegendentry{Corollary~28, $\sigma=0.05$}
\addplot [semithick, green01270, mark=x, mark size=3, mark options={solid}]
table {%
64 137766868695.87
150 380186511096.053
250 756692656869.744
350 1247122380783.46
500 2185495634834.64
800 4809059777031.86
1000 7131565315725.06
1500 14937255156100.9
};
\addlegendentry{Lipshchitzness-based}
\addplot [semithick, red, mark=diamond*, mark size=3, mark options={solid}]
table {%
64 126913696006677
150 5.37782869049592e+15
250 5.81160033330929e+16
350 2.97997493555181e+17
500 1.79263869708445e+18
800 2.09689808144143e+19
1000 6.98024506489934e+19
1500 6.52683570575932e+20
};
\addlegendentry{Pseudo-dim-based}
\addplot [semithick, darkorange25512714, mark=*, mark size=3, mark options={solid}]
table {%
64 3.7956495693336e+26
150 4.33315129139668e+26
250 2.65862473449465e+26
350 3.29559739224765e+26
500 2.97220906335308e+26
800 4.87732395291279e+26
1000 5.58337408836247e+26
1500 1.15447141003031e+27
};
\addlegendentry{Spectral}
\addplot [semithick, purple, mark=triangle*, mark size=3, mark options={solid}]
table {%
64 6.36929183470006e+81
150 3.64209159126428e+77
250 1.13376887344502e+76
350 5.69429440409314e+76
500 3.45174289577547e+77
800 6.1450842623535e+78
1000 6.70106956864947e+79
1500 7.59349454378849e+80
};
\addlegendentry{Norm-based}
\end{axis}

\end{tikzpicture}} 
    \caption{NVAC of different generalization bounds as a function of the number of hidden layers and width of the network.}    \label{fig:app1}
\end{figure}

The first justification behind this observation is the dependence on $1/\epsilon$. As it was discussed, we know that the NVAC in Norm-based and Spectral bounds has an extra polynomial dependence on $1/\epsilon$, compared to other bounds including Corollary~\ref{thm:ours_main}.

The second reason behind this observation is that the Spectral and Norm-based bounds depend on the product of the weights. Although one may think that in networks with large number of parameters this dependency would be better than those on the number of parameters, we will see that the Pseudo-dim-based bound, Lipschitzness-based bound, and Corollary~\ref{thm:ours_main} perform better in these cases. For instance, consider the network that has been trained with three hidden layers, each containing $1500$ neurons. In this case, the number of parameters is $\approx 5\times 10^{9}$, while in the Spectral bound, the contribution of product of norms to covering number is $\approx 1\times 10^{9}$ and the contribution of $1/\epsilon$ is $\approx 4 \times 10^{4}$. In the norm-based bound the contribution of the product of norms is $\approx 1 \times 10^{53}$ alone.

Finally, we will explore this observation by considering the dependence of these bounds on size of the network. In Section~\ref{subsec:cover_bounds} we discussed that Pseudo-dim-based bound has the worst dependence and comparing Corollary~\ref{thm:ours_main} with Lipschitzness-based bound is not straightforward. The empirical results, however, suggests that the Lipschitzness-based bound is worse than Corollary~\ref{thm:ours_main}.

It is worth mentioning that in the rightmost graph in Figure~\ref{fig:nvac_all}, the output of noisy networks are averaged over 1000 noisy outputs to obtain results that are more close to the true expectation that has been considered in the output of architecture in Corollary~\ref{thm:ours_main}. 

\section{Techniques to estimate smooth densities with mixtures of Gaussians}\label{app:gmm}
\paragraph{Notation.} 
Denote by $\mD(\rv{x})$ the probability density function of the random variable $\rv{x}$. Let $\indicator\{x \in S\}$ be an indicator function that outputs $1$ if $x\in S$ and $0$ if $x \notin S$. For a function $f:\cX\rightarrow\cY$, let $f_+(x)=\max\{0,f(x)\}$ and $f_-(x)=\min\{0,f(x)\}$. By $\bR^d\setminus [-B,B]^d$ we refer to the complement of set $[-B,B]^d$ with respect to $\bR^d$. We also denote by $f*g$ the convolution of functions $f$ and $g$. For two sets $S_1$ and $S_2$, we define their Cartesian product by $S_1\times S_2$ and by $S^d$ we refer to the Cartesian power, i.e., $S^d=\{(s_1,\ldots,s_d)\mid s_i\in S,\forall i\in[d]\}$. In the following lemma, we sometimes drop the overlines in our notation and simply write $x$ when we are referring to random variables. When it is clear from the context, we write $f$ instead of $f(x)$.

\begin{lemma}[Gaussian kernel estimation of bounded distributions]\label{lemma:gaus_kernel}
Let $\rv{x}$ be a random variable in $\rv{\cX_{B,d}}$ and denote its probability density function by $f=\mD(\rv{x})$. Let $g$ be the density function of a zero mean Gaussian random variable with covariance matrix $\sigma^2I_d$. Given a set $S=\{x_1,\ldots,x_n\}\subset\bR^d$ of i.i.d. samples $x_i\sim f,\,i\in[n]$, we define the empirical measure as $\mu_n(x)=\frac{\indicator\{x\in S\}}{n}$. Then, we have
\begin{equation*}
    \expect{\int_{\bR^d} \left|(\mu_n * g)(x) - (f*g)(x)\right|dx}\leq 2\sqrt{\frac{1}{n}}\left(\frac{2B}{\sqrt{(2\pi\sigma^2)}}+1\right)^d
\end{equation*}
\end{lemma}
\begin{proof}
Note that $\int \mu_n(x)dx=1$ and since $f$ and $g$ are probability density functions, we know that $\int(f*g)(x)dx=1$ and $\int (\mu_n*g)(x)dx=1$. Therefore, we have (for simplicity, we write $\bE_{x_i\sim f}$ instead of $\bE_{\substack{x_i\sim f,\\i\in[n]}}$)
\begin{equation}\label{eq:C0}
    \begin{aligned}
           &\expects{x_i\sim f}{\int_{\bR^d} \left|(\mu_n * g)(x) - (f*g)(x)\right|dx}\\
           &= \int_{\bR^d} \expects{x_i\sim f}{ \left|(\mu_n * g)(x) - (f*g)(x)\right|dx}\\
           &=2\int_{\bR^d} \expects{x_i\sim f}{ \left(\mu_n * g - f*g\right)_+(x)dx}\\
           &\leq 2\int_{\bR^d} \sqrt{\expects{x_i\sim f}{ \left((\mu_n * g)(x) - (f*g)(x)\right)^2}}dx && \text{(By Jensen's inequality)}\\
           &\leq 2\int_{\bR^d} \sqrt{\expects{x_i\sim f}{ \left(\frac{1}{n}\sum_{i=1}^n g(x-x_i)- \int f(y)g(x-y)dy\right)^2}}dx.
    \end{aligned}
\end{equation}
Now, we can write
\begin{equation}\label{eq:C1}
    \begin{aligned}
    &\expects{x_i\sim f}{ \left(\frac{1}{n}\sum_{i=1}^n g(x-x_i)- \int f(y)g(x-y)dy\right)^2} = \expects{x_i\sim f}{\left(\frac{1}{n}\sum_{i=1}^n g(x-x_i)\right)^2}\\
    &+ \expects{x_i\sim f}{\left(  \int f(y)g(x-y)dy \right)^2}- \expects{x_i\sim f}{2\left(\frac{1}{n}\sum_{i=1}^n g(x-x_i)\right)\left(  \int f(y)g(x-y)dy \right)}\\
    &= \expects{x_i\sim f}{\left(\frac{1}{n}\sum_{i=1}^n g(x-x_i)\right)^2}+\left(\int f(y)g(x-y)dy \right)^2\\
    &- 2\left(\int f(y)g(x-y)dy \right)\left(\frac{1}{n}\sum_{i=1}^n \expects{x_i\sim f}{g(x-x_i)}\right)\\
    &=\expects{x_i\sim f}{\left(\frac{1}{n}\sum_{i=1}^n g(x-x_i)\right)^2}-\left(\int f(y)g(x-y)dy \right)^2,
    \end{aligned}
\end{equation}
where the last equality comes from the fact that the expectation is over random variables $x_1,\ldots,x_n$
\begin{equation*}
    \begin{aligned}
    \frac{1}{n}\sum_{i=1}^n \expects{x_i\sim f}{g(x-x_i)}=\frac{1}{n}\sum_{i=1}^n\int g(x-y)f(y)dy=\int g(x-y)f(y)dy=f*g.
    \end{aligned}
\end{equation*}
Next, we know that
\begin{equation}\label{eq:C2}
    \begin{aligned}
    &\expects{x_i\sim f}{\left(\frac{1}{n}\sum_{i=1}^n g(x-x_i)\right)^2} = \frac{1}{n^2}\expects{x_i\sim f}{\left(\sum_{i=1}^n g(x-x_i)\right)^2}\\
    &=\frac{1}{n^2}\expects{x_i\sim f}{\sum_{i=1}^n g(x-x_i)^2}+\frac{1}{n^2}\expects{}{\sum_{i\neq j}^n g(x-x_i)g(x-x_j)}\\
    &=\frac{1}{n^2}\sum_{i=1}^n \expects{x_i\sim f}{g(x-x_i)^2}+\frac{1}{n^2}\sum_{i\neq j}^n\expects{x_i,x_j\sim f}{ g(x-x_i)g(x-x_j)}\\
    &=\frac{1}{n}\expects{x_i\sim f}{g(x-x_i)^2}+\frac{1}{n^2}\sum_{i\neq j}^n\expects{x_i\sim f}{g(x-x_i)}\expects{x_j\sim f}{g(x-x_j)}\\
    &=\frac{1}{n}\expects{x_i\sim f}{g(x-x_i)^2}+(1-\frac{1}{n})\left(\expects{x_i\sim f}{g(x-x_i)}\right)^2\\
    &=\frac{1}{n}\expects{x_i\sim f}{g(x-x_i)^2}+(1-\frac{1}{n})\left(\int g(x-y)f(y)dy\right)^2.
    \end{aligned}
\end{equation}
Putting Equations~\ref{eq:C2} and $\ref{eq:C1}$ together, we have
\begin{equation}
    \begin{aligned}
     &\expects{x_i\sim f}{ \left(\frac{1}{n}\sum_{i=1}^n g(x-x_i)- \int f(y)g(x-y)dy\right)^2}\\
     &=\frac{1}{n}\expects{x_i\sim f}{g(x-x_i)^2}-\frac{1}{n}\left(\int g(x-y)f(y)dy\right)^2\\
     &=\frac{1}{n}\int g(x-y)^2f(y)dy-\frac{1}{n}\left(\int g(x-y)f(y)dy\right)^2\\
     &=\frac{1}{n}\left(f*g^2-(f*g)^2\right).
    \end{aligned}
\end{equation}
Therefore, we can rewrite Equation~\ref{eq:C0} as
\begin{equation}\label{eq:C3}
    \begin{aligned}
           &\expects{x_i\sim f}{\int_{\bR^d} \left|(\mu_n * g)(x) - (f*g)(x)\right|dx}\\
           &\leq {2\int_{\bR^d}\sqrt{\frac{1}{n}(f*g^2-(f*g)^2)}dx}\\
           &\leq {2\sqrt{\frac{1}{n}}\int_{\bR^d}\sqrt{(f*g^2-(f*g)^2)}dx}.
    \end{aligned}
\end{equation}
We know that $g$ is the probability density function of $\cN(\mathbf{0},\sigma^2I_d)$. Consequently, we know that
\begin{equation*}
    g(x)^2=\frac{1}{(2\pi)^d \sigma^{2d}}\exp({-\frac{1}{\sigma^2}x\transpose x})\leq\frac{1}{(2\pi\sigma ^2)^d},
\end{equation*}
and we can rewrite Equation~\ref{eq:C3} as
\begin{equation*}
    \begin{aligned}
      &\expects{x_i\sim f}{\int_{\bR^d} \left|(\mu_n * g)(x) - (f*g)(x)\right|dx}\\
           &\leq 2\sqrt{\frac{1}{n}}\int_{\bR^d}\sqrt{(f*g^2-(f*g)^2)}dx\leq 2\sqrt{\frac{1}{n}} \int_{\bR^d}\sqrt{f*g^2}dx\\
           & \leq 2\sqrt{\frac{1}{n}} \int_{\bR^d}\sqrt{\int g(x-y)^2f(y)dy}\,dx\\
           & =2\sqrt{\frac{1}{n}} \int_{\bR^d}\sqrt{\int \frac{1}{(2\pi\sigma^2)^d}\exp\left(-\frac{1}{\sigma^2}(x-y)\transpose(x-y)\right)f(y)dy}\,dx\\
            & =2\sqrt{\frac{1}{n}} \int_{[-B,B]^d}\sqrt{\int \frac{1}{(2\pi\sigma^2)^d}\exp\left(-\frac{1}{\sigma^2}(x-y)\transpose(x-y)\right)f(y)dy}\,dx\\
           & +2\sqrt{\frac{1}{n}} \int_{\bR^d\setminus [-B,B]^d}\sqrt{\int \frac{1}{(2\pi\sigma^2)^d}\exp\left(-\frac{1}{\sigma^2}(x-y)\transpose(x-y)\right)f(y)dy}\,dx\\
            & \leq 2\sqrt{\frac{1}{n}} \int_{[-B,B]^d}\sqrt{\int \frac{1}{(2\pi\sigma^2)^d}f(y)dy}\,dx \\
            &+ 2\sqrt{\frac{1}{n}} \int_{\bR^d\setminus[-B,B]^d}\sqrt{\frac{1}{(2\pi\sigma^2)^d}\int \exp\left(-\frac{1}{\sigma^2}(x-y)\transpose(x-y)\right)f(y)dy}\,dx.
    \end{aligned}
\end{equation*}
We can then conclude that
\begin{equation}
    \begin{aligned}
             &\expects{x_i\sim f}{\int_{\bR^d} \left|(\mu_n * g)(x) - (f*g)(x)\right|dx}\\
             & \leq 2\sqrt{\frac{1}{n}} \int_{[-B,B]^d}\sqrt{ \frac{1}{(2\pi\sigma^2)^d}}\,dx\\
            &+ 2\sqrt{\frac{1}{n}} \int_{\bR^d\setminus[-B,B]^d}\sqrt{\frac{1}{(2\pi\sigma^2)^d}\int \exp\left(-\frac{1}{\sigma^2}(x-y)\transpose(x-y)\right)f(y)dy}\,dx\\
             & \leq 2\sqrt{\frac{1}{n}} \frac{(2B)^d}{\sqrt{(2\pi\sigma^2)^d}}+ 2\sqrt{\frac{1}{n}} \int_{\bR^d\setminus[-B,B]^d}\sqrt{\frac{1}{(2\pi\sigma^2)^d}\int \exp\left(-\frac{1}{\sigma^2}(x-y)\transpose(x-y)\right)dy}\,dx\\
            & \leq 2\sqrt{\frac{1}{n}} \frac{(2B)^d}{\sqrt{(2\pi\sigma^2)^d}}+2\sqrt{\frac{1}{n}} \int_{\bR^d\setminus[-B,B]^d}\sqrt{\frac{1}{(2\pi\sigma^2)^d}} \int\exp\left(-\frac{1}{2\sigma^2}(x-y)\transpose(x-y)\right)dy\,dx\\
            & \leq 2\sqrt{\frac{1}{n}}\frac{(2B)^d}{\sqrt{(2\pi\sigma^2)^d}}+2\sqrt{\frac{1}{n}}\sum_{i=1}^d \binom{d}{i} \frac{(2B)^{d-i}\sqrt{(2\pi)^i}\sigma^i}{\sqrt{(2\pi\sigma^2)^d}}\\
            &\leq 2\sqrt{\frac{1}{n}}\sum_{i=0}^d \binom{d}{i} \frac{(2B)^{d-i}}{\sqrt{(2\pi\sigma^2)^{d-i}}} = 2\sqrt{\frac{1}{n}}\left(\frac{2B}{\sqrt{(2\pi\sigma^2)}}+1\right)^d.
    \end{aligned}
\end{equation}
Here, we used the fact that for $f$ is supported on $[-B,B]^d$ and the maximum value of $\exp(-(1/\sigma^2)(x-y)\transpose(x-y))$ is $1$ over $[-B,B]^d$. Moreover, for a fixed $x$ in $\bR^d\setminus[-B,B]^d$, the maximum value of $\exp(-(1/\sigma^2)(x-y)\transpose(x-y))$ happens when $(x-y)\transpose(x-y)$ is minimized, therefore,
Whenever $x^{(i)}>B$, the minimization occurs when $y^{(i)}=B$. On the other hand, when $x^{(i)}<B$, the minimization happens when $y^{(i)}=-B$. We can, then, consider the integration over $\bR^d\setminus[-B,B]^d$ as sum of integrals over subsets where for some $i \in [d],|x^{(i)}|>B$. Then we can upper bound the integration over each subset by the marginalization of the Gaussian variable in dimensions where $|x^{(i)}|>B$ and consider the fact that the exponent is always smaller than the exponent of an $i$ dimensional Gaussian distribution in those subsets. Note that, when we use this lemma, we consider large values of $n$ such that the expectation of our kernel estimation can get as small as desired. It is also noteworthy that the upper bound on the expectation implies that there exists a set of samples $S=\{x_1,\ldots,x_n\}$ that can achieve the desired upper bound.
\end{proof}

Lemma~\ref{lemma:gaus_kernel} can be used to estimate any bounded distribution that is perturbed with Gaussian noise with a mixture of Gaussians with bounded means and equal diagonal covariance matrices. To do so, we can first use Lemma~\ref{lemma:gaus_kernel} to approximate the distributions with Gaussian kernels over $n$ i.i.d samples from the distribution. We can then divide the subset $[-B,B]^d$ into several subsets and define a Gaussian on each subset that has a weight equal to the number of samples on each interval. We provide the formal version of this estimation in the following lemma. 
\begin{lemma}\label{lema:gmm}
Let $\rv{x}\in \rv{\cX_{B,d}}$ be a random variable and denote its probability density function by $f=\mD(\rv{x})$. Let $g$ be the density function of a zero mean Gaussian random variable with covariance matrix $\sigma^2I_d$. Then for any small value $\eta$, we can estimate $f*g$ by a mixture of $k=\lceil \frac{B}{\eta} \rceil^d$ Gaussians $\sum_{i=1}^{k} g(x-\mu_i)$, where $\mu_i\in[-B,B]^d$ and
\begin{equation*}
    d_{TV}(f*g,\sum_{i=1}^{k} g(x-\mu_i))\leq \frac{2\sqrt{d}\eta}{\sigma}
\end{equation*}
\end{lemma}
\begin{proof}
From Lemma~$\ref{lemma:gaus_kernel}$, we know that there exists a set $S=\{x_1,\ldots,x_n\}\subset\bR^d$ of i.i.d. samples from $f$ and its empirical measure $\mu_n(x)=\frac{\indicator\{x\in S\}}{n}$ such that the total variation between $f$ and the sum of Gaussian kernels defined on empirical measure is bounded
\begin{equation*}
    d_{TV}\left(f*g,\sum_{i=1}^n g(x-x_i)\right)\leq2\sqrt{\frac{1}{n}}\left(\frac{2B}{\sqrt{(2\pi\sigma^2)}}+1\right)^d = \epsilon.
\end{equation*}
Denote $m=\lceil \frac{B}{\eta} \rceil$. We construct the following grid $P$ of points on $[-B,B]^d$ and choose means of the Gaussian densities based on it
\begin{equation*}
    P=\{-B+2i\eta \mid i \in [m]\}^d.
\end{equation*}
For any $a=(a_1,\ldots,a_d)\in[m]^d$, we define
\begin{equation*}
    \mu_a=\left[-B+(2a_1+1)\eta, \,\,\,\ldots\,\,\, ,-B+(2a_d+1)\eta\right]\transpose \in \bR^d
\end{equation*}
as a choice of mean vector for the Gaussian mixture. We claim that by choosing appropriate weights, we can estimate $f*g$ with respect to total variation distance by a mixture of Gaussians with means in the following set
\begin{equation*}
    M=\left\{\mu_a=[\mu_a^{(1)} \ldots \mu_a^{(d)}]\transpose \in \bR^d \mid \mu_a^{(i)}=-B+(2a_i+1)\eta,\ \forall a=(a_1,\ldots,a_d)\in[m]^d\right\}.
\end{equation*}
For the set $S=\{x_1,\ldots,x_n\}$ that was sampled for kernel estimate $\mu_n*g$, we choose the weight $w_a$ for the Gaussian density with mean $\mu_a$ as follows. Define the set $S_a$ as
\begin{equation}\label{means}
    S_a=\left\{x_i\in S\mid x_i\in\left[-B+2a_1\eta, -B+2(a_1+1)\eta\right]\times\ldots\times[-B+2a_d\eta, -B+2(a_d+1)\eta]\right\}
\end{equation}
Next, we select $w_a$ as
\begin{equation*}
    w_a=\frac{1}{n}\sum_{i=1}^n \indicator\left\{x_i\in S_a \right\} =\frac{|S_a|}{n} .
\end{equation*}
In other words, $w_a$ is the number of samples in $S$ that the $\ell_{\infty}$ distance between those samples and $\mu_a$ is smaller than $2\eta$. Note that the cardinality of $M$, which is the number of Gaussian densities in the mixture is $|M|=(\lceil\frac{B}{\eta}\rceil)^d$.

We now prove that the total variation distance between $\mu_n*g$ and $\sum_{a\in[m]^d} w_ag(x-\mu_a)$ is smaller than $\frac{\sqrt{d}}{\sigma}\eta$. 
\begin{equation}\label{eq:gkernel1}
\begin{aligned}
        & d_{TV}\left(\frac{1}{n}\sum_{i=1}^n g(x-x_i),\sum_{a\in[m]^d} w_ag(x-\mu_a)\right)\\
        &=\frac{1}{2}\left\|  \frac{1}{n}\sum_{i=1}^n g(x-x_i)-\sum_{a\in[m]^d} w_ag(x-\mu_a)   \right\|_1\\
        & =\frac{1}{2}\left\| \sum_{a\in [m]^d}\left( \frac{1}{n}\sum_{x_i\in S_a} g(x-x_i)- w_ag(x-\mu_a)\right)   \right\|_1 \\
        & \leq \frac{1}{2}\sum_{a\in [m]^d}\left\| \left( \frac{1}{n}\sum_{x_i\in S_a} g(x-x_i)- w_ag(x-\mu_a)\right)   \right\|_1 && \text{(By triangle inequality)}.
\end{aligned}
\end{equation}
Now, we can write
\begin{equation}\label{eq:gkernel2}
    \begin{aligned}
           & \left\|\frac{1}{n}\sum_{x_i\in S_a} g(x-x_i)- w_ag(x-\mu_a)\right\|_1\\
           & \leq \left\|\frac{1}{n}\sum_{x_i\in S_a} \left(g(x-x_i)- g(x-\mu_a)\right)\right\|_1 && \text{(Since $w_a=\frac{|S_a|}{n}$)}\\
            & \leq \frac{1}{n}\sum_{x_i\in S_a} \left\|g(x-x_i)- g(x-\mu_a)\right\|_1.
    \end{aligned}
\end{equation}
From Theorem~\ref{thm:tv_gaussian}, we know that
\begin{equation}\label{eq:gkernel3}
    \begin{aligned}
           &2d_{TV}\left(  g(x-x_i)- g(x-\mu_a) \right) = \left\|g(x-x_i)- g(x-\mu_a)\right\|_1\\
           &\leq \frac{\|x_i-\mu_a\|_2}{\sigma}\leq \frac{\sqrt{d}}{\sigma}2\eta.
    \end{aligned}
\end{equation}
Putting Equation~\ref{eq:gkernel3} into Equation~\ref{eq:gkernel2}, we have
\begin{equation}\label{eq:gkernel4}
\begin{aligned}
     & \left\|\frac{1}{n}\sum_{x_i\in S_a} g(x-x_i)- w_ag(x-\mu_a)\right\|_1\\
     &  \leq \frac{1}{n}\sum_{x_i \in S_a} \frac{\sqrt{d}}{\sigma}2\eta=\frac{\sqrt{d}}{\sigma}2\eta w_a.
\end{aligned}
\end{equation}
Now, putting Equations~\ref{eq:gkernel2} and \ref{eq:gkernel4} together, we can rewrite Equation~\ref{eq:gkernel1} as
\begin{equation}\label{eq:gkernel5}
    \begin{aligned}
     & d_{TV}\left(\frac{1}{n}\sum_{i=1}^n g(x-x_i),\sum_{a\in[m]^d} w_ag(x-\mu_a)\right)\\
     & \leq \frac{1}{2}\sum_{a\in [m]^d}\left\| \left( \frac{1}{n}\sum_{x_i\in S_a} g(x-x_i)- w_ag(x-\mu_a)\right)   \right\|_1\\
    & \leq \frac{1}{2}\sum_{a\in[m]^d}\frac{\sqrt{d}}{\sigma}2\eta  w_a\\
    &=\frac{\sqrt{d}}{\sigma}\eta.
    \end{aligned}
\end{equation}
Note that the bound in Equation~\ref{eq:gkernel5} does not depend on the size of sampled set $S$. Therefore, we can choose $n$ as large as we want. Specifically, we choose $n$ as follows
\begin{equation*}
    n=\left(\frac{2B}{\sqrt{2\pi\sigma^2}}+1\right)^{2d}. \left(\frac{\sqrt{d}}{2\sigma}\eta\right)^{-2}
\end{equation*}
We can then conclude that for any random variable $\rv{x}$ defined over $[-B,B]^d$, we can approximate the density function of $\rv{x}+\rv{z},\rv{z}\sim \cN(\mathbf{0},\sigma^2I_d)$ with a mixture of $\lceil \frac{B}{\eta}\rceil^d$ Gaussians with means in $[-B,B]^d$ such that
\begin{equation*}
    d_{TV}\left(f*g,\sum_{a\in[m]^d}w_ag(x-\mu_a)\right)\leq \epsilon+\frac{\sqrt{d}}{\sigma}\eta=\frac{2\sqrt{d}\eta}{\sigma}.
\end{equation*}
\end{proof}

\end{document}